\documentclass{article}

\PassOptionsToPackage{numbers, compress}{natbib}
\usepackage[final]{NeurIPS/neurips_2025}

\usepackage[utf8]{inputenc} % allow utf-8 input
\usepackage[T1]{fontenc}    % use 8-bit T1 fonts
\usepackage[hidelinks]{hyperref}       % hyperlinks
\usepackage{url}            % simple URL typesetting
\usepackage{booktabs}       % professional-quality tables
\usepackage{amsfonts, amsmath, amsthm, amssymb, mathtools} % math
\usepackage{nicefrac}       % compact symbols for 1/2, etc.
\usepackage{booktabs}       % professional-quality tables
\usepackage{microtype}      % microtypography
\usepackage{xcolor}         % colors
\usepackage{enumitem}
\usepackage{algorithm}
\usepackage{algpseudocode}
\usepackage{pgfplots}
\usepackage{svg}
\usepackage{graphicx}
\usepackage{pifont}
\usepackage{subfigure}
\usepackage[most]{tcolorbox} % The 'many' library is loaded for additional options

\newcommand{\brown}{\raisebox{2pt}{\tikz{\draw[custom_brown, solid, line width=2.3pt](0,0) -- (5mm,0);}}}
\newcommand{\purple}{\raisebox{2pt}{\tikz{\draw[custom_purple, solid, line width=2.3pt](0,0) -- (5mm,0);}}}
\newcommand{\blue}{\raisebox{2pt}{\tikz{\draw[custom_blue, solid, line width=2.3pt](0,0) -- (5mm,0);}}}
\newcommand{\green}{\raisebox{2pt}{\tikz{\draw[custom_green, dashed, line width=2.3pt](0,0) -- (5mm,0);}}}
\newcommand{\red}{\raisebox{2pt}{\tikz{\draw[custom_red, dashed, line width=2.3pt](0,0) -- (5mm,0);}}}
\newcommand{\black}{\raisebox{2pt}{\tikz{\draw[black, dashed, line width=2.3pt](0,0) -- (5mm,0);}}}

\theoremstyle{plain}
\newtheorem{theorem}{Theorem}[section]
\newtheorem{prop}[theorem]{Proposition}
\newtheorem{lemma}[theorem]{Lemma}

\theoremstyle{definition}
\newtheorem{definition}[theorem]{Definition}
\newtheorem{assumption}[theorem]{Assumption}
\theoremstyle{remark}
\newtheorem{remark}[theorem]{Remark}
\newtheorem*{problem}{Problem}

\renewcommand{\hat}{\widehat}
\renewcommand{\tilde}{\widetilde}
\renewcommand{\bar}{\overline}

\newcommand{\R}{{\mathbb R}}

\newcommand{\E}{{\mathbb E}}

\renewcommand{\L}{{\mathcal L}}

\newcommand\norm[1]{\left\lVert#1\right\rVert}

\newcommand{\mc}{\mathcal}

\newcommand{\mbb}{\mathbb}

\definecolor{custom_blue}{HTML}{1F77B4}
\definecolor{black}{HTML}{000000}
\definecolor{custom_orange}{HTML}{FF7F0E}
\definecolor{custom_green}{HTML}{2CA02C}
\definecolor{custom_green}{HTML}{2CA02C}
\definecolor{custom_brown}{HTML}{8C564B}
\definecolor{custom_purple}{HTML}{9467BD}
\definecolor{custom_red}{HTML}{D62728}

\newtcbox{\highlight}[1][customblue]{
    on line,
    arc=0pt, % Set to 0pt for sharp corners
    colback=#1!35!white,
    colframe=#1!35!white,
    boxsep=0pt,
    left=1pt,
    right=1pt,
    top=2pt,
    bottom=2pt,
    boxrule=0pt,
    nobeforeafter
}

\allowdisplaybreaks
\newtheoremstyle{explanationstyle} % Name of the style
{3pt} % Space above
{3pt} % Space below
{\itshape} % Body font
{} % Indent amount
{\bfseries} % Theorem head font
{.} % Punctuation after theorem head
{.5em} % Space after theorem head
{\thmname{#1}\thmnumber{ #2}\thmnote{ \textnormal{#3}}} % Theorem head spec

\theoremstyle{explanationstyle}

\theoremstyle{plain}
\theoremstyle{definition}
\theoremstyle{remark}

\DeclareMathOperator*{\argmax}{arg\,max}

\newcommand{\cvarb}{\mathrm{CV a R}^\beta}

\newcommand{\cmark}{\ding{51}}%
\newcommand{\xmark}{\ding{55}}%

\newcommand\numberthis{\addtocounter{equation}{1}\tag{\theequation}}

\newcommand{\inbrace}[1]{\left\{#1\right\}}

\newcommand{\inparen}[1]{\left(#1\right)}
\newcommand{\insquare}[1]{\left[#1\right]}
\newcommand{\inangle}[1]{\left\langle#1\right\rangle}

\let\norm\relax
\newcommand{\norm}[1]{\ensuremath{\left\lVert #1 \right\rVert}}

\definecolor{navy}{RGB}{0, 0, 128}

\title{Risk-Averse Constrained Reinforcement Learning\\with Optimized Certainty Equivalents}

\author{%
  Jane H. Lee\\
  Department of Computer Science\\
  Yale University\\
  \texttt{jane.h.lee@yale.edu} \\
  \And
  Baturay Saglam \\
  Department of Electrical Engineering \\
  Yale University \\
  \texttt{baturay.saglam@yale.edu} \\
  \AND
  Spyridon Pougkakiotis \\
  Department of Mathematics \\
  King's College London \\
  \texttt{spyridon.pougkakiotis@kcl.ac.uk} \\
  \And
  Amin Karbasi\thanks{Work done while affiliated with Yale University.} \\
  Foundation AI \\
  Cisco Systems Inc. \\
  \texttt{karbasi@cisco.com} \\
  \And
  Dionysis Kalogerias \\
  Department of Electrical Engineering \\
  Yale University \\
  \texttt{dionysis.kalogerias@yale.edu} \\
}

\begin{document}

\maketitle
\begin{abstract}
  Constrained optimization provides a common framework for dealing with conflicting objectives in reinforcement learning (RL). In most of these settings, the objectives (and constraints) are expressed though the expected accumulated reward. However, this formulation neglects risky or even possibly catastrophic events at the tails of the reward distribution, and is often insufficient for high-stakes applications in which the risk involved in outliers is critical. In this work, we propose a framework for risk-aware constrained RL, which exhibits per-stage robustness properties jointly in reward values and time using optimized certainty equivalents (OCEs). 
  Our framework ensures an exact equivalent to the original constrained problem within a parameterized strong Lagrangian duality framework under appropriate constraint qualifications, and yields a simple algorithmic recipe which can be wrapped around standard RL solvers, such as PPO. Lastly, we establish the convergence of the proposed algorithm under common assumptions, and verify the risk-aware properties of our approach through several numerical experiments.
\end{abstract}

\section{Introduction}

Autonomous agents are often used in settings where they must handle several conflicting requirements while maximizing a main objective, such as navigating a maze without hitting any walls. 
% or maximizing their score in a game without cheating. 
These conflicting requirements are often modeled with constrained reinforcement learning (RL) (see \citet{sutton_rl_textbook}) and are either solved as a multi-objective problem with conflicting requirements controlled by weights, chosen manually or through hyperparameter tuning (e.g., see \citet{borkar_actor_critic, constr_policy_opt, multi_criterion_rl, reward_constr_policy_opt, random_search_static_linear_rl}), or by using primal-dual methods which %choose the weights for different constraints automatically 
optimize over the weights \cite{reward_constr_policy_opt, online_actor_critic_alg_constr_rl}. Others have approached constrained RL through reward-free techniques, e.g., \cite{pmlr-v162-miryoosefi22a}, or strategic exploration \cite{pmlr-v139-yu21b}. The work of \citet{constrained_rl_zero_duality_gap} has given theoretical support to primal-dual approaches by proving that, despite their inherent nonconvexity, constrained RL models exhibit zero duality gap under mild assumptions. However, in the standard (risk-neutral) setting (including that of \citet{constrained_rl_zero_duality_gap}), the objectives and constraints are often formulated as expected values of a (discounted) sum of rewards. Nonetheless, in many practical scenarios, especially in the context of safety \cite{ai_safety_gridworlds} and other high-stakes applications, the standard  expectation may not be strict enough to describe the desired behavior. For example, instead of maximizing the on-average return (e.g., of an investment), one may want to learn policies that mitigate risk (e.g., of losing money). 

\paragraph{Stochastic Optimization with Risk Measures}
Risk measures are often used to quantify this risk in problems with uncertain (random) outcomes (see, e.g., \citet{stoch_programming_text_shapiro}). More broadly, stochastic optimization with risk measures in the objectives and/or constraints has been studied in the online setting \cite{stoch_primal_dual_cvar}, bandits \cite{pmlr-v89-cardoso19a, NIPS2012_83f25503,seq_dec_coherent_risk}, statistical learning \cite{bedi2020nonparametric,svm_convex_risk_func,JMLR:v18:15-566,Vitt2018RiskAverseC}, non-convex resource allocation problems \cite{risk_constrained_resource_alloc}, and others \cite{Jiang2018554, approx_risk_submodular}. These methods and analyses are not readily applicable in our setting either due to non-convexity/concavity in the objective and constraints or due to the fact that the risk envelope for our problem has dependence on the policy $\pi$. 
% \textcolor{blue}{Dionysis: is this fine to write it this way?}

\paragraph{Risk-Averse/Risk-Aware RL} 
% Nonetheless, in many practical scenarios, especially in the context of safety \cite{ai_safety_gridworlds} and other high-stakes applications, the standard  expectation may not be strict enough to describe the desired behavior. For example, instead of maximizing the on-average return (e.g., of an investment), one may want to learn policies that mitigate risk (of losing money). %For learning such policies, 
In reinforcement learning specifically, there is a deep literature on risk-averse methods~\cite{parallel_nonstat_direct_policy, 8247278,NEURIPS2020_fdc42b6b,geibel2005risk,mihatsch2002risk,PrashanthL2018RiskSensitiveRL,2023_wang_near_minimax_cvar,2023_xu_regret_recursive_oce,wang2025reductionsapproachrisksensitivereinforcement} with many applications in the area of robust control~\cite{10.3389/frobt.2021.617839,lam2023riskaware}. While the most conservative approach consists of learning policies robust to worst-case scenarios~\cite{abdullah2019wasserstein, risk_aware_q_learning_mdp}, we often prefer to take into account a set of events with significant probability. To do this,~\citet{huang2021convergence,NIPS2015_024d7f84} consider the policy gradient method under coherent risk measures, where the risk measure replaces the expectation over the discounted sum of rewards. The work of \citet{risk_contr_rl_percentile_risk} handles similar CVaR risk objectives with CVaR risk constraints by utilizing Lagrangian relaxation. 
\citet{2023_xu_regret_recursive_oce,wang2025reductionsapproachrisksensitivereinforcement} study risk-sensitive RL with OCEs.
Recently, \citet{bonetti_risk_averse_rl_coherent_risk} propose an alternative risk-aware framework for (unconstrained) reinforcement learning which places the risk measure in the occupancy measure to derive what the authors call the reward-based conditional value at risk (RCVaR) and mean-mean absolute deviation (Mean-RMAD), in contrast to the standard  return-based formulations (like that of \cite{risk_contr_rl_percentile_risk,huang2021convergence,NIPS2015_024d7f84}). 

\paragraph{Contributions} This work approaches risk-aware constrained RL through employing reward-based risk measures in both the objective and the constraints. Our contributions are as follows:
\begin{itemize}[noitemsep]
    \vspace{-5bp}
    \item We extend the setting of \citet{bonetti_risk_averse_rl_coherent_risk} and, to the best of our knowledge, our work is the first to handle reward-based constraints, covering a large class of risk measures (OCEs);
    % \vspace{-5bp}
    \item We establish a parameterized strong duality relation resulting in a computationally tractable partial Lagrangian relaxation (which is exact under certain constraint qualifications);
    % \vspace{-5bp}
    \item We propose an online algorithm which is modular, offering the user flexibility in modeling (allowing for a mix of risk-neutral or risk-averse objectives and/or constraints) and implementation (enabling the use of existing RL algorithms as a black-box). By rigorously reducing the underlying problem to a certain instance of stochastic minimax optimization, we then establish convergence under common assumptions.
    %by making the problem amenable to min-max optimization;
    % \vspace{-5bp}
    \item We demonstrate practical usefulness of our approach in extensive numerical experiments on standard benchmarks, showcasing its effectiveness in reducing risk constraint violations and improving stochastic stability through explicit risk management.\footnote{We provide the code at \url{https://github.com/baturaysaglam/risk-averse-constrained-RL}.}
    \vspace{-5bp}
\end{itemize}
The setting of \citet{bonetti_risk_averse_rl_coherent_risk} is closely related to this work but is restricted to the unconstrained case, {while our setting handles reward-based risk constraints as well}. Our algorithm differs from the methodology of \cite{bonetti_risk_averse_rl_coherent_risk}, which is based on a block-coordinate descent scheme, the analysis of which does not readily extend to the type of optimization problems arising in the constrained setting. Further, \cite{bonetti_risk_averse_rl_coherent_risk} requires exact policy solvers but we work with inexact ones (cf. Assumption \ref{assum:local_opt}). Compared to other risk-averse constrained RL methods, our partial Lagrangian relaxation is exact under a certain constraint qualification, whereas no meaningful (let alone exact) relation between the original primal problem and the employed Lagrangian relaxation has been shown in other works (e.g.,  \cite{risk_contr_rl_percentile_risk}).

\section{Background}

\paragraph{Notation} 
Given $m \in \mathbb{N}$, we let $[m] \coloneqq \{1,\ldots,m\}$. Given $x \in \mathbb{R}$, we write $(x)_+ \equiv \max\{x,0\}$. 
% The space of $p$-integrable functions, with $p \in [1,\infty)$, from a measurable space $(\Omega, \mathcal{F})$ equipped with a finite measure $\mu: \mathcal{F} \rightarrow \R_+$ to 
% % $\R^n$
% a Banach space $\mathbb{A}$ 
% is denoted as
% %with the standard notation 
% $\mathcal{L}_p(\Omega, \mathcal{F}, \mu; \mathbb{A})$ (and abbreviated as $\mathcal{L}_p(\mu, \mathbb{A})$). 
We denote by $\mathcal{P(S)}$ the space of probability distributions over $\mathcal{S}$, equipped with its Borel $\sigma$-algebra. 

\subsection{Risk Measures}

\par In the context of stochastic optimization, is it well-known that expectations are unable to capture ``risky" events (related to statistical variability, dispersion, or fat-tail behavior of the randomness). \emph{Risk-averse optimization} aims to minimize the risk associated with such events, which is captured by certain functionals known as \emph{risk measures}. For an extended discussion on risk-averse optimization, we refer the reader to \citet[Chapter 6]{stoch_programming_text_shapiro}.

% \par We let $(\Omega, \mathcal{F},P)$ be a complete base probability space supporting all randomness considered below, and consider an integrable random variable $Z = Z(\omega) \in \mathcal{L}_p(P, {\R})$,
% where $\mathcal{L}_p(P, \mathbb{R})$ is the space of $p$-integrable real-valued functions equipped with probability measure $P$.

\paragraph{Optimized Certainty Equivalents (OCEs)} Among different classes of risk measures, in this work we are interested in those which are in \emph{infimal convolution} form, i.e., $\rho(Z) = \inf_t \{ t + \mathbb{E}\left[g(Z-t)\right] \},$ assuming that $Z$ represents losses. A notable example is the so-called \emph{Conditional Value-at-Risk}
% \footnote{Let $F_Z(z) = P[Z \leq z]$ be the CDF $Z$. Then the \emph{Value-at-Risk} ($\varb$) at level $\beta \in (0, 1)$ is the left-side $(1-\beta)$ quantile of $Z$, i.e., $\varb[Z] = F_Z^{-1}(1-\beta)$, with values exceeding $\varb[Z]$ occurring with probability at most $\beta$.} 
at level $\beta$ (denoted by $\cvarb(Z)$), which is retrieved by setting $g(\cdot) = \frac{1}{\beta}(\cdot)_+$, for some $\beta \in (0,1)$. 
% Informally, we can think of this as a quantile (conditional) mean.
% , as illustrated in Figure \ref{fig:cvarb}. 

If $Z$ represents rewards (as often in RL), we consider risk measures in the \emph{supremal convolution} form, i.e., $\tilde{\rho}(Z) = \sup_t \{t + \mathbb{E}\left[g(Z-t)\right] \}$. For example, we can consider the reflected $-\cvarb(-Z)$ (for reward maximization) by setting $g(u) = -\frac{1}{\beta}(-u)_+$. In this notation, $g$ is called a \emph{utility function}. Under certain conditions on the utility (e.g., satisfied by $g(u) = -(-u)_+$), these risk measures are known as \emph{optimized certainty equivalents} (OCEs) (see \citet[Definition 2.2]{ben-tal_teboulle_oce_07}, where the reader can find many practical examples and properties). The class of OCEs is very rich (including many risk measures of practical value) and satisfies several important properties. 
% A smoothing strategy for generating smooth OCEs from non-smooth utilities can be found in \citet[Example 2]{epi_regularization_risk_meas_20}.

% \begin{figure}
% % \vskip 0.2in
% \begin{center}
%     \includegraphics[width=0.5\columnwidth]{figures/cvarbz-2.png}
%     \vspace{-15bp}
%     \caption{Illustration of $\cvarb[Z]$ ($Z$ represents losses). $\varb$ gives a threshold beyond which greater losses occur with probability of at most $\beta$. $\cvarb$ is the average value of these high losses. See also \citet[Section 6.2.4]{stoch_programming_text_shapiro} for additional details. 
%     }
%     \label{fig:cvarb}
% \end{center}
% \vskip -0.2in
% \end{figure}

\subsection{Risk-Neutral RL}
% \textcolor{red}{Here, we should introduce the action and state spaces, $\mathcal{A}$ and $\mathcal{S}$, respectively. Since we use the results of Paternain et. al., we should probably also assume that both these are compact (they assume this in the beginning of Section 2. Another issue: do we assume continuous space/actions? For example, Bonetti et. al. assume everything is continuous. My intuition is that our assumptions implicitly imply that we assume continuous state/actions. It might also be worth assuming here (early on) that the reward is bounded; this is assumed by both Paternain et. al. and Bonetti et. al.}
Let $\tau \in \mathbb{N} \cup \{0\}$ denote the time instant and let $\mathcal{S} \subset \R^n$ and $\mathcal{A} \subset \R^d$ be compact sets describing the possible states and actions of an agent.
Reinforcement learning is often modeled as a Markovian dynamical system where given a state and action pair at time $\tau$, say $(s_\tau, a_\tau)$, the next state distribution only depends on the current state-action pair (and is independent of $(s_{< \tau}, a_{< \tau})$). The agent chooses actions at each time instant using a policy $\pi \in \mathcal{P}(\mathcal{S})$, and the action taken by the agent results in rewards defined by the uniformly bounded function $r: \mathcal{S} \times \mathcal{A} \rightarrow \R$.
We next describe the standard  RL objective, that is, to maximize the (discounted, infinite) expected sum of rewards.

\begin{problem}[\textbf{Primal RL Formulation}]
\begin{equation}
    \begin{aligned}
        \hspace{-2bp}V^* \hspace{-1bp}\coloneqq \hspace{-1bp}
        &\sup_{\pi \in \mathcal{P(S)}} \Bigg\{V(\pi) \hspace{-1bp}\coloneqq\hspace{-1bp} \E \left[ \sum_{\tau=0}^\infty \gamma^\tau r(s_\tau^\pi, a_\tau^\pi) \right]\hspace{-1.5bp}\Bigg\} \hspace{-0.5bp},\hspace{-2bp} \label{eq:classic_rl} 
    \end{aligned}
\end{equation}
\end{problem}
where $(s_{\tau}^{\pi},a_{\tau}^{\pi})$ denotes the state-action vector evolving under policy $\pi$.
It is well-known that the above can be reformulated in terms of the (discounted) \emph{occupancy (or occupation) measure} of a policy, defined as the discounted mixture $\mathrm{d}\nu^\pi(\cdot, \bullet) = (1-\gamma)\sum_{\tau=0}^\infty \gamma^\tau \mathrm{d} p_\pi^\tau(\cdot, \bullet)$,
%\[ \nu(s, a) = (1-\gamma)\sum_{\tau=0}^\infty \gamma^\tau p_\pi^\tau(s, a) \]
% \begin{align}
%     \mathrm{d}\nu^\pi(\cdot, \bullet) = (1-\gamma)\sum_{\tau=0}^\infty \gamma^\tau \mathrm{d} p_\pi^\tau(\cdot, \bullet),\label{eq:occupation_measure}
% \end{align}
where $p_\pi^\tau$ denotes the Borel probability measure induced by vector $(s_\tau^\pi,a_\tau^\pi)$.
Let $\mathcal{R}$ denote the convex and weakly compact set of all occupation measures \cite{Borkar_convex_analytic_journal} induced by the policies $\pi \in \mathcal{P(S)}$.
We can then rewrite the RL problem as follows.
\begin{problem}[\textbf{Convex Analytic Dual Form\footnote{This notion of casting dynamic optimization problems into abstract ``static'' optimization problems over a closed convex set of measures (as above) is referred to as the \textit{convex analytic} approach by \citet{Borkar2002_convex_analytic_mdp}. Thus, we will refer to these equivalent problems \eqref{eq:classic_rl} and \eqref{eq:occupation_rl} as the ``primal'' and ``convex analytic dual'' formulations (not to be confused with duality in the Lagrangian sense) throughout.}}]
\begin{equation}
    \begin{aligned}
        V^* \coloneqq &\sup_{\nu^\pi \in \mathcal{R}} \bigg\{V(\nu^\pi) \coloneqq \frac{1}{1-\gamma} \cdot \E_{\nu^\pi} [r(s,a)]\bigg\}. \label{eq:occupation_rl}
    \end{aligned}
\end{equation}
\end{problem}
It is known that the occupation measure $\nu^\pi$ is in a one-to-one relationship with a policy $\pi$, i.e., if two policies have the same occupation measure, they must be the same (\citet{sutton_rl_textbook}).
% This is because the occupation measure is the unique solution to the Bellman flow constraint (e.g., see \citet{sutton_rl_textbook}). 
% For completeness, refer to the Appendix \ref{apx:occupancy_meas} for a derivation of their equivalence (see also \citet{constrained_rl_zero_duality_gap}). 

% This notion of casting dynamic optimization problems into abstract ``static'' optimization problems over a closed convex set of measures (as above) is referred to as the \textit{convex analytic} approach by \citet{Borkar2002_convex_analytic_mdp}. Thus, we will refer to these equivalent problems \eqref{eq:classic_rl} and \eqref{eq:occupation_rl} as the ``primal'' and ``convex analytic dual'' formulations (not to be confused with duality in the Lagrangian sense) throughout. 

\begin{table*}[!t]
\vspace{-8bp}
\caption{Comparison of the reward-based objective and the standard RL objective for general OCEs.}
\vspace{-4bp}
\label{table:cvar_primal_conv_dual}
\vskip 0.15in
\begin{center}
\begin{small}
\begin{sc}
\begin{tabular}{@{}c lcc@{}}
\toprule
\multicolumn{1}{l}{} & Formulation & Risk-Averse (Reward-Based) & Risk-Neutral \\ 
\midrule
& {Primal} 
% & $\max_{\pi, t} \E\left[ \sum_{\tau=0}^\infty \gamma^\tau \left(t - \frac{1}{\beta}(t - r(s_\tau, a_\tau))_+ \right) \right]$ \eqref{eq:primal_cvar_occupation_rl} 
& $\displaystyle{\sup_{\pi, t} ~\E\left[ \sum_{\tau=0}^\infty \gamma^\tau \left( t + g(r(s_{\tau}^{\pi},a_{\tau}^{\pi}) - t) \right) \right]}$ 
& $\displaystyle{\sup_\pi \E\left[ \sum_{\tau=0}^\infty \gamma^\tau r(s_{\tau}^{\pi},a_{\tau}^{\pi}) \right]}$  \\
% & & $\max_{\pi, t} \E\left[ \sum_{\tau=0}^\infty \gamma^\tau \left( t + g(r(s_\tau, a_\tau) - t) \right) \right]$ \eqref{eq:convex_an_primal_sup_conv} & \\
\midrule
& Dual 
% & $\max_{\nu^\pi, t} \frac{1}{1-\gamma} \E_\nu\left[ t - \frac{1}{\beta}(t - r(s,a))_+ \right]$ \eqref{eq:Cvar_occupation_RL} 
& $\displaystyle{\sup_{\nu^\pi, t} ~\frac{1}{1-\gamma} \E_{\nu^\pi}\left[ t + g(r(s,a) - t) \right]}$
& $\displaystyle{\sup_{\nu^\pi} \frac{1}{1-\gamma} \E_{\nu^\pi}[r(s,a)]}$  \\
% & (convex analytic) & $\max_{\nu^\pi, t} \frac{1}{1-\gamma} \E_\nu\left[ t + g(r(s,a) - t) \right]$ \eqref{eq:convex_an_dual_sup_conv} &  \\
\bottomrule
\end{tabular}
\end{sc}
\end{small}
\end{center}
\vskip -0.1in
\end{table*}

\subsection{Risk-Averse RL Formulations}

There are several different frameworks for modeling ``risk-awareness'' and ``safety'' in RL. Existing works which explore risk-averse RL formulations have often taken the approach of replacing the expectations in the ``primal'' formulation \eqref{eq:classic_rl} with risk measures, such as in the studies of \citet{huang2021convergence,NIPS2015_024d7f84, risk_contr_rl_percentile_risk}, among others:
\begin{equation}
    \begin{aligned}
        \sup_{\pi \in \mathcal{P(S)}} -\rho \left[ -\sum_{\tau=0}^\infty \gamma^\tau r(s_{\tau}^{\pi},a_{\tau}^{\pi}) \right]. \label{eq:primal_risk_averse} 
    \end{aligned}
\end{equation} 
Alternatively, one could generalize the tower property of expectations and iteratively evaluate a risk measure at each decision stage (e.g., see \citet{ruszczynski2010risk}). 
A less common formulation --of interest herein-- \emph{applies a risk measure on the occupancy measure}, rather than the original probability space, as in \citet{bonetti_risk_averse_rl_coherent_risk}:
\begin{equation}
    \begin{aligned}
        R^* &= \sup_{\nu^\pi \in \mathcal{R}} \frac{1}{1-\gamma} \cdot -\rho_{\nu^\pi(s,a)} (-r(s,a)), \label{eq:risk_aware_gen_rl}
    \end{aligned}
\end{equation}
where $\rho: \mathcal{L}_1(\nu^\pi, \R) \rightarrow \R$ is a finite-valued risk measure.
This choice \eqref{eq:risk_aware_gen_rl} enforces robustness over both space and time, capturing \textit{per-stage risk} in contrast to \eqref{eq:primal_risk_averse}; for details see Appendix \ref{appendix:robustness}.

\paragraph{OCE Formulation} By utilizing the relationship between \eqref{eq:classic_rl} and \eqref{eq:occupation_rl} for risk measures induced by OCEs with a utility $g$, given that the reward is bounded, we obtain\footnote{Going forward, the expectation or risk without subscripts is taken over the initial state and sample path distribution induced by the state transitions of the system and the policy $\pi$.} 
\begin{align*}
    \frac{1}{1-\gamma} \sup_{\nu^\pi} \sup_t \E_{\nu^\pi}\left[ t + g(r(s,a) - t) \right] ~\iff~ \sup_t \sup_\pi \E\left[ \sum_{\tau=0}^\infty \gamma^\tau \left( t + g(r(s_{\tau}^{\pi},a_{\tau}^{\pi}) - t) \right) \right] \,.
\end{align*}
% with the convex analytic dual form given by
% \begin{align*}
%     \frac{1}{1-\gamma} \sup_{\nu^\pi} \sup_t \E_{\nu^\pi}\left[ t + g(r(s,a) - t) \right], %\label{eq:convex_an_dual_sup_conv}
% \end{align*}
% which given that the reward is bounded is the same as 
% \begin{align*}
%     \sup_t \sup_\pi \E\left[ \sum_{\tau=0}^\infty \gamma^\tau \left( t + g(r(s_{\tau}^{\pi},a_{\tau}^{\pi}) - t) \right) \right]. %\label{eq:convex_an_primal_sup_conv}
% \end{align*}

This relationship is summarized in Table \ref{table:cvar_primal_conv_dual}. As one can see, the risk-averse problem involves a maximization over an additional variable $t$ which depends on $\pi$ (in this sequence). Importantly, given a fixed $t$, maximization over $\pi$ resembles a problem of the form of \eqref{eq:classic_rl} with a modified reward function. Thus one can take advantage of several existing algorithms which solve \eqref{eq:classic_rl}. This has been explored in the unconstrained setting of \citet[Algorithm 2]{bonetti_risk_averse_rl_coherent_risk}.
% , e.g., using a block-cyclic coordinate ascent algorithm.
% which alternates between optimizing $t$ given $\pi$, and then $\pi$ given $t$. 
%
% Summarizing, this reward-based formulation for risk-awareness has desirable properties including ease of implementation and robustness to per-stage risk.

\subsection{Constraints}
Given the discussion in the previous section
% the formulation 
% \eqref{eq:primal_cvar_occupation_rl}
% \eqref{eq:convex_an_primal_sup_conv} 
for risk-aware objectives, we can similarly describe constraints also of this form. 
% Thus a constrained version of an RL problem with risk-aware objectives and constraints can be described as $(P^* = )$
% \begin{equation}
% \begin{aligned}
%     &\sup_{\substack{\pi \in \mathcal{P(S)} \\ t_0 \in \R}} \E \left[ \sum_{\tau=0}^\infty \gamma^\tau \left(  t_0 - \frac{1}{\beta} (f_0(t_0) + g_0(r_0(s_\tau,\pi(s_\tau))-t_0) \right) \right] \\
%     &\tiny{\textnormal{s.t. }  \sup_{\substack{t_i \in \R \\ \forall i }}  \E \left[ \sum_{\tau=0}^\infty \gamma^\tau \left(  t_i - \frac{1}{\beta} (f_i(t_i) + g_i(r_i(s_\tau,\pi(s_\tau))-t_i) \right) \right] \geq c_i}
%     \label{eq:constr_rl_gen}
% \end{aligned}
% \end{equation}
% We can equivalently write this in the convex analytic dual form \eqref{eq:convex_an_dual_sup_conv}. 
It is common to address solving constrained problems through Lagrangian relaxation, which is amenable to numerical optimization, but (generally) only offers an approximate solution to the original constrained problem. We nonetheless show that the (partial) Lagrangian relaxation employed in this paper is \emph{exact} under certain constraint qualifications.

Here, we highlight that \citet{bonetti_risk_averse_rl_coherent_risk} do \textit{not} consider the {constrained RL setting} (which requires a dedicated analysis of the relation between the Lagrangian relaxation and the original constrained problem), while others using return-based risk objectives and constraints in the convex analytic ``primal'' space (e.g., \citet{risk_contr_rl_percentile_risk}) are not able to show any meaningful relation between the primal problem and their employed Lagrangian relaxation. 

\section{Main Results (Risk Constrained Learning)}\label{sec:main}
The framework we propose has favorable robustness properties, better capturing per-stage risk (see Appendix \ref{appendix:robustness}) for applications which are also time-sensitive. Now, we: (1) show that this formulation allows us to solve the original constrained problem through a partial Lagrangian relaxation which is practically implementable and (2) reduce the problem to an instance of stochastic minimax optimization which lends itself to non-asymptotic convergence analysis. 

\subsection{Problem Description and Equivalent Reformulations}
% We present our results using the $\cvarb$ for ease of presentation, but it should be noted that the results generalize to any OCE (cf. \citet[Definition 2.2]{ben-tal_teboulle_oce_07}).
To avoid generalities and to facilitate exposition, hereafter we exclusively consider the $\cvarb$. However, all results presented 
%below 
hold for general OCEs.
We also use the same $\beta$ for all constraints and the objective, but these need not be the same. We first write a constrained version of the problem 
\eqref{eq:risk_aware_gen_rl}:
% \eqref{eq:Cvar_occupation_RL}
% as:
%\begin{equation}
\begin{align}\label{eq:constr_rl_cvarb_implicit}
     &\sup_{\nu^\pi \in \mathcal{R}} \frac{1}{1-\gamma} \cdot -\cvarb_{\nu^\pi}(-r_0(s,a)) \ \ \ \textnormal{s.t.}  \ \ \frac{1}{1-\gamma} \cdot -\cvarb_{\nu^\pi}(-r_i(s,a)) \geq {c}_i, \ \forall i \in [m]\,,  
\end{align}
%\end{equation}
or, equivalently, as:
%\begin{equation}
%\begin{small}
\begin{align}\label{eq:constr_rl_cvarb}
     &\sup_{\pi \in \mathcal{P(S)}}\sup_{t_0 \in \R} \E \left[ \sum_{\tau=0}^\infty \gamma^\tau \left(r_0'(s_{\tau}^{\pi},a_{\tau}^{\pi},t_0) \right) \right] \ \ \ \textnormal{s.t.}  \ \ \sup_{t_i \in \R} \E \left[ \sum_{\tau=0}^\infty \gamma^\tau \left(r_i'(s_{\tau}^{\pi},a_{\tau}^{\pi},t_i) \right) \right] \geq c_i, \ \forall i \in [m]\,, 
\end{align}
%\end{equation}
%\end{small}
% \begin{equation}
%     \begin{aligned}
%          &\sup_{\pi \in \mathcal{P(S)}}\sup_{t_0 \in \R} \E \left[ \sum_{\tau=0}^\infty \gamma^\tau \left( t_0 - \frac{1}{\beta} (t_0 - r_0(s_\tau,\pi(s_\tau)))_+ \right) \right] \\
%         &\ \ \ \textnormal{s.t.}  \ \ -\cvarb_{\nu^\pi}(-r_i(s, a)) \geq c_i, \forall i = 1, \dots, m
%         \label{eq:constr_rl_cvarb}
%     \end{aligned}
% \end{equation}
\noindent where $r_i'(s,a,t) \coloneqq t - \frac{1}{\beta} (t - r_i(s,a))_+$, and $r_i: \mathcal{S} \times \mathcal{A} \rightarrow \R$ are the reward functions for $i \in \{0\}\cup[m]$. Note that $r_i'$ is written specifically for the $\cvarb$ here, but for general OCEs we simply replace $-\frac{1}{\beta}(-~\cdot)_+$ with $g(\cdot)$. Recall that, by standard convention, $r_i$ is bounded for all $i$.
% \footnote{Note $c_i = (1-\gamma)\tilde{c}_i$, $\forall i \in [m]$, to explicitly relate to \eqref{eq:constr_rl_cvarb_implicit}.}

\begin{remark} \label{remark: bounded t}
    Consider any OCE, and assume that $r_i$ are bounded for all $i \in \{0\}\cup[m].$ Then, by \citet[Proposition 2.1]{ben-tal_teboulle_oce_07}, we have that
   \begin{align}
    &\sup_{t \in \mathbb{R}} \E\left[ \sum_{\tau=0}^\infty \gamma^\tau \left( t + g(r_i(s_{\tau}^{\pi},a_{\tau}^{\pi}) - t) \right) \right] \nonumber
    = \sup_{t \in \mathcal{T}_i} \E\left[ \sum_{\tau=0}^\infty \gamma^\tau \left( t + g(r_i(s_{\tau}^{\pi},a_{\tau}^{\pi}) - t) \right) \right] \nonumber,
\end{align}
\noindent where $\mathcal{T}_i$ is a bounded interval containing the uniformly smallest and largest values of the reward $r_i(s_\tau,a_\tau)$. Thus, we can (equivalently) cast problem \eqref{eq:constr_rl_cvarb}, by restricting $t \coloneqq [t_0,t_1,\ldots,t_m] \in \mathcal{T}$, with $\mathcal{T} \subset \mathbb{R}^{m+1}$ some convex and compact set. This will be utilized throughout.
\end{remark}

\begin{lemma}\label{lem:move_sup_obj}
Problem \eqref{eq:constr_rl_cvarb} is equivalent to 
%\begin{equation}
%\small
    \begin{align}\label{eq:constr_rl_cvarb_ver2}
         \sup_{\substack{\pi \in \mathcal{P(S)} \\ t\in \mathcal{T}}} &\E \left[ \sum_{\tau=0}^\infty \gamma^\tau \left( r_0'(s_{\tau}^{\pi},a_{\tau}^{\pi},t_0) \right) \right]
        ~\textnormal{s.t. }~ \E \left[ \sum_{\tau=0}^\infty \gamma^\tau \left(r_i'(s_{\tau}^{\pi},a_{\tau}^{\pi},t_i) \right) \right] \geq c_i,\ \forall i \in [m]\,.   
    \end{align}
%\end{equation}
%\normalsize
\end{lemma}
The proof is based on \citet{risk_contr_rl_percentile_risk}, and given in Appendix \ref{appendix:proof_move_sup_obj}. Although it may seem somewhat redundant, 
%to write the problem \eqref{eq:constr_rl_cvarb} in the form \eqref{eq:constr_rl_cvarb_ver2}, the latter
\eqref{eq:constr_rl_cvarb_ver2} is a maximization problem jointly over the variables $\pi, t$, subject to functional inequality constraints.
Thus, the Lagrangian associated to \eqref{eq:constr_rl_cvarb_ver2} reads $\L(\pi, t, \lambda) \coloneqq \E[\hat{\L}(\pi, t, \lambda)]$, with
% \small
%\begin{small}
\begin{align*}
  % &\L(\pi, t, \lambda) \coloneqq \E[\hat{\L}(\pi, t, \lambda)] \numberthis\label{eq:lagrangian__} \\
   % &= \E \left[ \sum_{\tau=0}^\infty \gamma^\tau \left( t_0 - \frac{1}{\beta} (t_0 - r_0(s_\tau,a_\tau))_+ \right) \right] + \sum_{i=1}^m \lambda_i \cdot \left(c_i - \E \left[ \sum_{\tau=0}^\infty \gamma^\tau \left(t_i - \frac{1}{\beta} (t_i - r_i(s_\tau,a_\tau))_+\right) \right] \right)\\
    &\hat{\L}(\pi, t, \lambda) \coloneqq  \sum_{\tau=0}^\infty \gamma^\tau \bigg[\left( r_0'(s_{\tau}^{\pi},a_{\tau}^{\pi},t_0) \right)  
    - \sum^m_{i=1}\lambda_i c_i + \sum^m_{i=1}\lambda_i \left(r_i'(s_{\tau}^{\pi},a_{\tau}^{\pi},t_i)\right)\bigg].
    \numberthis\label{eq:lagrangian__}
    % &~ + (1-\gamma)\sum^m_{i=1}\lambda_i c_i - \sum^m_{i=1}\lambda_i \left(t_i - \frac{1}{\beta} (t_i - r_i(s_\tau,a_\tau))_+\right)\bigg) \bigg] 
\end{align*}
%\end{small}
%\normalsize
%\[ \L(\pi, t, \lambda) &= \E \left[ \sum_{\tau=0}^\infty \gamma^\tau \left( t_0 - \frac{1}{\beta} (t_0 - r_0(s_\tau,a_\tau))_+ \right) \right] + \sum_{i=1}^m \lambda_i \cdot \left(c_i - \E \left[ \sum_{\tau=0}^\infty \gamma^\tau \left(t_i - \frac{1}{\beta} (t_i - r_i(s_\tau,a_\tau))_+\right) \right] \right), \]
\noindent Then, the primal problem is $P^* = \sup_{\pi, t} \inf_{\lambda \geq 0} \L(\pi, t, \lambda)$. Note that the above, for fixed $t, \lambda$, is the Lagrangian for a risk-neutral constrained problem (similar to those considered by \citet{constrained_rl_zero_duality_gap}), parameterized by $t$, which  controls risk-aversion (in harmony with the choice of $\beta$). 
\subsection{Partial Lagrangian Relaxation}

\par {Next, we derive a partial Lagrangian relaxation for \eqref{eq:constr_rl_cvarb_ver2} (which is exact under a constraint qualification; cf. Assumption \ref{assum:constraint qual}). In turn, this allows us to cast the problem in a stochastic minimax optimization framework and derive an effective online algorithm for its solutions
% , which is non-asymptotically convergent under standard assumptions 
(cf. Section \ref{subsec:algo}).}
\paragraph{Partial Maximization Problem} Let us now consider problem \eqref{eq:constr_rl_cvarb_ver2} for fixed $t \in \mathcal{T}$, i.e., 
%\begin{equation}
    \begin{align}\label{eq:constr_rl_cvarb_fixed_t}
         \sup_{\pi \in \mathcal{P(S)}} & \E \left[ \sum_{\tau=0}^\infty \gamma^\tau \left( r_0'(s_{\tau}^{\pi},a_{\tau}^{\pi},t_0) \right) \right] 
        ~\textnormal{s.t. }~  \E \left[ \sum_{\tau=0}^\infty \gamma^\tau \left( r_i'(s_{\tau}^{\pi},a_{\tau}^{\pi},t_i) \right) \right] \geq c_i,\ \forall i \in [m]\,. 
    \end{align}
%\end{equation}
This problem exhibits strong duality under usual conditions.

\begin{prop}\label{prop:dual}
    Let $r_i$ be bounded functions for all $i \in \{0\} \cup [m]$. Assume that Slater's condition\footnote{Slater's condition requires that there exists a strictly feasible policy.} holds for \eqref{eq:constr_rl_cvarb_fixed_t}. Then, \eqref{eq:constr_rl_cvarb_fixed_t} exhibits strong duality, and thus 
    \[ \sup_\pi \inf_{\lambda \geq 0} \L(\pi, t, \lambda) = \inf_{\lambda \geq 0} \sup_\pi \L(\pi, t, \lambda). \]
\end{prop}
\noindent The proof of Proposition \ref{prop:dual} is given in Appendix \ref{appendix:proof_dual_prop}.

\paragraph{Maximizing over $t$} From Remark \ref{remark: bounded t}, we know that $t \in \mathcal{T}$ and $\mathcal{T}$ is a convex and compact set (since we assume bounded rewards). Note that Proposition \ref{prop:dual} holds for each fixed $t \in \mathcal{T}$, assuming Slater's condition is satisfied at this point. To proceed, it suffices to introduce another constraint qualification, as follows. 

%belongs to the image space of the reward functions, that is $t$ is a member of $\{(r_0(s,a), r_1(s,a), \dots, r_m(s,a)) : (s,a) \in \mathcal{S} \times \mathcal{A}\}$.
%Since we have assumed that the rewards are bounded, we have that 

\begin{assumption}[Constraint Qualification] \label{assum:constraint qual}
   Let $t^*$ be an argument which attains the maximum over the variable $t = [t_0, t_1, \dots, t_m]$ in \eqref{eq:constr_rl_cvarb_ver2}. There is a (non-singleton) convex and compact set $\mathcal{I} \subset \mathcal{T}$, with $t^* \in \mathcal{I}$, such that Slater's condition holds for \eqref{eq:constr_rl_cvarb_fixed_t}, for every $t \in \mathcal{I}$. 
\end{assumption}

Under Assumption \ref{assum:constraint qual}, we have that
\begin{align*}
\sup_{t \in \mathcal{I}}\sup_\pi \inf_{\lambda \geq 0} \L(\pi, t, \lambda) = \sup_{t \in \mathcal{I}} \inf_{\lambda \geq 0} \sup_\pi \L(\pi, t, \lambda). \label{eq:dual_prob_lagrangian}
\end{align*}

% Note that $t_{\beta_0}^*$ tends toward the smallest value of the support of $Z$ (thus also the smallest value that $t$ can take) as $\beta_0$ tends to 0. 
% \textcolor{red}{Here, it would be good to argue that we, almost without loss of generality, restrict $t$ to some compact set, based on the observations from Lemma 3.3. This is important later on, and also explains why we project $t$ onto a ball in our algorithm. For example, for the approximate duality gap to hold (assuming a parametrized policy), we need to assume that our altered rewards, $r_i'$ are bounded. This is only true if $t$ is restricted on some compact space.}
\paragraph{Parametrized Policies and Almost-zero Duality Gap}
In order to solve the (inner) policy optimization problem, we parametrize the policy by a vector $\theta \in \Theta \subset \mathbb{R}^p$, representing the coefficients of, say, a neural network, assuming that $\Theta$ is a compact set (which is the case in practice). We have shown that under Assumption \ref{assum:constraint qual}, the proposed partial Lagrangian relaxation is exact. We argue that this remains almost true (in the sense described below) even if we utilize parametrized policies. In what follows, we assume that $\pi_{\theta}$ is an $\epsilon$-universal parametrization of measures in $\mathcal{P}(\mathcal{S})$, according to \citet[Definition 1]{constrained_rl_zero_duality_gap}.
\begin{theorem}[Almost-zero Duality Gap] \label{thm: almost-zero duality gap}
Let Assumption \textnormal{\ref{assum:constraint qual}} hold, assume that $\pi_{\theta}$ is $\epsilon$-universal for measures in $\mathcal{P}(\mathcal{S})$, and that the policy-parametrized version of \eqref{eq:constr_rl_cvarb_ver2} is feasible. Then,
\begin{align*}
    &P^*(t^*) \coloneqq  \sup_{t \in \mathcal{I}}  \sup_{\pi}\inf_{\lambda \geq 0} \mathcal{L}(\pi,t,\lambda) ~\geq~  \sup_{t \in \mathcal{I}} \inf_{\lambda \geq 0}\sup_{\theta} \mathcal{L}(\pi_{\theta},t,\lambda) ~\geq~ P^*(t^*) -\mathcal{O}\left(\frac{\epsilon}{1-\gamma} \right).
\end{align*}
\end{theorem}
\noindent The proof is given in Appendix \ref{appendix: almost zero duality gap}. Theorem \ref{thm: almost-zero duality gap} states that solving the parameterized (partial) dual problem incurs negligible cost (depending on the universal parametrization error $\epsilon>0$) under mild assumptions and our usual constraint qualification (see Assumption \ref{assum:constraint qual}). 
% Next, we attempt to approximately solve this parameterized dual problem, highlighting its strong connection to the original primal (infinite-dimensional) problem.

\paragraph{The Practical Parametrized Partial Dual Problem}
\par In practice, we do not have access to the set $\mathcal{I}$ (assuming that Assumption \ref{assum:constraint qual} is satisfied), and thus, we may heuristically search over all $t \in \mathcal{T}$, noting that a reasonable initial guess can maximize our chances of landing in $\mathcal{I}$ (and thus solving the original constrained problem \eqref{eq:constr_rl_cvarb_ver2} (almost) exactly). Nonetheless, the proposed algorithmic framework does not rely on Assumption \ref{assum:constraint qual}. Indeed, in the absence of this, we recover a partial Lagrangian relaxation which yields an approximate solution to problem \eqref{eq:constr_rl_cvarb_ver2}. However, we highlight that, unlike Lagrangian relaxations considered in the literature, our approach is exact under Assumption \ref{assum:constraint qual}. 
 \par Additionally, from Proposition \ref{prop:dual}, we know that the optimal Lagrange multiplier is attained, assuming that Slater's condition holds for \eqref{eq:constr_rl_cvarb_ver2}.
 % , which is a weaker condition than Assumption \ref{assum:constraint qual}. 
 Thus, there is an optimal Lagrange multiplier associated to \eqref{eq:constr_rl_cvarb_ver2} within some sufficiently large convex and compact set $\Lambda \subset \mathbb{R}^m_+$.
 % , of diameter $D$. 
 In practice, the size of this set can be adjusted dynamically, if necessary. Hence, in what follows, we restrict the minimization over $\lambda \in \Lambda$. In light of the previous discussion, we focus on the following parametrized partial dual formulation:
\begin{equation} \label{eqn:param_dual_prob_lagrangian}D_{\theta}^* \coloneqq \sup_{t \in \mathcal{T}} \inf_{\lambda \in \Lambda} \underbrace{\sup_{\theta} \mathcal{L}(\pi_{\theta},t,\lambda)}_{\text{blackbox RL}}.\end{equation}
\paragraph{A Wrapper for Blackbox Unconstrained RL Algorithms} Notice that $\L(\pi_{\theta}, t, \lambda)$ is a linear combination of the expected discounted sum of adjusted reward functions which take the original $r_i$ to $r'_i(s,a,t_i) = t_i - \frac{1}{\beta}(t_i - r_i(s,a))_+$ for $i = 0, 1, \dots, m$. Thus, the innermost problem in \eqref{eqn:param_dual_prob_lagrangian} reduces to an unconstrained RL problem with a linear combination of terms 
% resembling the standard RL objective, 
defined by reward functions $r'_i$, given $\lambda$ and $t$.  
Hence, we have a practical algorithm for solving \eqref{eqn:param_dual_prob_lagrangian} which, if solved exactly for a starting point $t \in \mathcal{I}$, is arbitrarily close to the true primal solution $P^*$ of  \eqref{eq:constr_rl_cvarb}, under Assumption \ref{assum:constraint qual}. In practice, this problem is non-convex so we may only hope to find a locally optimal solution. Convergence to such a solution is established next.

\subsection{Algorithm and Analysis}\label{subsec:algo}
\par Next, we present the algorithmic procedure for solving problem \eqref{eqn:param_dual_prob_lagrangian},
% and establish an upper bound on the iteration complexity for convergence of the algorithm close to a stationary point.  Specifically, we consider a \textit{solver} that outputs the optimal parameterized policy as a function of $t$ and $\lambda$. To determine the optimal values of $t$ and $\lambda$, we adopt a stochastic gradient descent-ascent approach~\cite{pmlr-v119-lin20a}. 
in Algorithm \ref{alg:main}. 
Consider a policy oracle that precisely solves the problem $ \sup_{\theta \in \Theta} \L(\pi_{\theta}, t, \lambda)$, such that for every $(t, \lambda) \in \mathcal{T} \times \Lambda$, it returns a policy $\pi_{\theta^*(t, \lambda)} \in \arg \max_{\theta \in \Theta} \L(\pi_{\theta}, t, \lambda)$. Using such an oracle, problem \eqref{eqn:param_dual_prob_lagrangian} reads
\begin{align}
    D_{\theta}^* = \sup_{t \in \mathcal{T}} \inf_{\lambda \in \Lambda}  \underbrace{\L(\pi_{\theta^*(t,\lambda)},  t, \lambda)}_{\coloneqq -f(t,\lambda)}= -\inf_{t \in \mathcal{T}} \sup_{\lambda \in \Lambda}  f(t,\lambda)\,.
    \label{eq:minimax_dual}
\end{align}  
Having access to a solver which returns $\pi_{\theta^*(t,\lambda)}$ exactly is difficult. Instead, we assume the availability of an inexact oracle, as defined next. We justify the generality of this assumption in Appendix \ref{appendix:disc_local_opt} to save space in the main body.

\begin{assumption}[Local Solutions]\label{assum:local_opt}
    For any $(t, \lambda) \in \mathcal{T} \times \Lambda$, let $\theta^\star(t, \lambda)$ be a maximum of $\mathcal{L}(\pi_\theta, t, \lambda)$ and $\theta^\dagger(t, \lambda)$ be returned by a generic RL algorithm. Then, there exists a $\delta > 0$ such that $\forall (t, \lambda) \in \mathcal{T} \times \Lambda$, $\hat{\nabla}_{t,\lambda}\mathcal{\hat{L}}(\pi_{\theta^\dagger(t, \lambda)}, t, \lambda)  = \hat{\nabla}_{t,\lambda}\mathcal{\hat{L}}(\pi_{\theta^*(t, \lambda)}, t, \lambda) + b(\theta^*, \theta^\dagger,t,\lambda)$, with $\| b(\theta^*,\theta^\dagger,t,\lambda) \| \leq \delta$. 
\end{assumption}

 In what follows, we show that \eqref{eqn:param_dual_prob_lagrangian} can be reduced to a stochastic minimax optimization problem satisfying appropriate assumptions, and thus we can obtain a non-asymptotic convergence rate for Algorithm \ref{alg:main}. 
 For notational convenience, we define the function $\Phi (t) \coloneqq \sup_{\lambda \in \Lambda} f(t,\lambda)$. First, we show that $f$ is Lipschitz with respect to $t\in\mathbb{R}^{m+1}$ uniformly over $\lambda \in \Lambda$.

\begin{algorithm}[tb]
\caption{Reward-Based SGDA with Risk Constraints}\label{alg:main}
\begin{algorithmic}[1]
\State \textbf{Input:} Bounded reward functions $r_i: \mathcal{S} \times \mathcal{A} \mapsto \R$, discount factor $\gamma \in (0,1)$, 
% $\cvarb$ parameter $0 < \beta \leq 1$, 
step sizes $\eta_\lambda, \eta_t > 0$, batch size $n$.
%, \textcolor{red}{and a policy solver that computes} $\pi_{\theta^*} (t,\lambda) \in \arg \max_{\theta} \L(\pi_{\theta}, t, \lambda)$.
\State \textbf{Initialize:} $t^{(0)}$, $\lambda^{(0)}$
\For{each iteration $j = 1, 2, \dots, J$}
    \State Call the inexact oracle to obtain $\pi_{\theta^{(j)}}$ satisfying Assumption \ref{assum:local_opt}.
   % \begin{center}
    %\vspace{1bp}
    %    $\pi_{\theta^{(j)}} \leftarrow \arg \max_{\theta} \L(\pi_{\theta},  t^{(j)}, \lambda^{(j)})$
   % \vspace{2bp}
    %\end{center}
    \State Collect a batch of trajectories $\{ x_k \}_{k=1}^n$ using $\pi_{\theta^{(j)}}$ and compute sample (sub)gradients:
    \begin{center}
        \vspace{3bp}
        $\hat{\nabla}_{\lambda}\hat{\mathcal{L}}(\pi_{\theta^{(j)}}, t^{(j)}, \lambda^{(j)}),\quad \hat{\nabla}_{t}\hat{\mathcal{L}}(\pi_{\theta^{(j)}}, t^{(j)}, \lambda^{(j)})$.
        \vspace{1.5bp}
    \end{center}
    \State Update dual ($\lambda$) and auxiliary ($t$) variables:
    % and CVaR ($t$) variables:
    \begin{center}
    \vspace{1bp}
        $\lambda^{(j+1)} \leftarrow \Pi_{\Lambda}\left(\lambda^{(j)} - \eta_\lambda \hat{\nabla}_{\lambda}\hat{\mathcal{L}}(\pi_{\theta^{(j)}}, t^{(j)}, \lambda^{(j)})\right)$, \\
        $t^{(j+1)} \leftarrow \Pi_{\mathcal{T}} \left(t^{(j)} + \eta_t \hat{\nabla}_t \hat{\mathcal{L}}(\pi_{\theta^{(j)}}, t^{(j)}, \lambda^{(j)}) \right)$.
    \end{center}
\EndFor
\State \textbf{Return:} $\pi_{\theta^{(j^*)}}$, where we sample $j^* \sim \text{Unif}([J])$.
\end{algorithmic}
\end{algorithm}

% \par We assume access to a policy solver that precisely solves the problem $ \sup_{\theta \in \Theta} \L(\pi_{\theta}, t, \lambda)$ and, for every $(t, \lambda) \in \mathcal{T} \times \Lambda$, returns a policy $\pi_{\theta^*(t, \lambda)} \in \arg \max_{\theta \in \Theta} \L(\pi_{\theta}, t, \lambda)$. Under this assumption, problem \eqref{eqn:param_dual_prob_lagrangian} reads
% \begin{align*}
%     D_{\theta}^* = \sup_{t \in \mathcal{T}} \inf_{\lambda \in \Lambda}  \underbrace{\L(\pi_{\theta^*(t,\lambda)},  t, \lambda)}_{\coloneqq f(t,\lambda)}= \inf_{t \in \mathcal{T}} \sup_{\lambda \in \Lambda}  -f(t,\lambda).
% \end{align*}  In what follows, we show that this problem can be reduced to a stochastic minimax optimization problem, and thus we can readily obtain a non-asymptotic convergence rate for Algorithm \ref{alg:main}. For notational convenience, we define the function $\Phi (t) \coloneqq \sup_{\lambda \in \Lambda} -f(t,\lambda)$. First, we show that $f$ is Lipschitz with respect to $t\in\mathbb{R}^{m+1}$ uniformly over $\lambda \in \Lambda$.

\begin{lemma}\label{lemma:lipschitz}
    There exists a constant $C>0$ such that \begin{align*}
        \big| \textstyle{ \sup_\pi \mc{L}(\pi , t_1,\lambda) -  \sup_\pi \mc{L}(\pi , t_2,\lambda)} \big|\leq C \Vert t_1 -t_2 \Vert_2
    \end{align*} for all pairs $t_1 ,t_2\in\mathbb{R}^{m+1}$ and $\lambda \in \Lambda$.
\end{lemma}
\noindent We prove Lemma \ref{lemma:lipschitz} (which readily applies to parametrized policies as well) in Appendix \ref{appendix:proof_lemma_sup_ineq}. Following~\citet{pmlr-v119-lin20a}, we analyze our method under the assumption of Lipschitz smoothness of the function $\mathcal{L}(\pi_{\theta^*(t,\lambda)},t,\lambda)$, for $\theta^*(t,\lambda) \in \arg\max_{\theta \in \Theta} \mathcal{L}(\pi_{\theta},t,\lambda)$, which we state below.
\begin{assumption}[Lipschitz Smoothness]\label{assump:3}
The function $f(t,\lambda) \equiv -\mc{L}(\pi_{\theta^*(t,\lambda)},t ,\lambda)$ is $\ell$-smooth over $\mathcal{T} \times \Lambda$, where $\theta^*(t,\lambda) \in \arg\max_{\theta \in \Theta} \mathcal{L}(\pi_{\theta},t,\lambda)$ is an arbitrary selection. 
\end{assumption}
\begin{remark}
    We note that Assumption \ref{assump:3} is not particularly restrictive in our setting, and it holds under several general conditions, as outlined in Appendix \ref{appendix:conditions_for_Lipschitz_smoothness}. Nonetheless, if not satisfied, we can still enforce this assumption by making a slight algorithmic adjustment involving the appropriate addition of small uniform noise.
    % to the Lagrangian function. 
    For a comprehensive discussion we again refer the reader to Appendix \ref{appendix:conditions_for_Lipschitz_smoothness}, where we also present an explicit formula for computing the (exact) sample gradients. 
    % in Algorithm \ref{alg:main}.
\end{remark}

\iffalse
\par An important detail to note is that if $f(\cdot,\cdot)$ is differentiable (which is, under Assumption \ref{assump:3}), we can readily compute its gradient, since it is necessarily true that:
\[ \nabla_{t,\lambda} f(t,\lambda) = \nabla_{t,\lambda} \mathcal{L}(\pi_{\theta},t,\lambda)\vert_{\pi_{\theta} = \pi_{\theta^*(t,\lambda)}},\]
    \noindent where $\theta^*(t,\lambda) \in \arg\max_{\theta \in \Theta} \mathcal{L}(\pi_{\theta},t,\lambda)$ is an arbitrary selection. This is a direct consequence of Danskin's theorem (e.g., see Theorem 9.26 of \citet{stoch_programming_text_shapiro}).
\fi 
\par As appears in prior work~\cite{pmlr-v119-lin20a}, we make the standard (and, in our case, mild) assumption that the exact stochastic (sub)gradient oracles are unbiased and have at most $\sigma^2$ variance.
\begin{assumption}[(Sub)gradient oracles]\label{assump:1} For all $(t,\lambda) \in \mathcal{T}\times \Lambda$, we assume that
    \begin{align*}
        \E [\hat{\nabla}_{t,\lambda} \hat{\mathcal{L}}(\pi^*(t,\lambda),t,\lambda) + \nabla f (\lambda, t)]=0, \quad
        \E [\Vert  \hat{\nabla}_{t,\lambda} \hat{\mathcal{L}}(\pi^*(t,\lambda),t,\lambda) + \nabla f (\lambda, t) \Vert^2_2 ] \leq \sigma^2. 
    \end{align*}
\end{assumption}

Given the imposed assumptions, we now present the convergence guarantees of Algorithm \ref{alg:main}.  In the context of nonsmooth weakly convex optimization, the Moreau envelope serves as a measure of closeness to stationarity. With this in mind, we introduce the definition of an $\epsilon$-stationary point.
% by using the Moreau envelope.

%To that end, we first introduce some necessary notation. Given some parameter $\delta >0$, we define the Moreau envelope of $\Phi$ as $\Phi_{\delta} \colon \mathbb{R}^{m+1} \rightarrow \mathbb{R}$, where $\Phi_{\delta}(x) \coloneqq \inf_{w \in \mathbb{R}^{m+1}} \left\{\Phi(w) + \frac{1}{2\delta} \|w- x\|^2\right\}.$
%\noindent It is well-known that the Moreau envelope is a smoothed version (with smoothing parameter $\delta$) of the function $\Phi(\cdot)$ (seen as a special case of infimal convolution).

\begin{definition}[$\epsilon$-stationary point] \label{def:Stationary_point} A point $x$ is $\epsilon$-stationary if $\Vert \nabla \Phi_{1/2\ell}(x)\Vert\leq \epsilon $, where, given some $\delta > 0$, we define the Moreau envelope as $\Phi_{\delta}(x) \coloneqq \inf_{w \in \mathbb{R}^{m+1}} \left\{\Phi(w) + \frac{1}{2\delta} \|w- x\|^2\right\}.$
\end{definition}
% \kn{Introduce the notion of $\epsilon$ stationary point on the gradient of the  envelope of $\Phi$}
When $\epsilon = 0$, a stationary point of the Moreau envelope $\Phi_{\delta}(\cdot)$ is also stationary for $\Phi(\cdot)$. For $\epsilon > 0$, one can argue that an $\epsilon$-stationary point for $\Phi_{1/2\ell}$ is close to an $\epsilon$-stationary point of $\Phi(\cdot)$, as long as $\Phi(\cdot)$ is $\ell$-weakly convex (noting that, under our assumptions, $\Phi(\cdot)$ is $\ell$-weakly convex). We proceed with a bound on the iteration complexity of Algorithm \ref{alg:main}.
\begin{theorem}\label{thm:main}
Let the step-sizes $\eta_t$ and $\eta_\lambda$ in Algorithm \textnormal{\ref{alg:main}} be small enough as in Eq. \eqref{eq:eta_t} and \eqref{eq:eta_l}, respectively, and batch size $n=1$.
Define the quantities $\hat{\Delta}_\Phi \coloneqq \Phi_{1/2\ell}(t^{(0)}) - \min_t \Phi_{1/2\ell} (t) $ and $\hat{\Delta}_0\coloneqq \Phi (t^{(0)})-f(\lambda^{(0)},t^{(0)})$. 
If Assumptions \textnormal{\ref{assum:local_opt}}, \textnormal{\ref{assump:3}} and \textnormal{\ref{assump:1}} hold, the iteration complexity for recovering an $\mathcal{O}(\sqrt{\epsilon^2 + \delta \ell( \mathrm{diam}(\mathcal{T})+ \mathrm{diam}(\Lambda))})$-stationary point (Definition \textnormal{\ref{def:Stationary_point}}) is of order 
\begin{align*}
    \mathcal{O} \inparen{\inparen{\frac{\ell^3 (C^2 + \sigma^2 + \delta^2)(\mathrm{diam}(\Lambda))^2 \cdot \hat{\Delta}_\Phi}{\epsilon^6} + \frac{\ell^3(\mathrm{diam}(\Lambda))^2\cdot \hat{\Delta}_0}{\epsilon^4} } \max \inbrace{1, \frac{\sigma^2 + \delta^2}{\epsilon^2}} }\,.
    % \mc{O}\lp \frac{\ell^3(C^2+\sigma^2)D^2}{\epsilon^6}+\frac{\ell^3D^2}{\epsilon^4}\max\left\{ 1, \frac{\sigma^2}{\epsilon^2} \right\} \rp.
    %\mc{O}\lp \frac{\ell^3(C^2+\sigma^2)D^2}{\epsilon^6}+\frac{\ell^3D^2}{\epsilon^4}\max\{ 1, \frac{\sigma^2}{\epsilon^2} \} \rp
\end{align*}
\end{theorem}

Note that this result says we only need single trajectories ($n=1$). Thus our algorithm can be run \emph{online}.
% , exhibiting a standard $\mathcal{O}(\epsilon^{-8})$ complexity, in the context of nonconvex-concave stochastic optimization. 
The proof, which is detailed in Appendix \ref{appendix:main_thm}, extends the developments of \cite{pmlr-v119-lin20a}, by incorporating biased (sub)gradient samples, occurring due to the use of inexact oracles. If $\delta = \mathcal{O}(\epsilon^2)$, then we recover an $\epsilon$-stationary point and the complexity of \cite{pmlr-v119-lin20a}.

\section{Experiments}

We conduct experiments on locomotion tasks from the Safety-Gymnasium benchmark~\citep{safety_gym}. 

\subsection{Setup}
We consider two constraints in the environment: safe navigation and safe velocity. The first involves simple problems with low-dimensional state and action spaces, where the agent must navigate under rules (e.g., along a path or without hitting obstacles). The second is more challenging, involving standard high-dimensional control tasks based on MuJoCo-v4 agents~\citep{mujoco}. At each timestep, the agent receives a cost of +1 if it violates a constraint and 0 otherwise. These violations do not affect environment dynamics—they neither terminate episodes nor alter transitions. Importantly, our experiments do not access this cost directly during training; instead, the agent manages risk by being risk-aware in its velocities.

To simulate risk and uncertainty in an otherwise deterministic setting (where identical actions yield identical outcomes), we inject zero-mean Gaussian noise with a standard deviation of 0.05 (5\% of the action range) into all agent actions during both training and evaluation. This adds stochasticity to the agent’s actions and introduces risk.

\paragraph{Experimental Goals and Motivation}
Our objective is to evaluate whether the proposed risk-sensitive method ensures stable convergence of the dual variable ($\lambda$) and the auxiliary variable ($t$), while also inducing safe and meaningful behavioral changes to the agent. To emphasize convergence, we train agents for extended durations (5–18M steps) to provide a thorough \textit{proof-of-work}. Specifically, we aim to: (1) verify that $t$ converges to the empirical $\cvarb$ of the post-training constraint violation distribution (e.g., safe velocity violations), (2) confirm that $\lambda$ stabilizes to enforce the constraint, and (3) assess whether the policy maximizes reward while maintaining safe velocities.

\paragraph{Optimization with CVaR Constraints}
Let $\upsilon: \mathcal{S} \times \mathcal{A} \rightarrow \R$ denote the constraint-quantifying function (e.g., velocity). We formulate the following constrained optimization problem:
\[
    \sup_{\pi \in \mathcal{P(S)}} \E \left[ \sum_{\tau=0}^\infty \gamma^\tau r(s_\tau, \pi(s_\tau)) \right] \quad \textnormal{s.t.} \quad \cvarb_{\nu^\pi}(\upsilon(s,a)) \leq c,
\]
where \(c\) is a user-defined threshold and $\beta$ is the CVaR parameter. Our framework is general and allows risk-neutral objectives to be combined with risk-aware constraints. To solve this, we transform the reward at each timestep using the CVaR-based Lagrangian (see Appendix \ref{app:implementation} for the derivation):
\[
    r(s_\tau, \pi(s_\tau)) + \lambda_i \left( c + t - \frac{1}{\beta}\left(t + \upsilon(s_\tau, \pi(s_\tau))\right)_+ \right).
\]
In safe navigation, the constraint variable reduces to $\upsilon(s,a) \in \{0,1\}$, since navigation is not quantifiable in the same way as velocity. This makes the optimization problem more difficult, as the penalty term (the $\lambda$ multiplier) becomes sparse.

Proximal Policy Optimization (PPO)~\citep{ppo} is used as a black-box solver to optimize the policy \(\pi\), implementing line 4 of Algorithm~\ref{alg:main}. The dual variable \(\lambda\) and auxiliary \(t\) are updated via stochastic gradient steps following the rest of Algorithm~\ref{alg:main}. For further experimental details, see Appendix \ref{appendix:experimental_details}.

\begin{table}[t]
\centering
\caption{Cumulative episodic evaluation costs and rewards of the converged agents, computed as the mean of the last 10 episodes. The PPO baseline is unconstrained and risk-neutral. We use the simplest Safety-Gymnasium agent, Point at level 1 difficulty, to avoid additional challenges and isolate the effects of the constraints.}
\label{tab:nav_results}
\begin{tabular}{@{}lccccc@{}}
\toprule
 & & \multicolumn{2}{c}{\textbf{Cost $\downarrow$}} & \multicolumn{2}{c}{\textbf{Reward $\uparrow$}} \\
\cmidrule(lr){3-4} \cmidrule(lr){5-6}
Environment & \# Training Steps & PPO & Ours & PPO & Ours \\
\midrule
Button & 5M & 150.76 & 0.0 & 24.29 & 2.58 \\
Circle & 5M & 206.74 & 0.0 & 60.18 & 39.19 \\
Goal & 10M & 45.09 & 0.0 & 21.89 & 13.56 \\
Push & 5M & 38.48 & 0.0 & 0.93 & 2.42 \\
\bottomrule
\end{tabular}
\end{table}

\begin{figure*}[!t]
    \centering
    \small{
    \begin{align*}
        &\text{{\brown} Evaluation reward}  \qquad &&\text{{\blue} $t$} \qquad &&\text{{\purple} $\lambda$} \\
        &\text{{\black} Velocity threshold} \qquad &&\text{{\red} $\beta$-upper quantile} \qquad &&\text{{\green} Converged $t$-value}
    \end{align*}
    }
    \subfigure[Environment: HalfCheetah \quad $c = 1.450$ \quad $\beta$-upper quantile: $1.419$ \quad Converged $t$-value: $1.417$]{
                \includegraphics[width=0.238\linewidth]{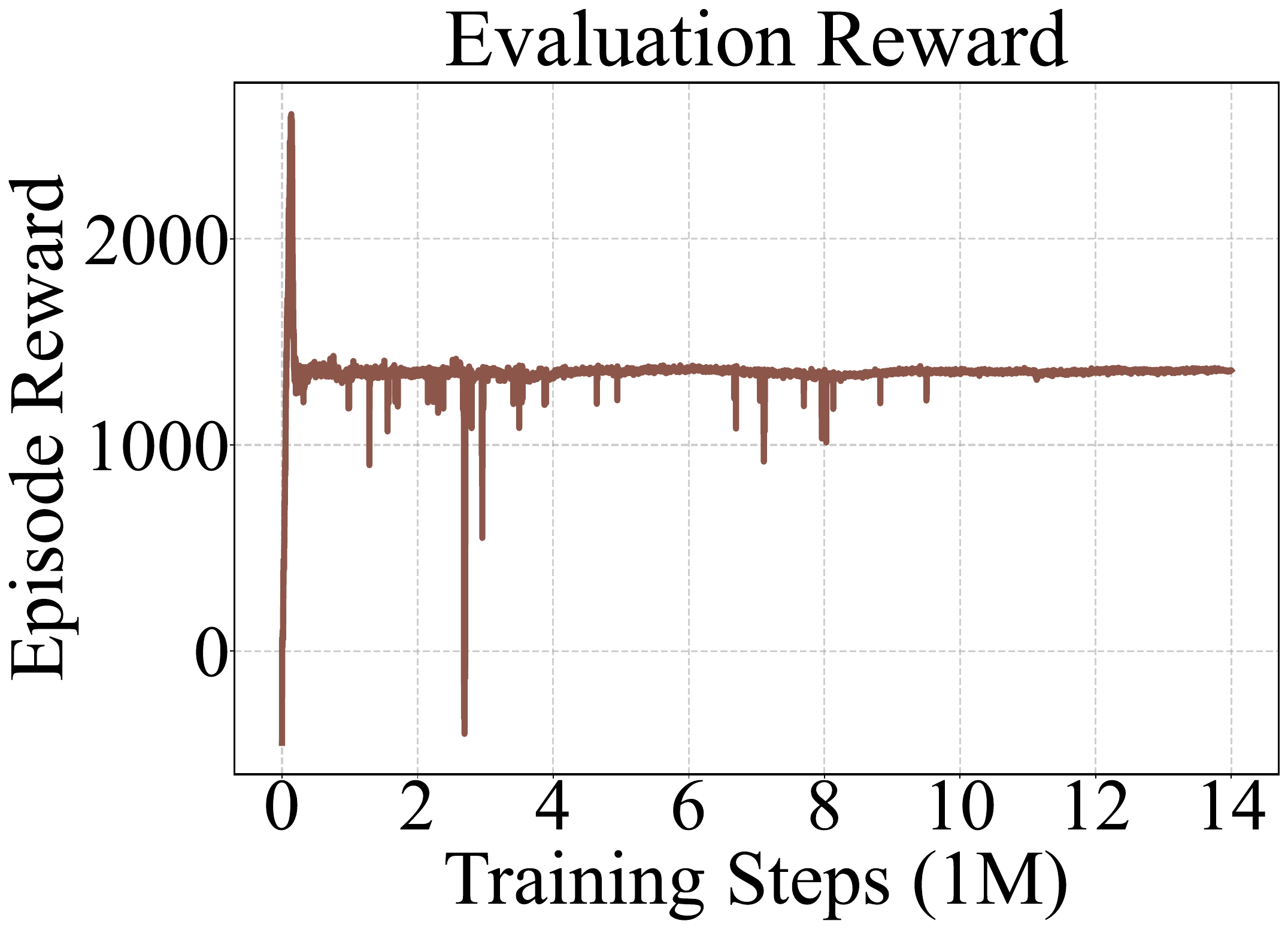}%
                \includegraphics[width=0.242\linewidth]{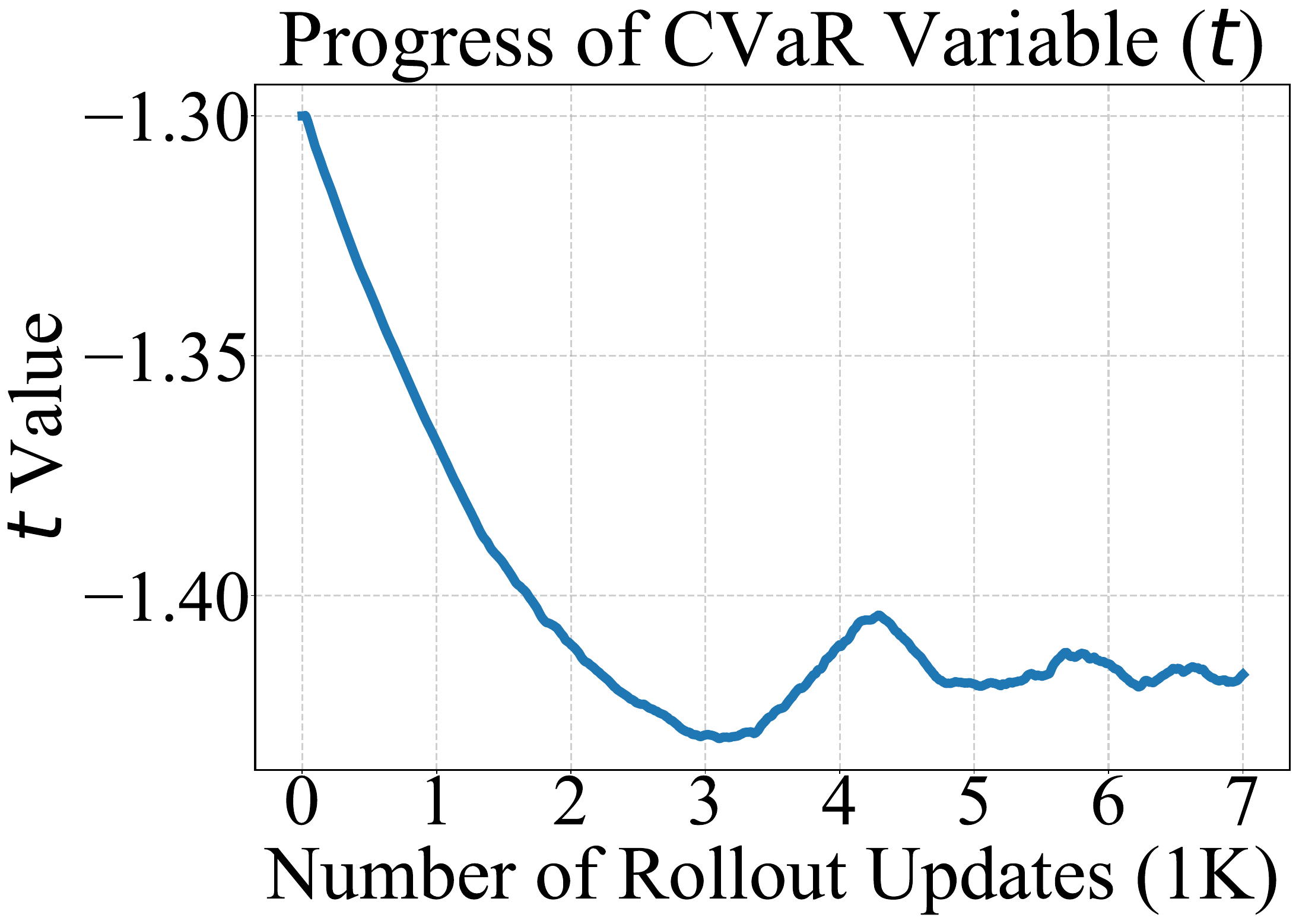}%
                \includegraphics[width=0.2328\linewidth]{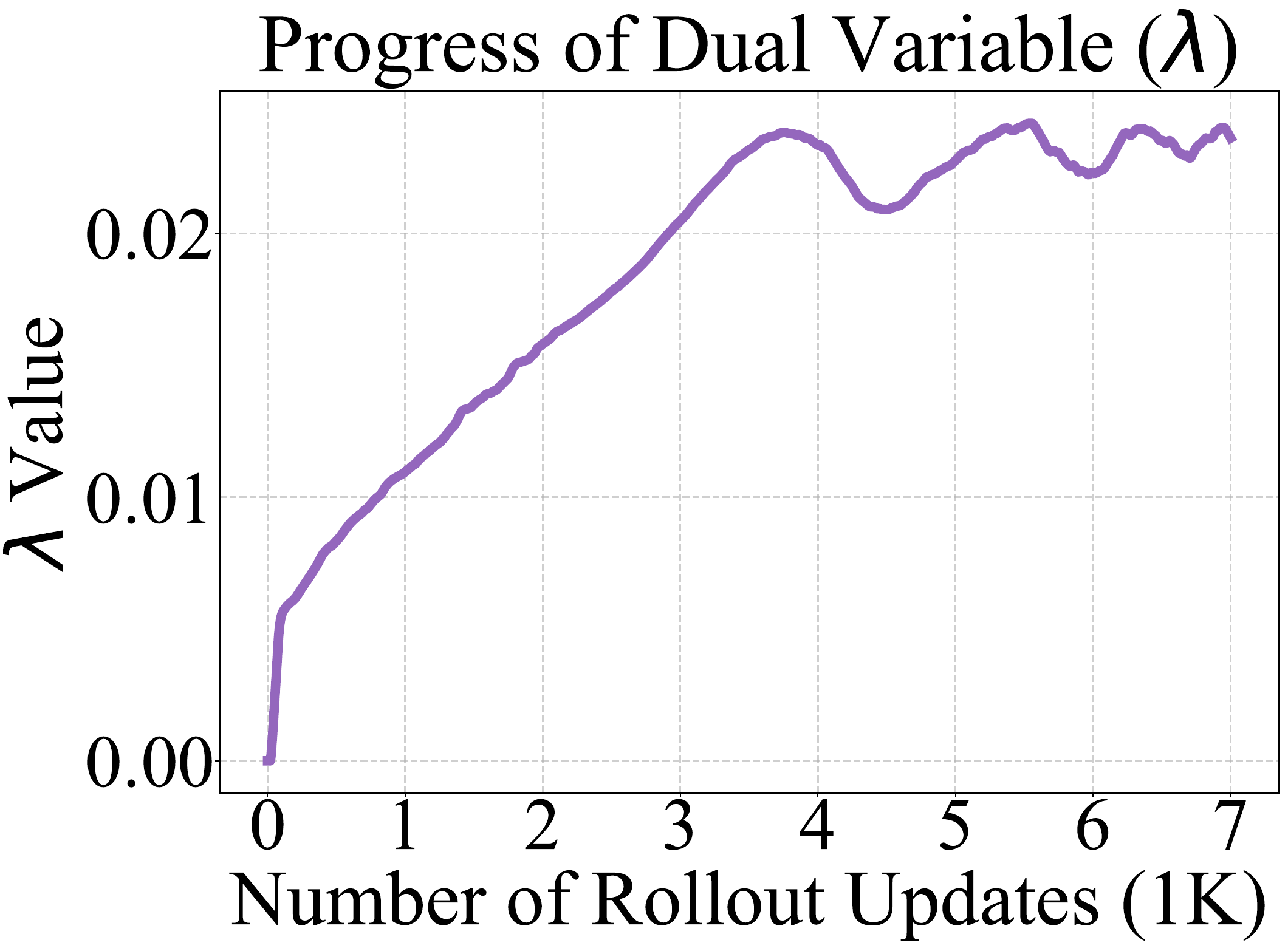}%
                \includegraphics[width=0.2095\linewidth]{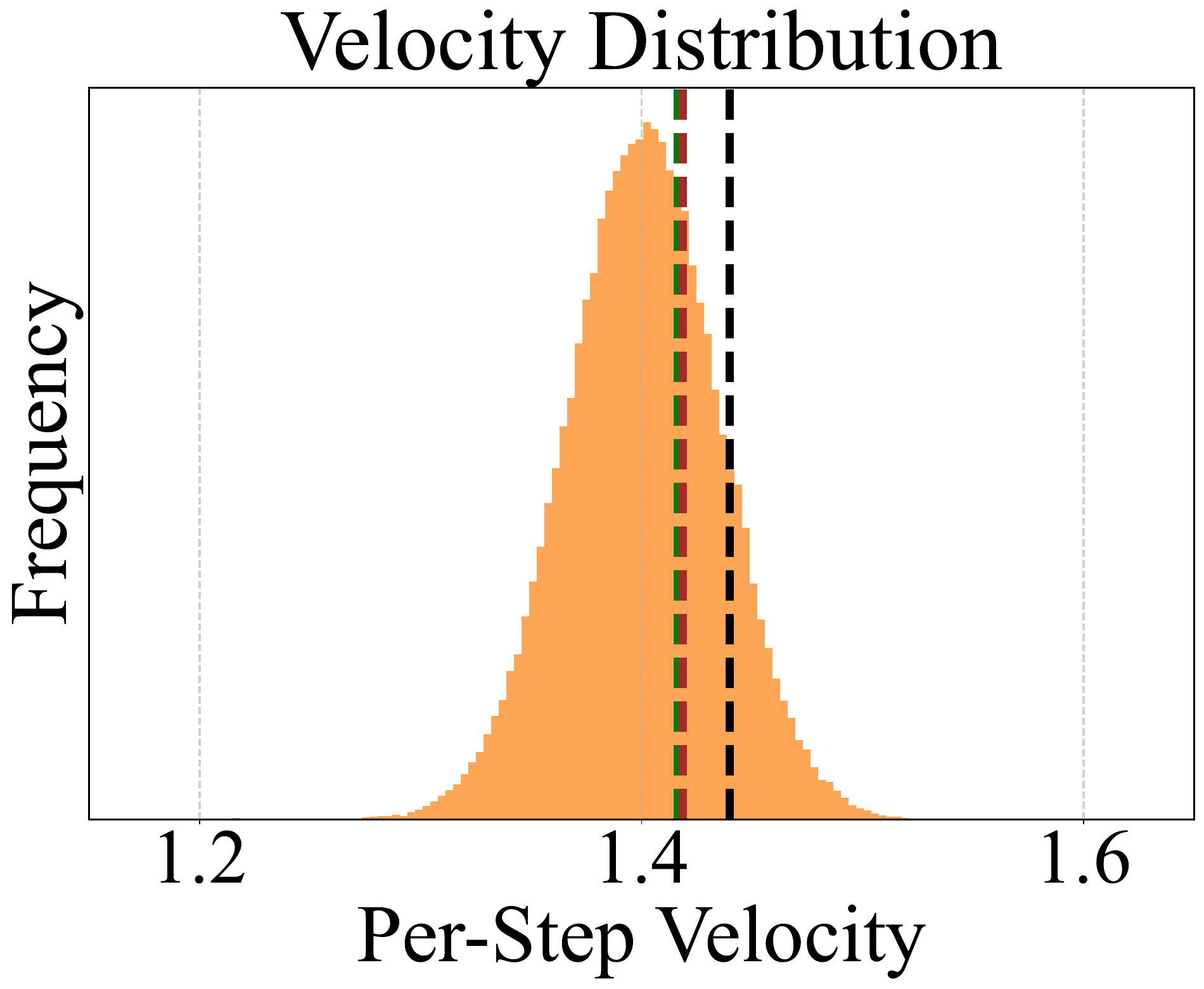}%        
        }
	\subfigure[Environment: Hopper \quad $c = 0.373$ \quad $\beta$-upper quantile: $0.370$ \quad Converged $t$-value: $0.370$]{
                \includegraphics[width=0.238\linewidth]{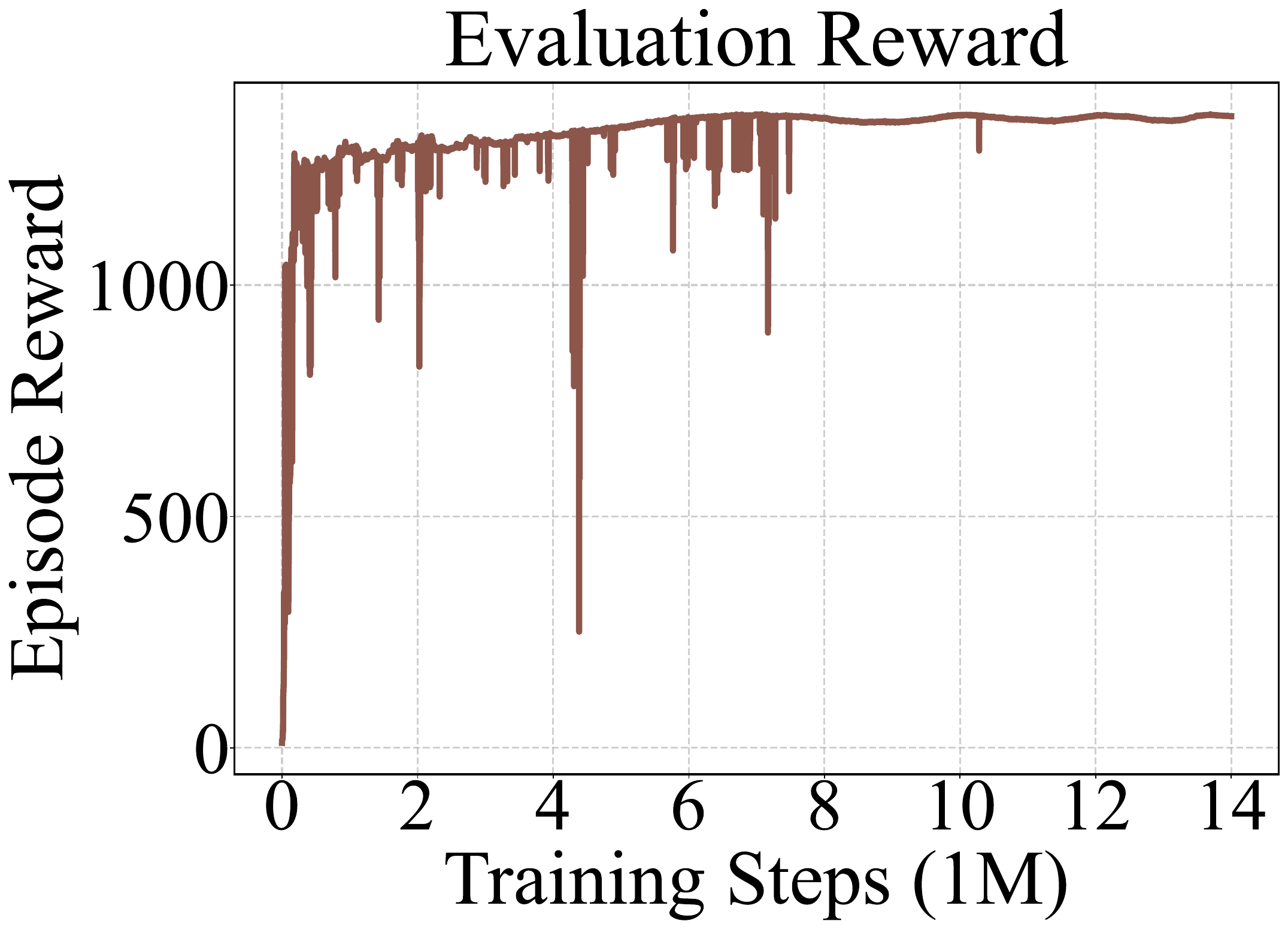}%
                \includegraphics[width=0.242\linewidth]{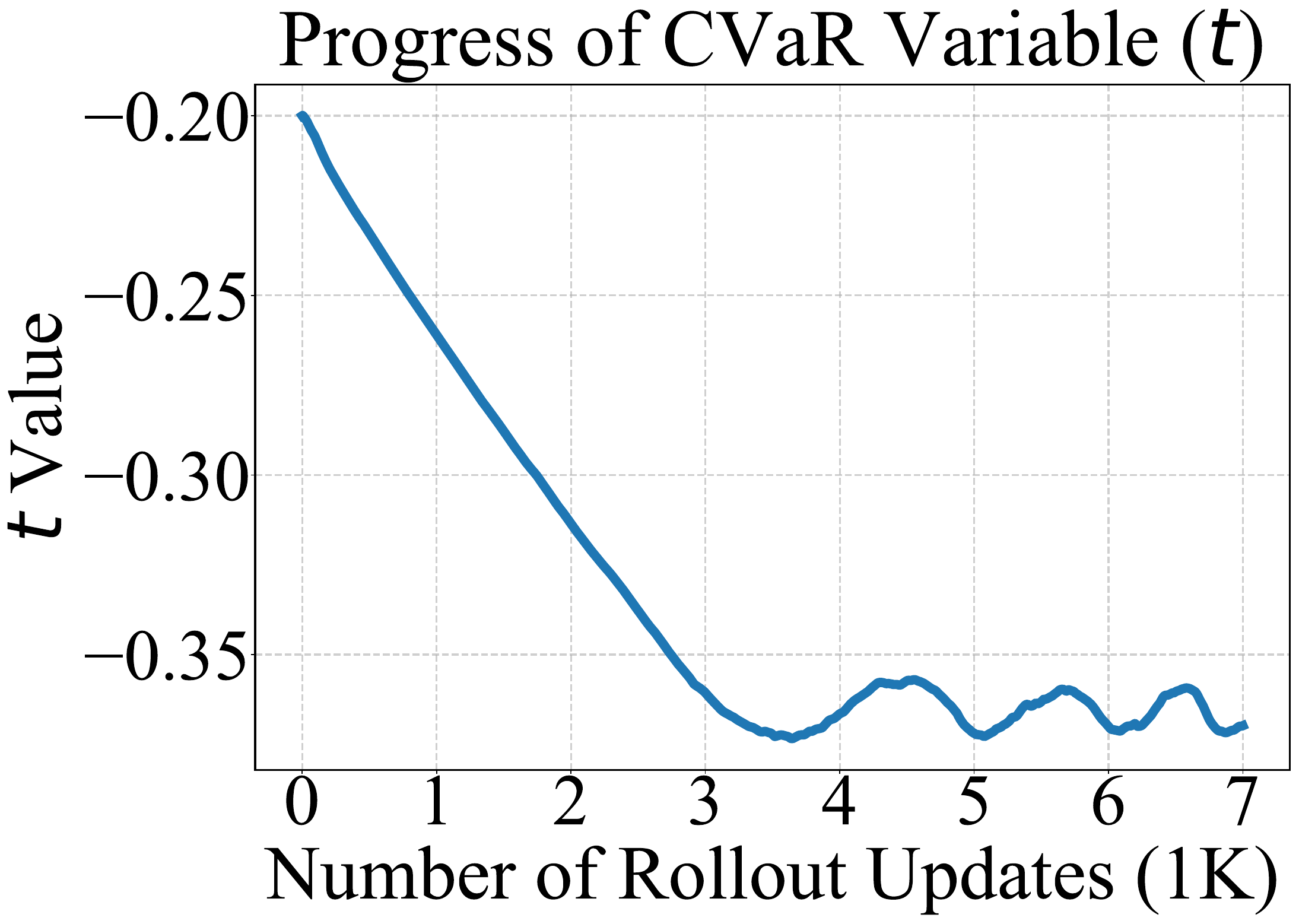}%
                \includegraphics[width=0.233\linewidth]{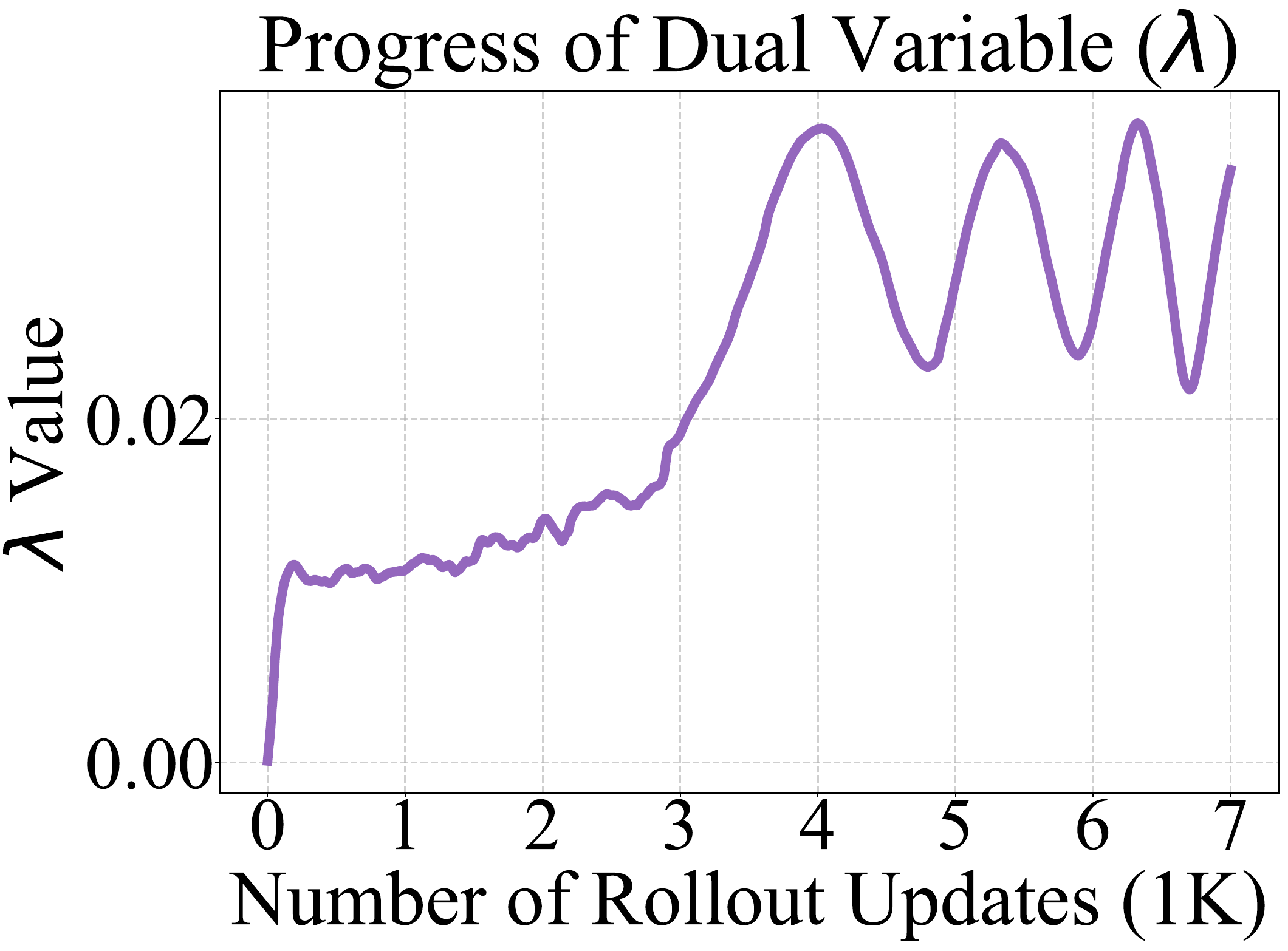}%
                \includegraphics[width=0.21\linewidth]{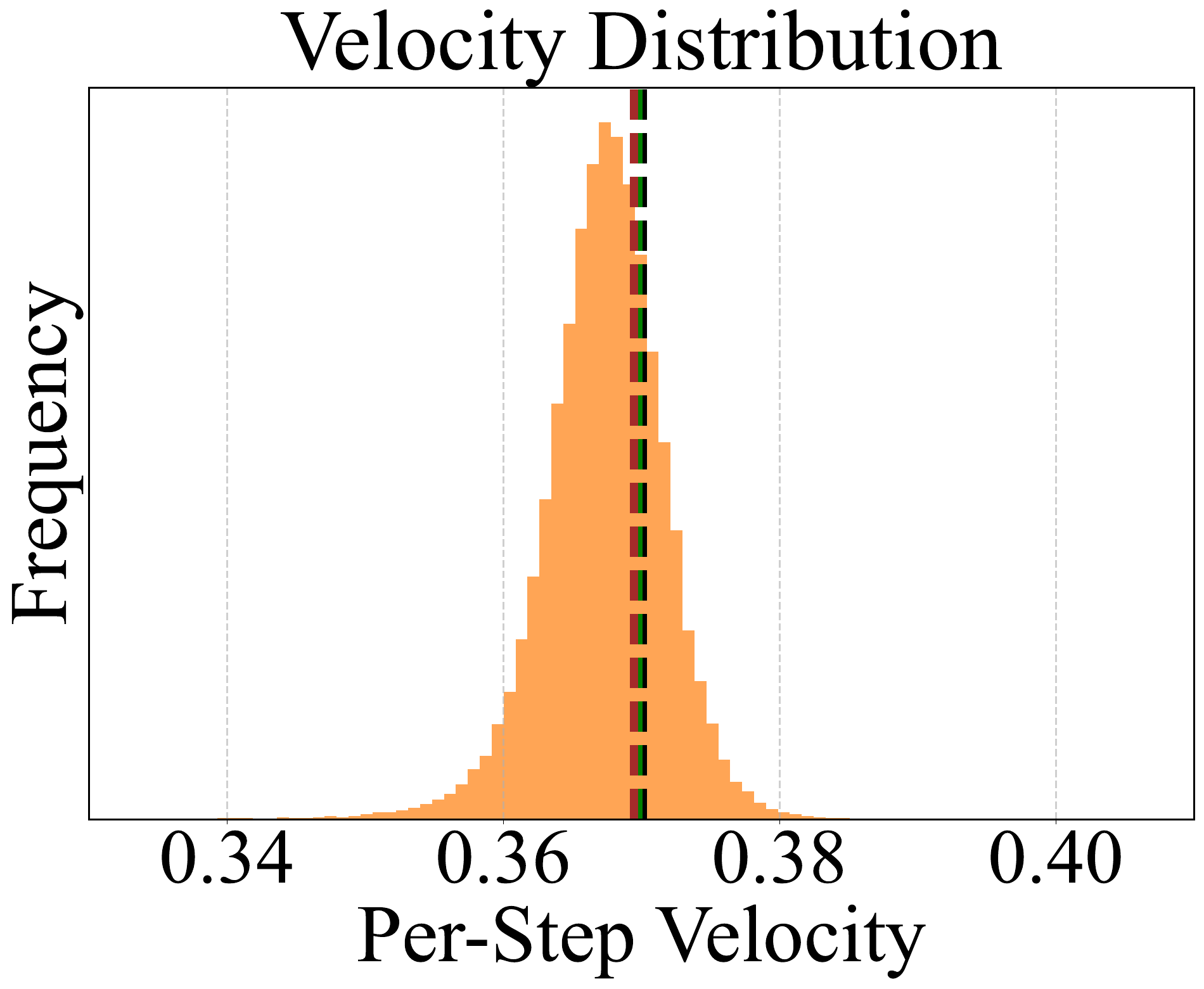}%        
        }
	\subfigure[Environment: Swimmer \quad $c = 0.228$ \quad $\beta$-upper quantile: $0.248$ \quad Converged $t$-value: $0.207$]{
                \includegraphics[width=0.238\linewidth]{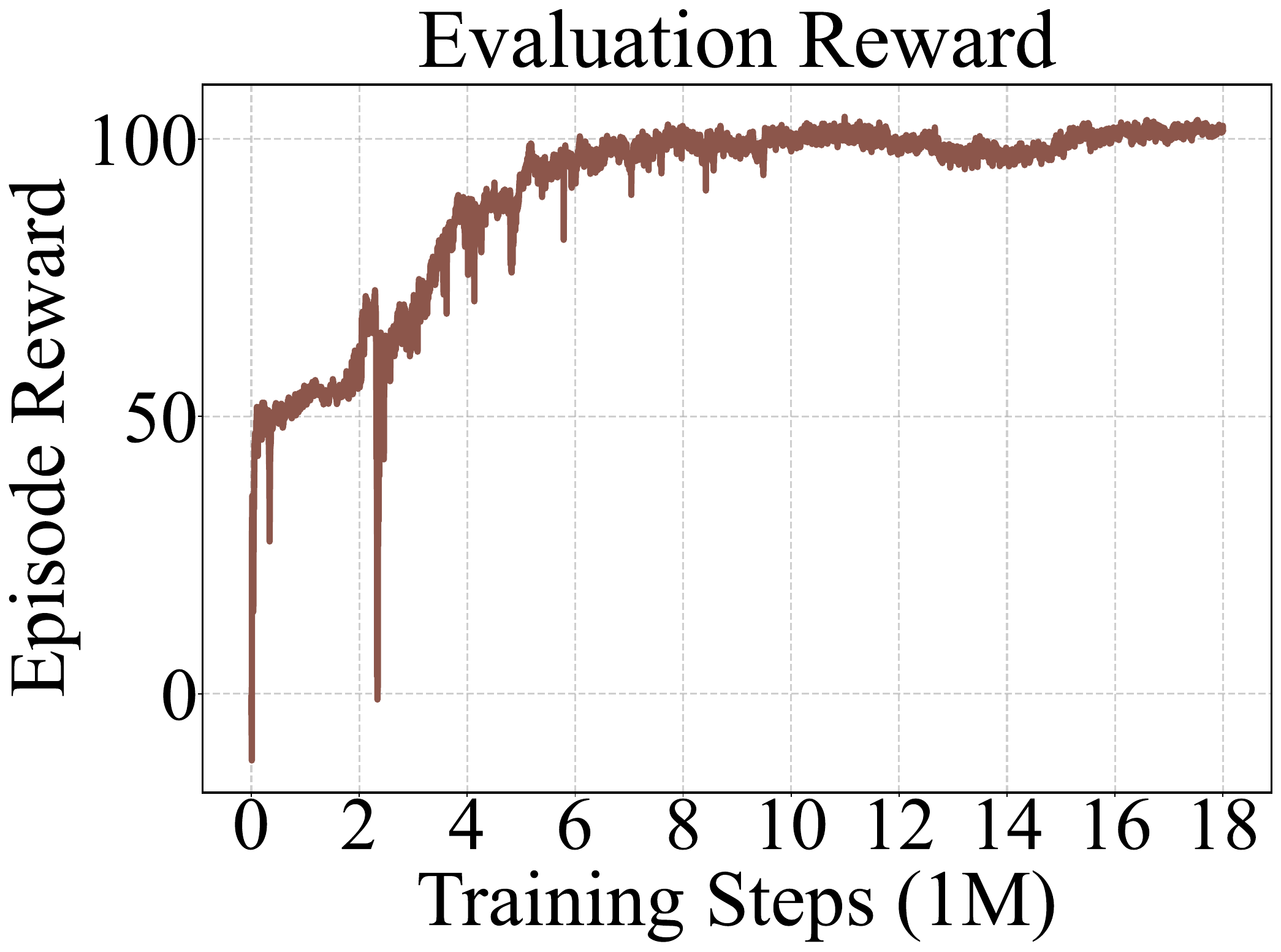}%
                \includegraphics[width=0.242\linewidth]{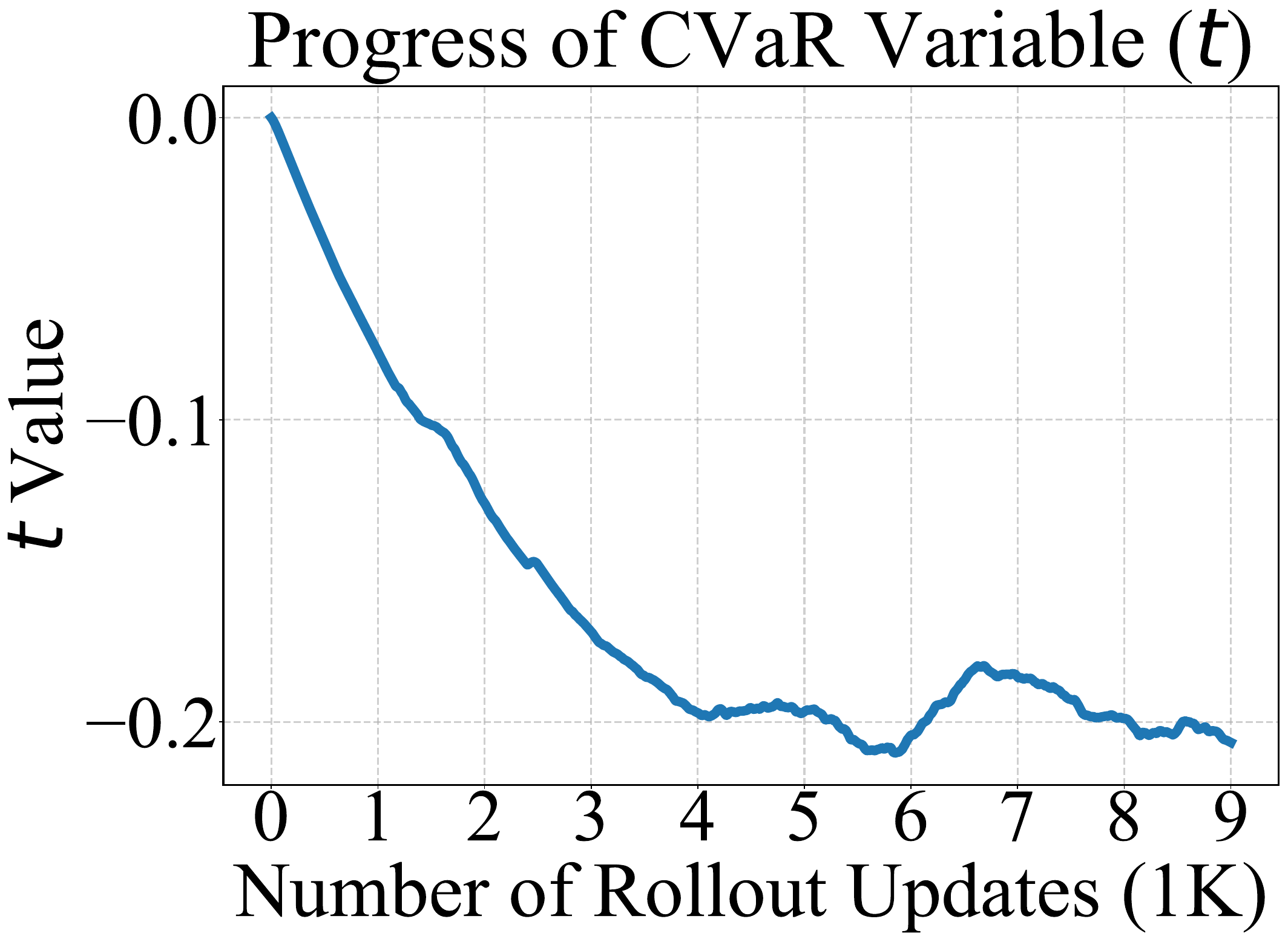}%
                \includegraphics[width=0.24\linewidth]{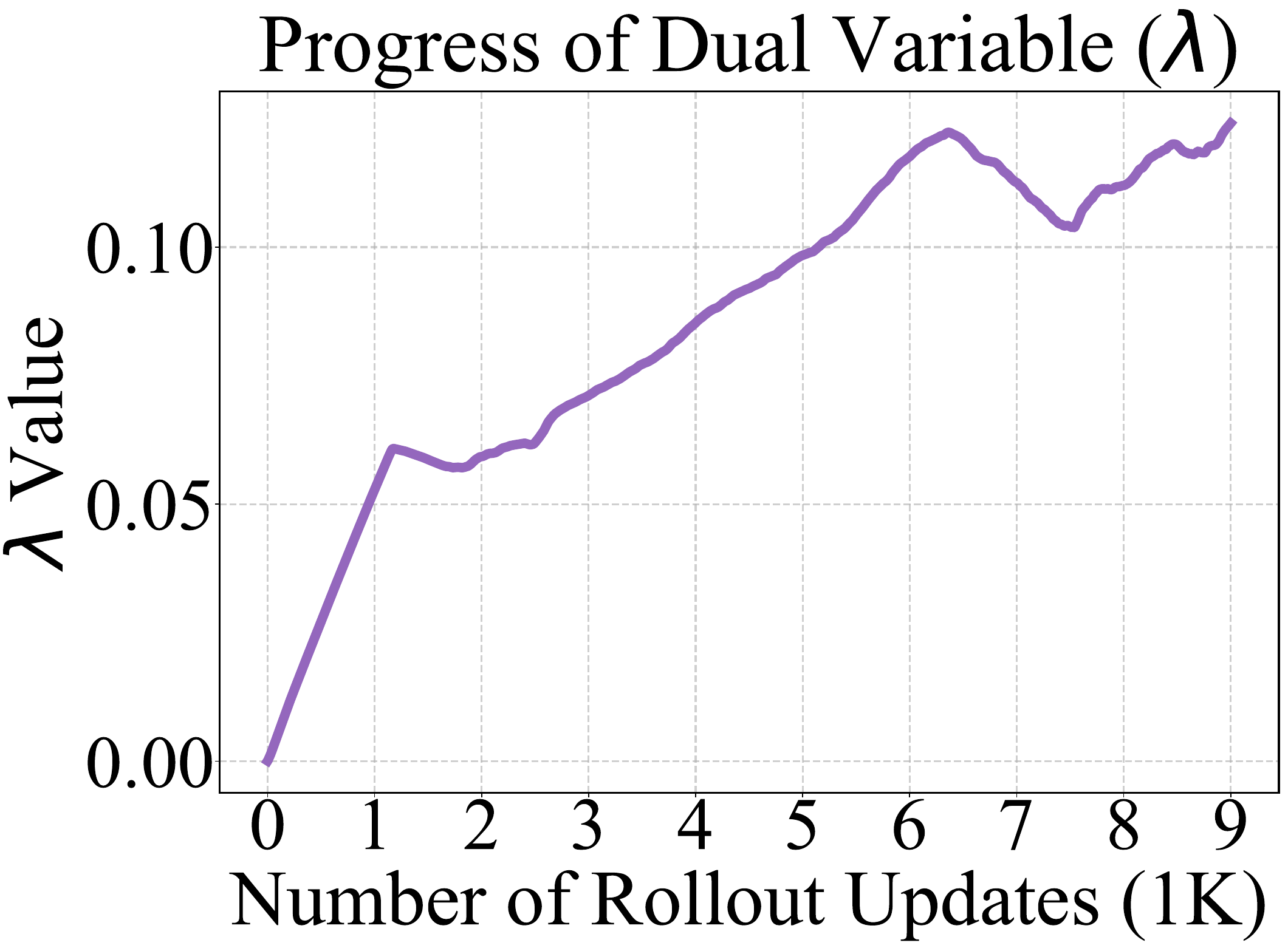}%
                \includegraphics[width=0.216\linewidth]{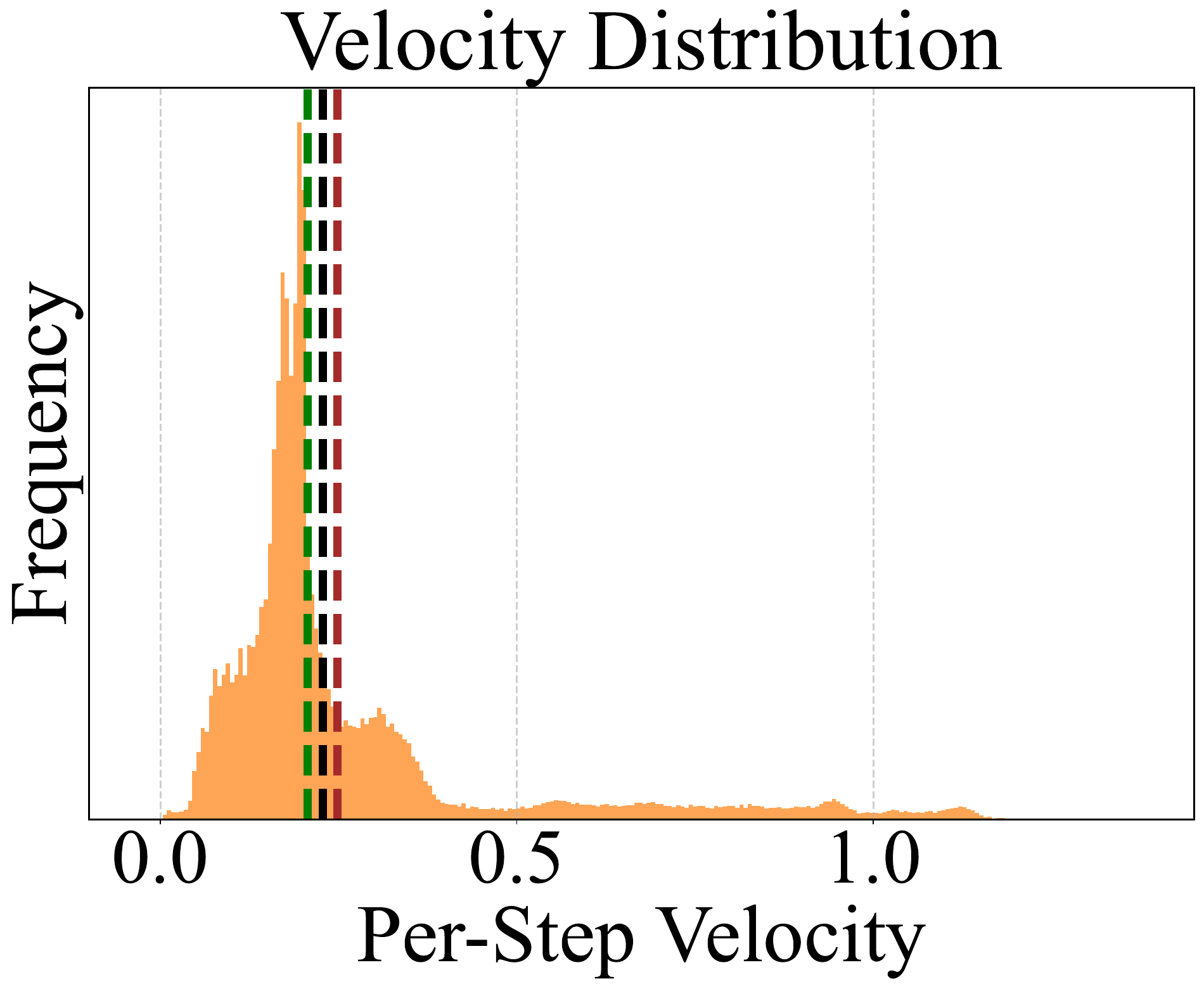}%
        }

      \subfigure[Environment: Walker2d \quad  $c = 1.171$ \quad $\beta$-upper quantile: $1.133$ \quad Converged $t$-value: $1.122$]{
                \includegraphics[width=0.238\linewidth]{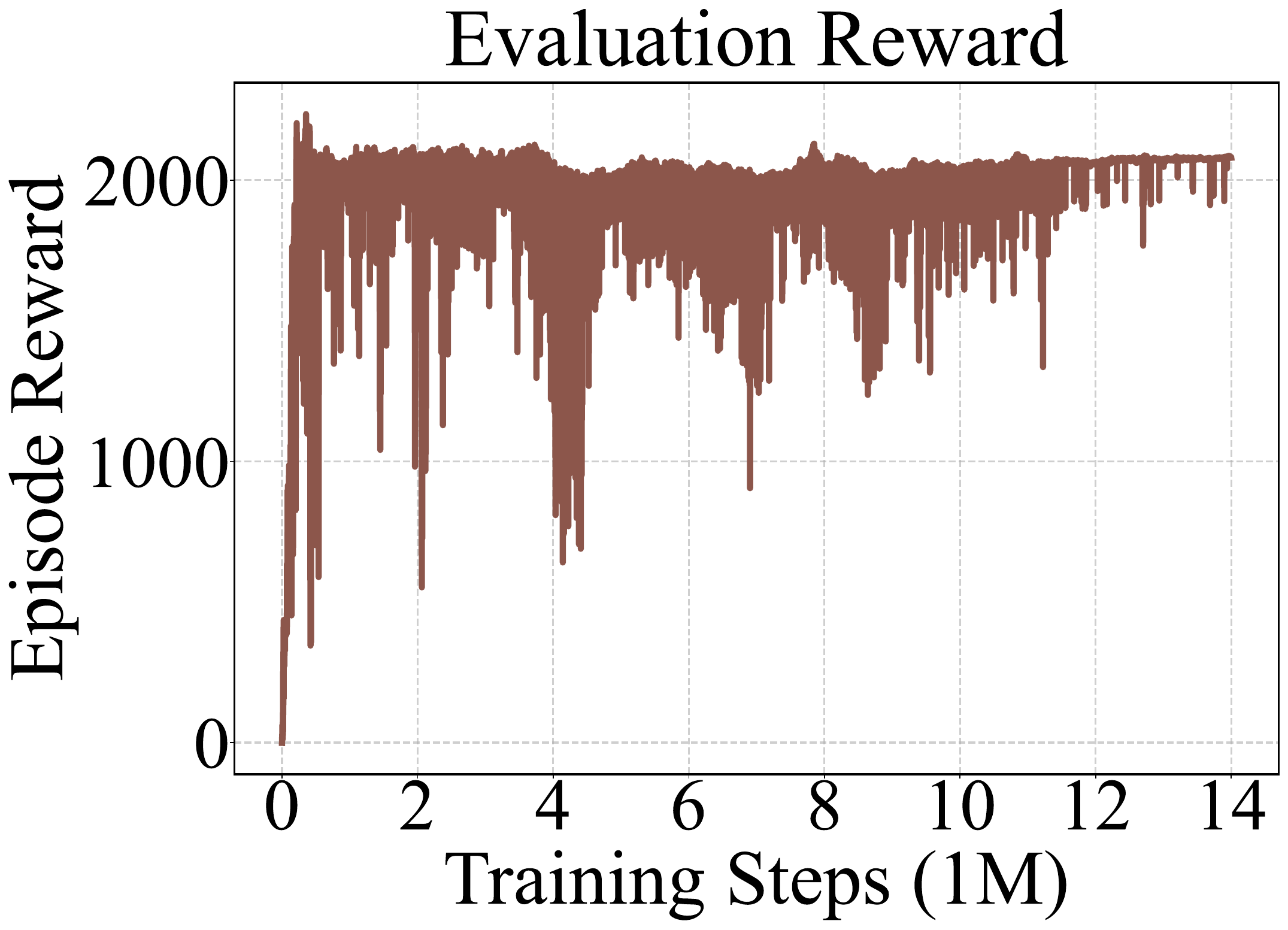}%
                \includegraphics[width=0.242\linewidth]{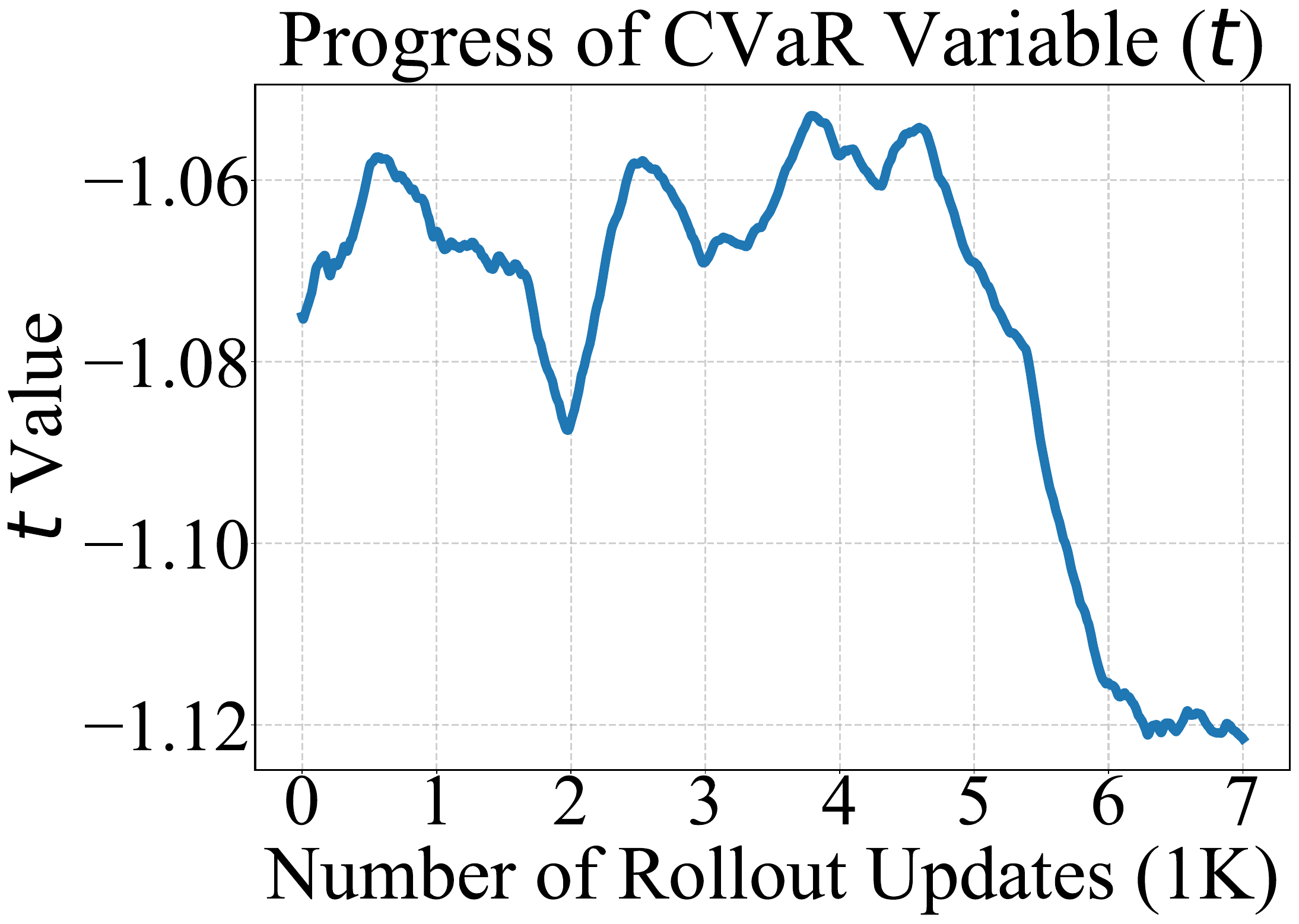}%
                \includegraphics[width=0.24\linewidth]{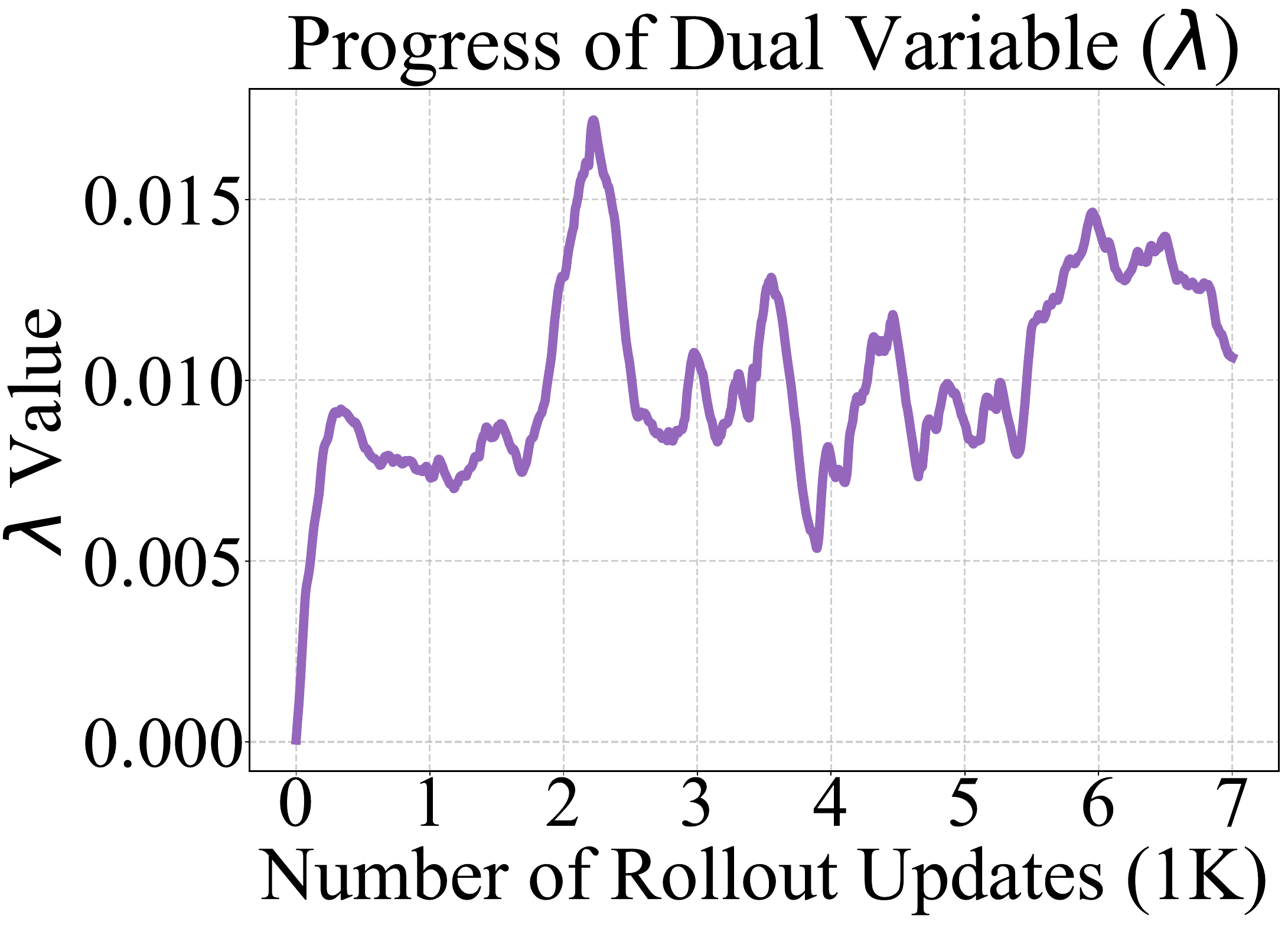}%
                \includegraphics[width=0.21\linewidth]{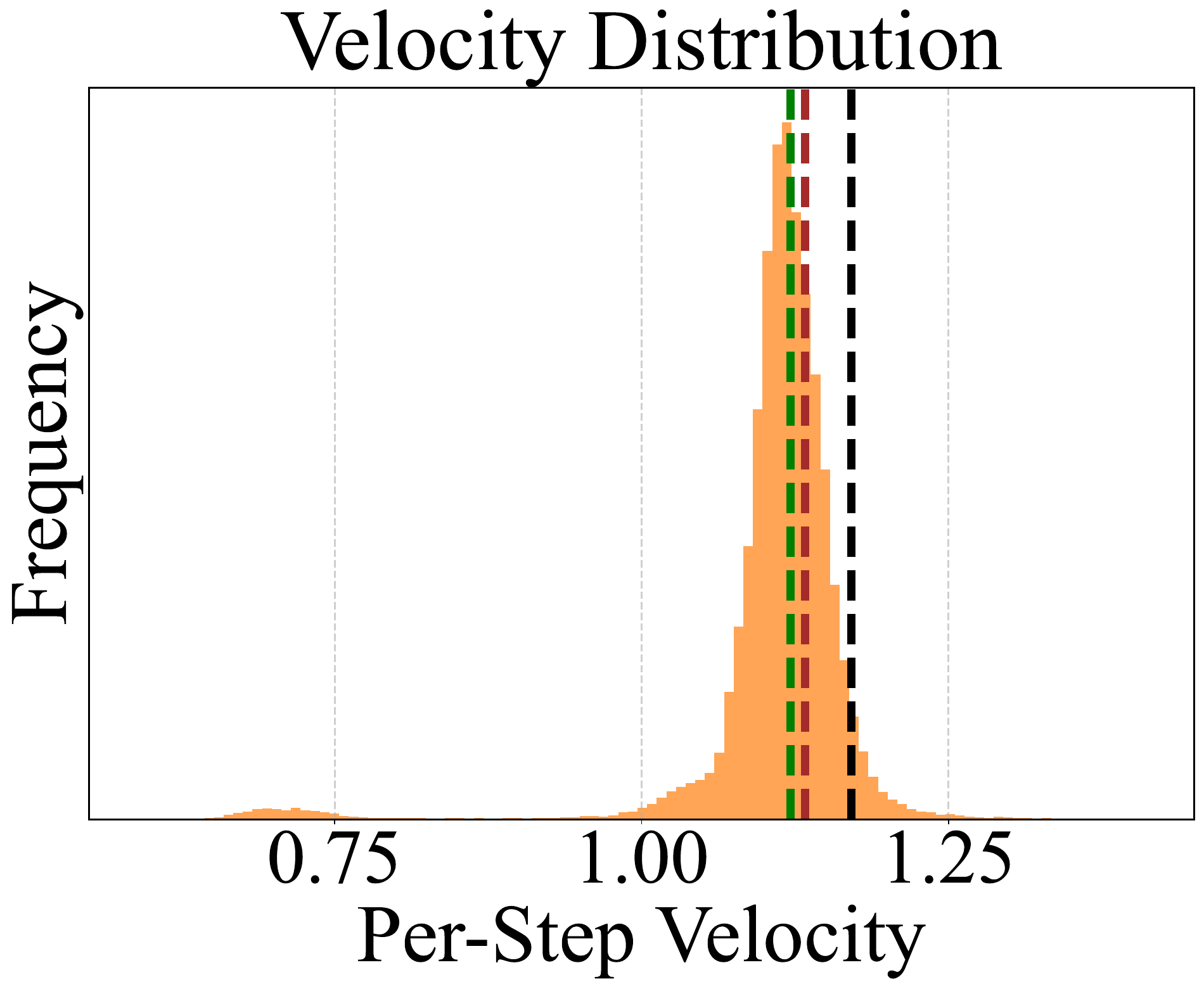}%
        }
    \caption{Learning curves of the agent trained to convergence. We report episodic evaluation rewards, the progression of optimization variables $t$ and $\lambda$, and the velocity distribution of the trained agent evaluated over 100 post-training episodes without further learning. Curves are \textit{not} smoothed.}
    \label{fig:vel_results}
\end{figure*}

\subsection{Results}
Results are shown in Table~\ref{tab:nav_results} for safe navigation and in Figure~\ref{fig:vel_results} for safe velocity.

\paragraph{Constraint Handling}
The performance in the navigation tasks shows us that our method can strictly prevent any constraint violations, whether quantified continuously or not. In fact, it is the only method achieving strictly zero violations compared to other PPO-based constraint learning algorithms (see Table 5 in \citep{safety_gym}). In some cases, constraint satisfaction also leads to higher rewards than the unconstrained vanilla agent (e.g., in the Push environment). This shows that constraints are not arbitrary but can guide the agent toward optimal task behavior.

\paragraph{Risk-Management}
On the other hand, the velocity tasks provide insights into the risk management capabilities. With properly tuned step sizes, the dual and CVaR variables stabilize and oscillate around consistent values. The post-training evaluations (the last column in Figure \ref{fig:vel_results}) extract the distribution of the trained agent's velocity over several evaluation episodes. We observe that the converged $t$-value matches the $\beta$-upper quantile of the velocity distribution. This is what the CVaR aims to do: it captures the expected cost in the worst $\beta$ fraction of outcomes.

Further, the evaluation reward becomes increasingly stable as training progresses despite the uncertainty in the evaluation environment. Considering the number of data points in the curves—1000 per 1M steps (e.g., about 14,000 points in Hopper)—the robustness of our method in managing risk is evident. This is also confirmed in post-training simulations: the agent moves forward cautiously compared to the unconstrained (risk-neutral) PPO agent, demonstrating that the learned policy adheres to the task while remaining interpretable from the perspective of the constraint.

We recognize the extended training duration and the need to tune the step sizes of $\lambda$ and $t$. Nonetheless, once these variables converge, the algorithm is effective not only in toy tasks but also in realistic control settings, making patience during training an important practical consideration.

\section{Conclusion}
\vspace{-3.75pt}

We introduce a training framework for solving constrained reinforcement learning (RL) problems with risk-awareness which has desirable properties. The problem exhibits a parameterized strong duality under constraint qualifications, which lends itself to an algorithm that is theoretically-supported and flexible to implement; the inner problem can be solved using any black-box RL algorithm, while the other variables are updated using SGDA. The framework handles general OCEs beyond the CVaR, and captures both risk-neutral (e.g., by setting $\beta = 1$ for the $\cvarb$) and risk-aware objectives and constraints, or a combination of the two, with empirical evidence that, in the context of risk-aware constrained RL, it effectively manages risk.

\vspace{-3.75pt}

\paragraph{Limitations} 
An open problem resulting from this work asks whether or not full strong duality holds unconditionally. 
Also, while our framework is inherently risk-aware, it is more computationally intensive compared to risk-neutral methods (as each update of $\lambda$ and $t$ variables requires solving for an approximately optimal policy), as is also true for other risk-aware methods, e.g., \cite{bonetti_risk_averse_rl_coherent_risk, risk_contr_rl_percentile_risk}. 

%==========================

\bibliographystyle{plainnat}
\bibliography{references}

%%%%%%%%%%%%%%%%%%%%%%%%%%%%%%%%%%%%%%%%%%%%%%%%%%%%%%%%%%%%
\clearpage  % NeurIPS wants to start appendices with new page

\appendix

\section{Supporting Details for Theoretical Results}

\subsection{Robustness in Value and Time}\label{appendix:robustness}
% We will illustrate how this risk-averse formulation results in an interpretation of enforcing robustness in both time and value for the problem at hand. 
The so-called \emph{reward-based} formulation in \eqref{eq:risk_aware_gen_rl}, as coined by \citet{ijcai2020p632,bonetti_risk_averse_rl_coherent_risk}, captures \textit{per-stage risk}, in contrast to the return-based approach \eqref{eq:primal_risk_averse}, which captures aggregate discounted risk. To see this, first note that the occupation measure can be reexpressed as
\begin{equation}\nonumber
    \mathrm{d}\nu^{\pi}	=\sum_{\tau=0}^{\infty}(1-\gamma)\gamma^{\tau}\mathrm{d}p_{\pi}^{\tau}
     =:\sum_{\tau=0}^{\infty}\mathrm{d}p_{\pi}^{T=\tau}p^{T}(\tau),
\end{equation}
%We offer another perspective to illustrate a robustness in time and space property 
% in \cref{appendix:robustness}.
%below.
where $T$ is a $\mathbb{N}\cup\{0\}$-valued geometric random variable such that $P(T=\tau)=:p^{T}(\tau):=(1-\gamma)\gamma^{\tau},\tau\in\mathbb{N}\cup\{0\}$, and $p_{\pi}^{T}$ is now the conditional state-action distribution relative to $T$. In other words, $\nu^{\pi}$ may be seen as the marginalization of the disintegration $\mathrm{d}p_{\pi}^{T=\tau}(s,a)p^{T}(\tau)$ relative to $T$, the latter interpreted as a random time. One can show that
\begin{equation}\nonumber
    -\mathrm{CVaR}_{\nu^{\pi}}^{\beta}(-r(s,a))=-\mathrm{CVaR}^{\beta}(-r(s_{T}^{\pi},a_{T}^{\pi})),
\end{equation}
where $(s_{T}^{\pi},a_{T}^{\pi})$ denotes the state-action vector under policy $\pi$ and evaluated at random time $T$. Specifically for the case of $\mathrm{CVaR}$ (and similarly for other OCEs), it can be readily seen that for $\beta=1$ we recover the objective of \eqref{eq:classic_rl} which can be equivalently written as $V(\pi)=(1-\gamma)^{-1}\mathbb{E}\{r(s_{T}^{\pi},a_{T}^{\pi})\}$, whereas if $\beta\downarrow0$, we obtain
\begin{equation}\nonumber
    \lim_{\beta\downarrow0}-\mathrm{CVaR}_{\nu^{\pi}}^{\beta}(-r(s,a))=\inf_{\tau}\mathrm{ess\,inf}_{(s,a)\sim p_{\pi}^{\tau}}\:r(s,a),
\end{equation}
which shows in particular that such reward-based RL formulations enforce \textit{joint aversion in reward value and the time} for which the reward takes each value. For instance, for sufficiently small $\beta$, the functional $-\mathrm{CVaR}^{\beta}(-r(s_{T}^{\pi},a_{T}^{\pi}))$ will be more sensitive to very small rewards happening far in the future (and with low probability $(1-\gamma)\gamma^\tau$), as compared with \eqref{eq:classic_rl} and \eqref{eq:primal_risk_averse}, which both consider reward time averaging, heavily discounting future reward contributions.

\subsection{Proof of Lemma \ref{lem:move_sup_obj}}\label{appendix:proof_move_sup_obj}
We will work with the convex analytic dual form for this proof, but replacing the occupancy measure with the expected discount sum will give the same result.
\paragraph{Statement}
The problem 
\begin{equation}
    \begin{aligned}
        P^* = & \sup_{\pi \in \mathcal{R}}\sup_{t_0 \in \R} \E_{\nu^\pi(s,a)} \left[  t_0 - \frac{1}{\beta} (t_0 - r_0(s,a))_+ \right]& \\
        &\textnormal{s.t.}  \sup_{t_i \in \R} \E_{\nu^\pi(s,a)} \left[  t_i - \frac{1}{\beta} (t_i - r_i(s,a))_+ \right] \geq c_i &\qquad \forall i = 1, 2, \dots, m
        \label{eq:constr_rl_cvarb_occ}
    \end{aligned}
\end{equation}
and that of
\begin{equation}
    \begin{aligned}
         \sup_{\pi \in \mathcal{R}, t_0 \in \R, t_i \in \R} &\E_{\nu^\pi(s,a)} \left[  t_0 - \frac{1}{\beta} (t_0 - r_0(s,a))_+ \right]& \\
        \textnormal{s.t.}\qquad &\E_{\nu^\pi(s,a)} \left[  t_i - \frac{1}{\beta} (t_i - r_i(s,a))_+ \right] \geq c_i &\qquad \forall i = 1, 2, \dots, m
        \label{eq:constr_rl_cvarb_occ_ver2}
    \end{aligned}
\end{equation}
are equivalent.

\begin{proof}
We need to prove that it is valid to move the supremum over each $t_i$ to the objective. The proof follows that of \citet{risk_contr_rl_percentile_risk}. 

\begin{itemize}
    \item (Case \eqref{eq:constr_rl_cvarb} $\leq$ \eqref{eq:constr_rl_cvarb_ver2}.) Let $\pi^1$ be a feasible policy for the problem \eqref{eq:constr_rl_cvarb}. Then define \[ t^1_i \coloneqq \arg\max_{t_i \in \R} \E_{\nu^{\pi^1}(s,a)}\left[ t_i - \frac{1}{\beta}(t_i - r_i(s,a))_+ \right] \geq c_i, \qquad \forall i=1, 2, \dots, m, \]
    so that $(\pi^1, t_i^1)$ is feasible for \eqref{eq:constr_rl_cvarb}. Then this tuple $(\pi^1, t_i^1$) is clearly also feasible for \eqref{eq:constr_rl_cvarb_ver2}.
    
    \item (Case \eqref{eq:constr_rl_cvarb_ver2} $\leq$ \eqref{eq:constr_rl_cvarb}.) On the other hand, for any $\pi$, we clearly have that
    \[ -\cvarb_{\nu^\pi}[-Z] = \sup_t \E_{\nu^\pi}\left[t - \frac{1}{\beta}(t - Z)_+ \right] \geq \E_{\nu^\pi}\left[\hat{t} - \frac{1}{\beta}(\hat{t} - Z)_+ \right], \]
    for any $\hat{t}$. So, if the quantity on the RHS of the inequality is lower bounded by $c_i$ (that is, feasible for \eqref{eq:constr_rl_cvarb_ver2}), then the LHS is clearly also lower bounded by $c_i$ (thus, feasible for \eqref{eq:constr_rl_cvarb}). 
\end{itemize}
\end{proof}

\subsection{Proof of Proposition \ref{prop:dual}} \label{appendix:proof_dual_prop}
\paragraph{Statement}  Let $r_i$ be bounded functions for all $i \in \{0\} \cup [m]$. Assume that Slater's condition holds for \eqref{eq:constr_rl_cvarb_fixed_t}. Then, \eqref{eq:constr_rl_cvarb_fixed_t} exhibits strong duality, and thus 
    \[ \sup_\pi \inf_{\lambda \geq 0} \L(\pi, t, \lambda) = \inf_{\lambda \geq 0} \sup_\pi \L(\pi, t, \lambda). \]
   \begin{proof}
    \citet[Theorem 1]{constrained_rl_zero_duality_gap} proved strong duality holds for problems of the form
\begin{equation}
    \begin{aligned}
         \sup_{\pi \in \mathcal{P}(\mathcal{S})} &\  \E \left[ \sum_{\tau=0}^\infty \gamma^\tau  r_0(s_{\tau}^{\pi},a_{\tau}^{\pi}) \right]
         %, & a_\tau \sim \pi(\cdot \mid s_\tau)  
        ~\textnormal{s.t. }~   \E \left[ \sum_{\tau=0}^\infty \gamma^\tau r_i(s_{\tau}^{\pi},a_{\tau}^{\pi}) \right] \geq c_i, \,\,  \forall i \in [m]\,,
        \label{eq:constr_rl_zero_duality_gap}
    \end{aligned}
\end{equation}
which involve the expectation of the discounted sums of some bounded reward functions $r_i$, for $i \in  \{0\}\cup [m]$. Thus, the problem \eqref{eq:constr_rl_cvarb_fixed_t} is exactly in the form of \citet{constrained_rl_zero_duality_gap}, with the reward functions $r'_i(s,a,t_i) = t_i - \frac{1}{\beta}(t_i - r_i(s,a))_+$, for $i \in \{0\}\cup[m]$. Since the original $r_i$ are bounded and $t_i$ are fixed, we have that $r'_i$ are bounded and we can readily apply \citet[Theorem 1]{constrained_rl_zero_duality_gap}.
\end{proof}
\subsection{Proof of Theorem \ref{thm: almost-zero duality gap}} \label{appendix: almost zero duality gap}
Under Assumption \ref{assum:constraint qual}, we have shown that the proposed partial Lagrangian relaxation is exact, i.e., that
\begin{align*}
\sup_{t \in \mathcal{I}}\sup_\pi \inf_{\lambda \geq 0} \L(\pi, t, \lambda) = \sup_{t \in \mathcal{I}} \inf_{\lambda \geq 0} \sup_\pi \L(\pi, t, \lambda). 
\end{align*}
 We assume that $\pi_{\theta}$, for $\theta \in \Theta \subset \mathbb{R}^p$ (with $\Theta$ some compact set), is an $\epsilon$-universal parametrization of measures in $\mathcal{P}(\mathcal{S})$, according to \citet[Definition 1]{constrained_rl_zero_duality_gap} and that the parametrized-policy version of \eqref{eq:constr_rl_cvarb_ver2} is feasible. We proceed to show that there is almost no price to pay for the policy parametrization, in terms of duality gap.

Indeed, assuming that the parametrized optimization problem is feasible, we can utilize \citet[Theorem 2]{constrained_rl_zero_duality_gap} and Assumption \ref{assum:constraint qual} to deduce that, for all $t \in \mathcal{I}$, the following inequalities hold:
\begin{align*} 
P^*({t}) \coloneqq   \sup_{\pi}\inf_{\lambda \geq 0} \mathcal{L}(\pi,t,\lambda) 
 \geq  \inf_{\lambda \geq 0}\sup_{\theta} \mathcal{L}(\pi_{\theta},t,\lambda) \geq P^*(t) - G(t) \epsilon/(1-\gamma),
\end{align*}
\noindent {where $G(t) = \mathcal{O}(1)$ is a constant depending on $t$, and $\epsilon > 0$ is an arbitrarily small constant (depending on the employed parametrization). Let $t^* \in {\arg\max}_{t \in \mathcal{I}}\ P^*(t)$ and $t_{\theta}^* \in {\arg\max}_{t \in \mathcal{I}}\ \inf_{\lambda \geq 0}\sup_{\theta} \mathcal{L}(\pi_{\theta},t,\lambda)$. Let also $G = \max\{G(t^*),G(t_{\theta}^*)\}.$ Then, we obtain that}
\begin{align*} P^*(t^*) \geq &\ \inf_{\lambda \geq 0}\sup_{\pi} \mathcal{L}(\pi,t^*,\lambda) \geq  P^*(t^*) - G \epsilon/(1-\gamma),\\
P^*(t^*_{\theta}) \geq &\ \inf_{\lambda \geq 0}\sup_{\theta} \mathcal{L}(\pi_{\theta},t_{\theta}^*,\lambda) \geq  P^*(t_{\theta}^*) - G \epsilon/(1-\gamma).
\end{align*}
\noindent Since, by definition, $P^*(t^*) \geq P^*(t_{\theta}^*)$, while $\sup_{\theta}\inf_{\lambda \geq 0} \mathcal{L}(\pi_{\theta},t^*,\lambda) \leq \sup_{\theta}\inf_{\lambda \geq 0} \mathcal{L}(\pi_{\theta},t_{\theta}^*,\lambda)$, we obtain that
\begin{align*}
    P^*(t^*) \equiv  \sup_{t \in \mathcal{I}}  \sup_{\pi}\inf_{\lambda \geq 0} \mathcal{L}(\pi,t,\lambda)  \geq  \sup_{t \in \mathcal{I}} \inf_{\lambda \geq 0}\sup_{\theta} \mathcal{L}(\pi_{\theta},t,\lambda) \geq P^*(t^*) -\mathcal{O}\left(\frac{\epsilon}{1-\gamma} \right),
\end{align*}
\noindent which completes the proof.\hfill \qed

\subsection{Discussion on Assumption \ref{assum:local_opt}}
\label{appendix:disc_local_opt}

Below, we provide an insight as to why the inexact oracle utilized in Assumption \ref{assum:local_opt} is reasonable under minimal conditions (under which the sample gradients used in Algorithm \ref{alg:main} satisfy Assumption \ref{assum:local_opt}). Specifically, building upon the discussion given later in Appendix \ref{appendix:conditions_for_Lipschitz_smoothness}, we note that the most general condition guaranteeing Assumption \ref{assump:3} is given in the second scenario of Appendix \ref{appendix:conditions_for_Lipschitz_smoothness}, which requires that the function $\mathcal{L}(\pi_{\theta^*}(t,\lambda),t,\lambda)$ satisfies the Lojasiewicz inequality with uniform constants over $(t,\lambda)$, for any selection $\pi_{\theta^*}(t,\lambda) \in \arg\max_{\theta \in \Theta} \mathcal{L}(\pi_{\theta},t,\lambda)$ (again, see Appedinx \ref{appendix:conditions_for_Lipschitz_smoothness} for the definition of the Lojasiewicz inequality in this context). Since the (parametrized) policy optimization problem $\max_{\theta \in \Theta} \mathcal{L}(\pi_{\theta},t,\lambda)$ is nonconvex (and is approximately solved using a solver like PPO), we can only expect to obtain an approximate solution $\theta^{\dagger}(t,\lambda)$ such that
\[ \mathcal{L}(\pi_{\theta^*}(t,\lambda),t,\lambda) - \mathcal{L}(\pi_{\theta^{\dagger}(t,\lambda)},t,\lambda) \leq \tilde{\varepsilon}(\theta^{\dagger},\theta^*,t,\lambda) \leq \varepsilon,\]
\noindent for some $\varepsilon > 0$ (uniformly bounded in $(t,\lambda)$). Then, if $\mathcal{L}(\pi_{\theta^*}(t,\lambda),t,\lambda)$ satisfies the Lojasiewicz inequality for some uniform positive constants $C, \eta$, we obtain that
\[ \|\theta^{\dagger} - \theta^*\| \leq C \varepsilon^{\eta}.\]
\noindent Thus, under the mere assumption that the sample gradient function $\hat{\nabla}_{t,\lambda} \hat{\mathcal{L}}(\pi_{\theta},t,\lambda)$ is $\hat{L}$-Lipschitz continuous with respect to $\theta$ (which can be enforced, if necessary, by appropriate smoothing of the utility function $g$; a general smoothing strategy for generating smooth OCEs from non-smooth utilities can be found in \citet[Example 2]{epi_regularization_risk_meas_20}), we obtain that
\[ \left\| \hat{\nabla}_{t,\lambda} \hat{\mathcal{L}}(\pi_{\theta^{\dagger}(t,\lambda)},t,\lambda) - \hat{\nabla}_{t,\lambda} \hat{\mathcal{L}}(\pi_{\theta^{*}(t,\lambda)},t,\lambda)\right\| \leq \hat{L}\|\theta^{\dagger}-\theta^*\| \leq \hat{L} C \varepsilon^{\eta} \coloneqq \delta.\]
\noindent Upon noting that $\hat{\nabla}_{t,\lambda} \hat{\mathcal{L}}(\pi_{\theta^{*}(t,\lambda)},t,\lambda)$ is an unbiased sample of $-\nabla_{t,\lambda} f(t,\lambda)$ (cf. Appendix \ref{appendix:conditions_for_Lipschitz_smoothness}), we can easily deduce that under the said minor conditions, the sample gradients utilized in Algorithm \ref{alg:main} satisfy Assumption \ref{assum:local_opt}. Thus, we can see that our oracle condition in Assumption \ref{assum:local_opt} is justified and can be shown to hold under minimal conditions (without requiring convexity or uniqueness of the solution of $\max_{\theta \in \Theta} \mathcal{L}(\pi_{\theta},t,\lambda)$).

%%%%%%%%%%%%%%%%%%%%%%%%%%%%%%%%%
\subsection{Lipschitzness of the Lagrangian}\label{appendix:proof_lemma_sup_ineq}
First we show a technical lemma, and then we proceed with the proof of Lemma \ref{lemma:lipschitz}.
\begin{lemma}\label{lemma:sup_inequality}
For any pair $(\tilde{t},\bar{t})\in  \mbb{R}^{2m+2}$, and for any $\lambda\in \mbb{R}^m$, it is true that
\begin{align}\label{eq:supLinequality}
    |\sup_\pi \L(\pi, \tilde{t}, \lambda) - \sup_\pi \L(\pi, \bar{t}, \lambda)  | \leq  \sup_{\pi} |\L(\pi, \tilde{t}, \lambda) -  \L(\pi, \bar{t}, \lambda)  | 
\end{align}    
\end{lemma}
\paragraph{Proof of Lemma \ref{lemma:sup_inequality}} For any pair $\tilde{t},\bar{t}\in\mbb{R}$ it is true that
\begin{align}
    \sup_\pi \L(\pi, \tilde{t}, \lambda) &= \sup_\pi \L(\pi, \tilde{t}, \lambda) - L(\pi, \bar{t}, \lambda) + L(\pi, \bar{t}, \lambda) \\
    &\leq \sup_\pi \L(\pi, \tilde{t}, \lambda) - \L(\pi, \bar{t}, \lambda) + \sup_\pi \L(\pi, \bar{t}, \lambda) \implies \\
    \sup_\pi \L(\pi, \tilde{t}, \lambda) - \sup_\pi \L(\pi, \bar{t}, \lambda) &\leq  \sup_\pi \L(\pi, \tilde{t}, \lambda) - \L(\pi, \bar{t}, \lambda)\leq  \sup_\pi | \L(\pi, \tilde{t}, \lambda) - \L(\pi, \bar{t}, \lambda)|
\end{align} Similarly, we can show that
\begin{align}
    \sup_\pi \L(\pi, \bar{t}, \lambda) - \sup_\pi \L(\pi, \tilde{t}, \lambda) &\leq  \sup_\pi \L(\pi, \bar{t}, \lambda) - \L(\pi, \tilde{t}, \lambda)\leq  \sup_\pi | \L(\pi, \bar{t}, \lambda) - \L(\pi, \tilde{t}, \lambda)|.
\end{align} The last two displays give the inequality \eqref{eq:supLinequality}.

\paragraph{Proof of Lemma \ref{lemma:lipschitz}}
For any pair $(\tilde{t},\bar{t})\in  \mbb{R}^{2m+2}$, we find an upper bound on the term $|\L(\pi, \tilde{t}, \lambda) -\L(\pi, \bar{t}, \lambda) |$ as follows
\begin{align*}
   &|\L(\pi, \tilde{t}, \lambda) -\L(\pi, \bar{t}, \lambda) |\\
    &\leq  \E \left[ \sum_{\tau=0}^\infty \gamma^\tau \Bigg|\left( \tilde{t}_0 - \frac{1}{\beta} (\tilde{t}_0 - r_0(s_\tau,a_\tau))_+ \right)-\left( \bar{t}_0 - \frac{1}{\beta} (\bar{t}_0 - r_0(s_\tau,a_\tau))_+ \right)\Bigg|\right] \\
    &\qquad + \E\left[
        \begin{aligned}
            &\sum_{\tau=0}^\infty \gamma^\tau \Bigg| \sum^m_{i=1}\lambda_i \left(\tilde{t}_i - \frac{1}{\beta_i} (\tilde{t}_i - r_i(s_\tau,a_\tau))_+\right) \\
            &\qquad\qquad\qquad\qquad - \sum^m_{i=1}\lambda_i \left(\bar{t}_i - \frac{1}{\beta_i} (\bar{t}_i - r_i(s_\tau,a_\tau))_+\right)\Bigg|
        \end{aligned}
    \right] \\
    &= \E \left[ \sum_{\tau=0}^\infty \gamma^\tau \Bigg|\left( \tilde{t}_0 - \frac{1}{\beta} \max\left\{\tilde{t}_0 - r_0(s_\tau,a_\tau) ,0 \right\} \right)-\left( \bar{t}_0 - \frac{1}{\beta}  \max\left\{\bar{t}_0 - r_0(s_\tau,a_\tau),0 \right\} \right)\Bigg|\right] \\
    &\qquad + \E\left[
        \begin{aligned}
            &\sum_{\tau=0}^\infty \gamma^\tau \Bigg| \sum^m_{i=1}\lambda_i \left(\tilde{t}_i - \frac{1}{\beta_i} \max\left\{\tilde{t}_i - r_i(s_\tau,a_\tau)\right\}\right) \\
            &\qquad\qquad\qquad\qquad - \sum^m_{i=1}\lambda_i \left(\bar{t}_i - \frac{1}{\beta_i} \max\left\{\bar{t}_i - r_i(s_\tau,a_\tau),0\right\}\right)\Bigg|
        \end{aligned}
    \right] \\
    &\leq \E \left[ \sum_{\tau=0}^\infty \gamma^\tau (| \tilde{t}_0 -\bar{t}_0 | + \frac{1}{\beta} |\max\left\{\tilde{t}_0 - r_0(s_\tau,a_\tau) ,0 \right\}-\max\left\{\bar{t}_0 - r_0(s_\tau,a_\tau),0 \right\}|)\right] \\
    &\qquad + \E \left[ \sum_{\tau=0}^\infty \gamma^\tau \sum^m_{i=1}\lambda_i (| \tilde{t}_i -\bar{t}_i | + \frac{1}{\beta_i} |\max\left\{\tilde{t}_i - r_i(s_\tau,a_\tau) ,0 \right\}-\max\left\{\bar{t}_i - r_i(s_\tau,a_\tau),0 \right\}|)\right] \\
    &\leq \E \left[ \sum_{\tau=0}^\infty \gamma^\tau (| \tilde{t}_0 -\bar{t}_0 | + \frac{1}{\beta} \max\left\{|\tilde{t}_0 - r_0(s_\tau,a_\tau) - \bar{t}_0 - r_0(s_\tau,a_\tau)|,0 \right\})\right] \\
    &\qquad + \E \left[ \sum_{\tau=0}^\infty \gamma^\tau \sum^m_{i=1}\lambda_i (| \tilde{t}_i -\bar{t}_i | + \frac{1}{\beta_i} \max\left\{|\tilde{t}_i - r_i(s_\tau,a_\tau) - \bar{t}_i - r_i(s_\tau,a_\tau)|,0 \right\})\right]\\
    &= \E \left[ \sum_{\tau=0}^\infty \gamma^\tau (| \tilde{t}_0 -\bar{t}_0 | + \frac{1}{\beta} |\tilde{t}_0  - \bar{t}_0 |)\right] + \E \left[ \sum_{\tau=0}^\infty \gamma^\tau \sum^m_{i=1}\lambda_i (| \tilde{t}_i -\bar{t}_i | + \frac{1}{\beta_i} |\tilde{t}_i -\bar{t}_i| )\right]\\
    &\leq (1-\gamma)^{-1}(1+\frac{1}{\beta})| \tilde{t}_0 -\bar{t}_0 | + (1-\gamma)^{-1}\sum^m_{i=1}\lambda_i (1+\frac{1}{\beta_i})| \tilde{t}_i -\bar{t}_i |\\
    & \leq (1-\gamma)^{-1} \Vert \tilde{t} -\bar{t} \Vert \sqrt{(1+\frac{1}{\beta})^2+\sum^m_{i=1}\lambda^2_i (1+\frac{1}{\beta_i})^2}.
\end{align*}
Then, we apply Lemma \ref{lemma:sup_inequality} to get \begin{align}
    |\sup_\pi \L(\pi, \tilde{t}, \lambda) - \sup_\pi \L(\pi, \bar{t}, \lambda)  | &\leq  \sup_{\pi} |\L(\pi, \tilde{t}, \lambda) -  \L(\pi, \bar{t}, \lambda)  | \\&\leq (1-\gamma)^{-1} \Vert \tilde{t} -\bar{t} \Vert \sqrt{(1+\frac{1}{\beta})^2+\sum^m_{i=1}\lambda^2_i (1+\frac{1}{\beta_i})^2}.
\end{align} Thus the function $\mc{L} (\pi,\cdot, \lambda)$ is Lipschitz, with Lipschitz constant \begin{align}
    c(\gamma ,\beta,\lambda)\coloneqq (1-\gamma)^{-1} \sqrt{(1+\frac{1}{\beta})^2+\sum^m_{i=1}\lambda^2_i (1+\frac{1}{\beta_i})^2}.
\end{align}
\noindent The result then follows since we are interested in multipliers $\lambda \in \Lambda$, with $\Lambda$ a convex and compact set (and thus, there exists a constant $C \geq c(\gamma,\beta,\lambda)$, independent of $\lambda$ and $t$, for which $\mathcal{L}(\pi,\cdot,\lambda)$ is Lipschitz continuous for all $(t,\lambda) \in \mathcal{T} \times \Lambda$).\hfill \qed

\subsection{Conditions That Guarantee Assumption \ref{assump:3}} \label{appendix:conditions_for_Lipschitz_smoothness}
    Let us note that Assumption \ref{assump:3} is not particularly restrictive in our setting. In what follows, we discuss some general conditions under which this holds readily, as well as a simple methodology of enforcing this assumption, if necessary, by a slight algorithmic adjustment. Before we do this, let us first note that $\mathcal{L}(\pi_{\theta},\cdot,\lambda)$ is not necessarily differentiable (unless the probability of $t_i - r_i(s,a) = 0$ is 0), because of the plus function $(\cdot)_+$. Nonetheless, adding a small random noise to the reward following a random variable with a smooth mollifier density ensures that the function is infinitely differentiable with respect to both $t$ and $\lambda$. Note that, since we assume that the reward is bounded, we cannot  (in theory) add Gaussian noise, but any smooth and compactly supported mollifier-based noise is admissible. Thus, we assume (without loss of generality) that $\mathcal{L}(\pi_{\theta},\cdot,\lambda)$ is differentiable (in fact, to an arbitrary degree).
    \par Next, we focus on the task of assessing the differentiability properties of $f(t,\lambda)$. Specifically, we identify three general scenarios that guarantee Assumption \ref{assump:3} in our case. 
    \begin{enumerate}
    \item The first scenario, which is somewhat restrictive, relies on the assumption that, for each $(t,\lambda) \in \mathcal{T} \times \Lambda$, the maximization problem, with respect to $\theta$, admits a unique solution. In that case, we can readily utilize Lemma 2.2 of \citet{Shapiro_Extremal_functions}, which states that $\mathcal{L}(\pi_{\theta^*(t,\lambda)},t,\lambda)$ is twice-continuously differentiable. Since $\mathcal{T} \times \Lambda$ is a compact set, this immediately implies that the function is $\ell$-smooth over $\mathcal{T} \times \Lambda$, for some constant $\ell > 0$.
    \item The second scenario is significantly more general, and relies on the strong second-order sufficient optimality conditions for the maximization over $\theta$ and the assumption that the function $\mathcal{L}(\pi_{\theta^*(t,\lambda)},t,\lambda)$ satisfies the Lojasiewicz inequality with some constant uniform over $(t,\lambda)$, for any $\pi_{\theta^*(t,\lambda)} \in \arg\max_{\theta \in \Theta}\mathcal{L}(\pi_{\theta},t,\lambda)$, that is, for each $(t,\lambda) \in \mathcal{T} \times \Lambda$, there exist constants $C > 0$ and $\eta > 0$ such that
    \begin{equation*}
    \begin{split}
    &\text{dist}\left(\theta,\arg\max_{\theta \in \Theta}\mathcal{L}(\pi_{\theta},t,\lambda)\right) \leq C \left(\max_{\theta \in \Theta} \mathcal{L}(\pi_{\theta},t,\lambda) - \mathcal{L}(\pi_{\theta},t,\lambda)\right)^{\eta}.
    \end{split}
    \end{equation*}
    \noindent We note that the previous holds in several cases (typically with uniform exponent $\eta$). For example, this is true when $\max_{\theta \in \Theta} \mathcal{L}(\pi_{\theta},\cdot,\cdot)$ is sub-analytic (which, for example, is implied if $\mathcal{L}(\pi_{\theta},t,\lambda)$ is analytic, noting that this can easily be enforced our setting). Since $\mathcal{T} \times \Lambda$ is a compact and convex set, we can readily utilize Theorem 11 from \cite{Hashmi_etal} to deduce that under this general assumption, the function $\mathcal{L}(\pi_{\theta^*(t,\lambda)},t,\lambda)$ is $\ell$-smooth over $\mathcal{T} \times \Lambda$, irrespectively of the selections $\theta^*(t,\lambda) \in \arg\max_{\theta \in \Theta} \mathcal{L}(\pi_{\theta},t,\lambda)$.
    \item Under certain qualification conditions, laid out in Section 4 of \citet{Shapiro_Extremal_functions} (that do not require the second-order strong sufficient conditions for the maximization problem over $\theta$), $\ell$-smoothness of $\max_{\theta \in \Theta} \mathcal{L}(\pi_{\theta},t,\lambda)$ is also guaranteed. We refer the reader to \citet{Shapiro_Extremal_functions} for additional details, since this analysis directly applicable in our setting.
    \end{enumerate}
     If any of the above conditions holds, then Assumption \ref{assump:3} is true, and we can readily compute the gradient of $f$ as
    \[ \nabla_{t,\lambda} f(t,\lambda) = -\nabla_{t,\lambda} \mathcal{L}(\pi_{\theta},t,\lambda)\vert_{\pi_{\theta} = \pi_{\theta^*(t,\lambda)}},\]
    \noindent where $\theta^*(t,\lambda) \in \arg\max_{\theta \in \Theta} \mathcal{L}(\pi_{\theta},t,\lambda)$ is an arbitrary selection. 
    \par If none of the above is true, we can still guarantee that Assumption \ref{assump:3} holds by slightly altering Algorithm \ref{alg:main}. Specifically, for some small $\mu > 0$, we can define the surrogate function 
    \begin{equation} \label{eqn:smoothed Lagrangian}
    \begin{split}
    &\mathcal{L}_{\mu}(\pi_{\theta^*(t,\lambda)},t,\lambda) \coloneqq \mathbb{E}_{U_1,U_2}\left\{\mathcal{L}(\pi_{\theta^*(t + \mu U_1,\lambda + \mu U_2)},t+\mu U_1 , \lambda + \mu U_2) \right\},
    \end{split}
    \end{equation}
    \noindent where $U_1,U_2$ follow a uniform distribution over the unit ball (of appropriate dimension in each case). Then, this new surrogate function can be made arbitrarily close to $\max_{\theta \in \Theta} \mathcal{L}(\pi_{\theta},t,\lambda)$ (where the proximity is uniform in the constant $\mu$), while at the same time being $\ell_{\mu}$-smooth, with $\ell_{\mu} = \mathcal{O}(1/\mu)$. 
    \par Then, we can substitute the original Lagrangian with the smoothed Lagrangian given in \eqref{eqn:smoothed Lagrangian} and run Algorithm \ref{alg:main} on the surrogate problem. Following standard zeroth-order optimization techniques (e.g.,  see \citet{PougkKal}), we note that, in this case, the only algorithmic adjustments that need to be made relate to the computation of the sample gradients of $f$, which would require two (instead of one) evaluations of the policy oracle. In light of this discussion, we conclude that Assumption \ref{assump:3} can be utilized (almost) without loss of generality.

%========================
% remove or comment if not including

% \subsection{Relaxing Assumption \ref{assum:solver}}
\subsection{Proof of Main Theorem}
\label{appendix:main_thm}

In this section, we provide the proof of Theorem \ref{thm:main}. To prove the main theorem we will need the following auxiliary lemmas. The proof extends the analysis given in \cite{pmlr-v119-lin20a}, since the presented algorithm only assumes having access to biased and inexact (sub)gradient samples. 

Denote by $b_1(\theta^*,\theta^\dagger,t,\lambda), b_2(\theta^*,\theta^\dagger,t,\lambda)$ the components of $b(\theta^*,\theta^\dagger,t,\lambda)$ corresponding to $t$ and $\lambda$ from Assumption \ref{assum:local_opt}.

\begin{lemma}\label{lemma:bias_variance}
    $ \hat{\nabla}_t \mathcal{\hat{L}}(\pi_{\theta^\dagger(t,\lambda)}, t, \lambda))$ and $ \hat{\nabla}_\lambda \mathcal{\hat{L}}(\pi_{\theta^\dagger(t,\lambda)}, t, \lambda))$ have bounded bias and variance.
\end{lemma}
\begin{proof}
    By Assumption \ref{assum:local_opt} and by linearity of expectation, we have
    \begin{align*}
        \E\left[  \hat{\nabla}_{t,\lambda} \mathcal{\hat{L}}(\pi_{\theta^\dagger(t,\lambda)}, t, \lambda)) \right] 
        &= \E\left[  \hat{\nabla}_{t,\lambda} \mathcal{\hat{L}}(\pi_{\theta^*(t,\lambda)}, t, \lambda))  +  b(\theta^*, \theta^\dagger, t,\lambda) \right] \\
        &= -\nabla f(t,\lambda) + \E [b(\theta^*, \theta^\dagger, t,\lambda)]\,. \tag{by Assumptions \ref{assum:local_opt}, \ref{assump:1}}
    \end{align*}
The bias of the gradient is controlled by $ b(\theta^*,\theta^{\dagger},t,\lambda)$, whose norm is at most $\delta$.
Furthermore, 
\begin{align*}
    &\E\left[ \norm{  \hat{\nabla}_{t,\lambda} \mathcal{\hat{L}}(\pi_{\theta^\dagger(t,\lambda)}, t, \lambda)) }^2 \right] \\
    &= \E\left[ \norm{  \hat{\nabla}_{t,\lambda} \mathcal{\hat{L}}(\pi_{\theta^\dagger(t,\lambda)}, t, \lambda)) -  \hat{\nabla}_{t,\lambda} \mathcal{\hat{L}}(\pi_{\theta^*(t,\lambda)}, t, \lambda)) +  \hat{\nabla}_{t,\lambda} \mathcal{\hat{L}}(\pi_{\theta^*(t,\lambda)}, t, \lambda)) }^2 \right] \\
    &\leq 2\E\insquare{\norm{ b(\theta^*, \theta^\dagger,t,\lambda)}^2} + 2\E\insquare{\norm{ \hat{\nabla}_{t,\lambda} \mathcal{\hat{L}}(\pi_{\theta^*(t,\lambda)}, t, \lambda))}^2} \\
    &\leq 2\norm{\nabla f(t,\lambda)}^2 + 2(\sigma^2 + \delta^2)\,. \tag{by Lemma A.2 of \cite{pmlr-v119-lin20a} with $M=1$ and Assumptions \ref{assum:local_opt}, \ref{assump:1}}
\end{align*}
\end{proof}

\begin{lemma}\label{lem:d3}
    Let $\Delta^{(k)} = \E\insquare{\Phi(t^{(k)}) - f(t^{(k)}, \lambda^{(k)})}$.  The following holds for Algorithm \textnormal{\ref{alg:main}}:
    \begin{align*}
        \E\insquare{\Phi_{1/2\ell}(t^{(k)})} &\leq \E\insquare{\Phi_{1/2\ell}(t^{(k-1)})} + 2\eta_t\ell\Delta^{(k-1)} - \frac{\eta_t}{4}\E\norm{\nabla \Phi_{1/2\ell}(t^{(k-1)})}^2 \\ &\qquad + 2\eta_t \delta \ell \cdot \mathrm{diam}(\mathcal{T}) + 3\eta_t^2\ell \cdot (C^2 + \sigma^2 + \delta^2)\,.
    \end{align*}
\end{lemma}
\begin{proof}
    Let $\hat{t}^{(k-1)} = \mathrm{prox}_{\Phi/2\ell}(t^{(k-1)})$ and fix $\pi_{\theta^*} = \pi_{\theta^*(t^{(k-1)},\lambda^{(k-1)})}$ where $\theta^* \in \arg\max_{\theta \in \Theta} \mathcal{L}(\pi_{\theta},t^{(k-1)},\lambda^{(k-1)})$, and $\pi_{\theta^{\dagger}} = \pi_{\theta^{\dagger}(t^{(k-1)},\lambda^{(k-1)})}$ satisfying Assumption \ref{assum:local_opt}. Then, 
    % Since $f(\cdot, \lambda)$ is $C$-Lipschitz by Lemma \ref{lemma:lipschitz}, 
    \begin{align*}
        \norm{\hat{t}^{(k-1)} - t^{(k)}}^2  &= \norm{\hat{t}^{(k-1)} - \Pi_{\mathcal{T}}\inparen{t^{(k-1)} - \eta_t\cdot \hat{\nabla}_t \hat{\mathcal{L}}(\pi_{\theta^\dagger},t^{(k-1)}, \lambda^{(k-1)})}}^2 \\
        &\leq \norm{\hat{t}^{(k-1)} - t^{(k-1)} + \eta_t\cdot \hat{\nabla}_t \hat{\mathcal{L}}(\pi_{\theta^\dagger},t^{(k-1)}, \lambda^{(k-1)})}^2 \tag{projection is non-expansive}\\
        &= \norm{\hat{t}^{(k-1)}-t^{(k-1)}}^2 + \eta_t^2\cdot \norm{\hat{\nabla}_t \hat{\mathcal{L}}(\pi_{\theta^\dagger},t^{(k-1)}, \lambda^{(k-1)})}^2 \\ & \qquad + 2\eta_t \inangle{\hat{t}^{(k-1)}-{t}^{(k-1)}, \hat{\nabla}_t \hat{\mathcal{L}}(\pi_{\theta^\dagger},t^{(k-1)}, \lambda^{(k-1)})}\,.
    \end{align*}
    Taking expectations, conditioned on $(t_{k-1}, \lambda_{k-1})$, and using Assumption \ref{assum:local_opt} we have 
    \begin{align*}
        &\E\insquare{\norm{\hat{t}^{(k-1)} - t^{(k)}}^2 \mid (t^{(k-1)}, \lambda^{(k-1)})} \\
        &\quad \leq \E\insquare{\norm{\hat{t}^{(k-1)}-t^{(k-1)}}^2 \mid (t^{(k-1)}, \lambda^{(k-1)}) } \\
        &\qquad + 2\eta_t\cdot \inangle{{t}^{(k-1)}-\hat{t}^{(k-1)}, \nabla_t f(t^{(k-1)}, \lambda^{(k-1)}) - b(\theta^*,\theta^\dagger,t^{(k-1)},\lambda^{(k-1)})} 
        \\ &\qquad + 3\eta_t^2 \norm{b_1(\theta^*,\theta^\dagger,t^{(k-1)},\lambda^{(k-1)})}^2+ 3\eta_t^2\norm{\nabla_t f(t^{(k-1)}, \lambda^{(k-1)})}^2 \\
        &\qquad  + 3\eta_t^2 \E\insquare{\norm{\hat{\nabla}_t \hat{\mathcal{L}}(\pi_{\theta^*},t^{(k-1)}, \lambda^{(k-1)}) + \nabla f(t^{(k-1)}, \lambda^{(k-1)})}^2 \mid (t^{(k-1)},\lambda^{(k-1)})}
        \tag{using $\norm{x + y + z}^2 \leq 3\norm{x}^2 + 3\norm{y}^2 + 3\norm{z}^2$}\,.
    \end{align*}

    Taking expectations of both sides and using Lemma A.2 of \cite{pmlr-v119-lin20a}, that $f(\cdot,\lambda)$ is $C$-Lipschitz (cf. Lemma \ref{lemma:lipschitz}), and Assumption \ref{assum:local_opt},
    \begin{align*}
        \E{\norm{\hat{t}^{(k-1)} - t^{(k)}}^2} &\leq \E{\norm{\hat{t}^{(k-1)}-t^{(k-1)}}^2} \\
        &\quad +2 \eta_t \inangle{\hat{t}^{(k-1)}-{t}^{(k-1)}, \nabla f(t^{(k-1)}, \lambda^{(k-1)}) - b(\theta^*,\theta^\dagger,t^{(k-1)},\lambda^{(k-1)})} \\&\quad + 3\eta_t^2(C^2 + \sigma^2 + \delta^2) \,.
    \end{align*}
    Since $\mathcal{T}$ is compact, define $\mathrm{diam}(\mathcal{T}) \coloneqq \max_{x, y \in \mathcal{T}}\norm{x - y}$. Using Cauchy-Schwarz on the term
    \begin{equation*} 
    \begin{split}
    & 2 \eta_t \inangle{{t}^{(k-1)}-\hat{t}^{(k-1)}, b_1(\theta^*,\theta^\dagger,t^{(k-1)},\lambda^{(k-1)})}\\
    &\quad \leq 2\eta_t \norm{\hat{t}^{(k-1)}-t^{(k-1)}} \cdot \norm{b_1(\theta^*,\theta^\dagger,t^{(k-1)},\lambda^{(k-1)})},
    \end{split}
    \end{equation*}
    and substituting, yields
    \begin{align}
        \begin{split}
            \E{\norm{\hat{t}^{(k-1)} - t^{(k)}}^2} &\leq \E{\norm{\hat{t}^{(k-1)}-t^{(k-1)}}^2} + 2 \eta_t \inangle{\hat{t}^{(k-1)}-{t}^{(k-1)}, \nabla f(t^{(k-1)}, \lambda^{(k-1)})} \\ 
            &\qquad + 2\eta_t \delta \cdot \mathrm{diam}(\mathcal{T}) + 3\eta_t^2(C^2 + \sigma^2 + \delta^2) \,.
            \label{eq:a3end}
        \end{split}
    \end{align}
    From \cite[Eq. (20)]{pmlr-v119-lin20a}, we obtain 
    \begin{align}
        \Phi_{1/2\ell}(t^{(k)}) &\leq \Phi_{1/2\ell}(t^{(k-1)}) + \ell \norm{\hat{t}^{(k-1)} - t^{(k)}}^2
        \label{eq:lineq20}
    \end{align}
    while, from \cite[Enq. (23), (24)]{pmlr-v119-lin20a}, we obtain
    \begin{align}
        &\inangle{\hat{t}^{(k-1)} - t^{(k-1)}, \nabla_t f(t^{(k-1)}, \lambda^{(k-1)})} \nonumber \\ 
        &\quad \leq f(\hat{t}^{(k-1)}, \lambda^{(k-1)}) - f(t^{(k-1)}, \lambda^{(k-1)}) + \frac{\ell}{2}\norm{\hat{t}^{(k-1)} - t^{(k-1)}}^2 
        \label{eq:lineq23}\\
        & f(\hat{t}^{(k-1)}, \lambda^{(k-1)}) - f({t}^{(k-1)}, \lambda^{(k-1)}) \nonumber \\&\quad \leq \Phi(\hat{t}^{(k-1)}) - f({t}^{(k-1)}, \lambda^{(k-1)}) \leq \Delta^{(k-1)} - \frac{\ell}{2}\norm{\hat{t}^{(k-1)} - t^{(k-1)}}^2.
        \label{eq:lineq24}
    \end{align}
    Plugging in \eqref{eq:a3end} to \eqref{eq:lineq20}, then combining \eqref{eq:lineq23}, \eqref{eq:lineq24}, and the fact that $\norm{\hat{t}^{(k-1)} - t^{(k-1)}} = \norm{\nabla \Phi_{1/2\ell}(t^{(k-1)})}/2\ell$, finally gives
    \begin{align*}
        \E\insquare{\Phi_{1/2\ell}(t^{(k)})} &\leq \E\insquare{\Phi_{1/2\ell}(t^{(k-1)})} + 2\eta_t\ell\Delta^{(k-1)} - \frac{\eta_t}{4} \E \norm{\nabla \Phi_{1/2\ell}(t^{(k-1)})}^2 \\ &\qquad + 2\eta_t \delta \ell \cdot \mathrm{diam}(\mathcal{T}) + 3\eta_t^2\ell (C^2 + \sigma^2 + \delta^2)\,,
    \end{align*}
\noindent which completes the proof.
\end{proof}

\begin{lemma}\label{lem:d4}
    Let $\Delta^{(k)} = \E\insquare{\Phi(t^{(k)}) - f(t^{(k)}, \lambda^{(k)})}$ and $\lambda^*(t) \in \argmax_{\lambda  \in \Lambda} f(t, \lambda) $. The following holds for all $s \leq k-1$:
\begin{align*}
    \Delta^{(k-1)} &\leq 2 \eta_t C\sqrt{C^2 + \sigma^2 + \delta^2}(2t -2s - 1) \\
    &\quad + \frac{1}{2\eta_\lambda}\inparen{\E\norm{\lambda^{(k-1)} - \lambda^*(t^{(s)})}^2 - \E\norm{\lambda^{(k)} - \lambda^*(t^{(s)})}^2 } \\
    &\quad + \E\insquare{f(t^{(k)}, \lambda^{(k)}) - f(t^{(k-1)}, \lambda^{(k-1)})} +\delta \cdot \mathrm{diam}(\Lambda) +  2\eta_\lambda(\sigma^2 +  \delta^2) \,.
\end{align*}
\end{lemma}
\begin{proof}
    For any $\lambda \in \Lambda$, the update of $\lambda^{(k)}$ and convexity of $\Lambda$ imply that 
    \[ (\lambda - \lambda^{(k)})^\top \left(\lambda^{(k)} - \lambda^{(k-1)} - \eta_\lambda \nabla_\lambda f(t^{(k-1)}, \lambda^{(k-1)})\right) \geq 0\,. \]
% Then following \cite{pmlr-v119-lin20a},
%     \begin{align*}
%         \inparen{\lambda - \lambda_k}^2
%         &\leq 2\eta\lambda (\lambda_{k-1} - \lambda)\cdot \hat{\nabla}_\lambda \hat{\mathcal{L}}(\pi_{\theta^\dagger}, t_{k-1},\lambda_{k-1}) + 2\eta\lambda (\lambda_{k} - \lambda_{k-1})\cdot \frac{\partial}{\partial \lambda} f(t_{k-1}, \lambda_{k-1}) \\
%         &\quad + 2\eta\lambda (\lambda_{k} - \lambda_{k-1})\cdot \inparen{\hat{\nabla}_\lambda \hat{\mathcal{L}}(\pi_{\theta^\dagger}, t_{k-1},\lambda_{k-1}) - \frac{\partial}{\partial \lambda} f(t_{k-1}, \lambda_{k-1})} \\
%         &\qquad + (\lambda - \lambda_{k-1})^2 - (\lambda_{k} - \lambda_{k-1})^2\,.
%     \end{align*}
Then, we have 
\begin{align*}
    \left\|\lambda - \lambda^{(k)}\right\|^2
    &\leq 2\eta_\lambda (\lambda^{(k-1)} - \lambda)^\top 
        \nabla_\lambda f(t^{(k-1)}, \lambda^{(k-1)}) \\
    &\quad + 2\eta_\lambda (\lambda^{(k)} - \lambda^{(k-1)})^\top 
        \nabla_\lambda f(t^{(k-1)}, \lambda^{(k-1)}) \\
    &\quad + \|\lambda - \lambda^{(k-1)}\|^2 - \|\lambda^{(k)} - \lambda^{(k-1)}\|^2 \\
    &\leq 2\eta_\lambda (\lambda^{(k-1)} - \lambda^{(k)})^\top 
        \hat{\nabla}_\lambda \hat{\mathcal{L}}(\pi_{\theta^\dagger}, t^{(k-1)},\lambda^{(k-1)}) \\
    &\quad + 2\eta_\lambda (\lambda^{(k-1)} - \lambda)^\top 
        \nabla_\lambda f(t^{(k-1)}, \lambda^{(k-1)}) \\
    &\quad + 2\eta_\lambda (\lambda^{(k)} - \lambda^{(k-1)})^\top
        \Biggl(
            \nabla_\lambda f(t^{(k-1)}, \lambda^{(k-1)})
            + \hat{\nabla}_\lambda \hat{\mathcal{L}}(\pi_{\theta^\dagger}, t^{(k-1)},\lambda^{(k-1)})
        \Biggr) \\
    &\quad + \|\lambda - \lambda^{(k-1)}\|^2 - \|\lambda^{(k)} - \lambda^{(k-1)}\|^2\,,
\end{align*}
\noindent where we let $\pi_{\theta^{\dagger}} = \pi_{\theta^{\dagger}(t^{(k-1)},\lambda^{(k-1)})}$ (satisfying Assumption \ref{assum:local_opt}). Using Young's inequality,
\begin{align*}
    &\eta_\lambda (\lambda^{(k)} - \lambda^{(k-1)})^\top \inparen{\nabla_\lambda f(t^{(k-1)}, \lambda^{(k-1)}) + \hat{\nabla}_\lambda \hat{\mathcal{L}}(\pi_{\theta^\dagger}, t^{(k-1)},\lambda^{(k-1)})} \\
    &\leq \frac{1}{4}\norm{\lambda^{(k)} - \lambda^{(k-1)}}^2 + {\eta_\lambda^2} \norm{\hat{\nabla}_\lambda \hat{\mathcal{L}}(\pi_{\theta^\dagger}, t^{(k-1)},\lambda^{(k-1)}) + \nabla_\lambda f(t^{(k-1)}, \lambda^{(k-1)})}^2\,.
\end{align*}
Taking expectations on both sides, conditioned on $(t^{(k-1)},\lambda^{(k-1)})$, gives
\begin{align*}
    &\E\left[\norm{\lambda - \lambda^{(k)}}^2 \mid (t^{(k-1)}, \lambda^{(k-1)})\right] \\
    &\leq 2\eta_\lambda \inparen{\lambda^{(k)} - \lambda^{(k-1)}}^\top
        \Bigl(-\nabla_\lambda f(t^{(k-1)}, \lambda^{(k-1)}) 
            + b_2(\theta^*,\theta^{\dagger},t^{(k-1)},\lambda^{(k-1)})\Bigr) 
        \tag{Assumption \ref{assum:local_opt}} \\
    &\quad + 2\eta_\lambda \E\left[\inparen{\lambda^{(k-1)} - \lambda}^\top 
        \nabla_\lambda f(t^{(k-1)}, \lambda^{(k-1)}) 
        \,\middle|\, (t^{(k-1)},\lambda^{(k-1)})\right] \\
    &\quad + \tfrac{1}{2}\E\left[\norm{\lambda^{(k)} - \lambda^{(k-1)}}^2 
        \,\middle|\, (t^{(k-1)},\lambda^{(k-1)})\right] \\
    &\quad + 2\eta_\lambda^2 \E\Biggl[ 
        \norm{\hat{\nabla}_\lambda \hat{\mathcal{L}}(\pi_{\theta^\dagger}, t^{(k-1)},\lambda^{(k-1)}) 
        + \nabla_\lambda f(t^{(k-1)}, \lambda^{(k-1)})}^2 
        \,\Bigm|\, (t^{(k-1)}, \lambda^{(k-1)})\Biggr] \\
    &\quad + \E\left[\norm{\lambda - \lambda^{(k-1)}}^2 \,\middle|\, (t^{(k-1)}, \lambda^{(k-1)})\right] 
        - \E\left[\norm{\lambda^{(k)} - \lambda^{(k-1)}}^2 \,\middle|\, (t^{(k-1)}, \lambda^{(k-1)})\right].
\end{align*}
Taking expectation on both sides (and Cauchy-Schwarz),
\begin{align*}
    & \E\norm{\lambda - \lambda^{(k)}}^2 \\
    &\quad \leq 2\eta_\lambda\E\bigg[ \inparen{\lambda^{(k-1)} - \lambda^{(k)}}^\top \nabla_\lambda f(t^{(k-1)},\lambda^{(k-1)}) +  \inparen{\lambda^{(k-1)} - \lambda}^\top \nabla_\lambda f(t^{(k-1)},\lambda^{(k-1)})\bigg] \\
    &\qquad + \E\norm{\lambda - \lambda^{(k-1)}}^2 - \frac{1}{2}\E\norm{\lambda^{(k)} - \lambda^{(k-1)}}^2 + 2\eta_\lambda \norm{\lambda^{(k-1)} - \lambda^{(k)}}\cdot \norm{b_2(t^{(k-1)},\lambda^{(k-1)})} \\
    &\qquad + 2\eta_\lambda^2 \E \norm{\hat{\nabla}_\lambda \hat{\mathcal{L}}(\pi_{\theta^\dagger}, t^{(k-1)},\lambda^{(k-1)}) + \nabla_\lambda f(t^{(k-1)}, \lambda^{(k-1)})}^2 \,.
\end{align*}
\noindent From Lemma \ref{lemma:bias_variance}, \cite[Lemma A.2]{pmlr-v119-lin20a}, and since $\mathrm{diam}(\Lambda)$ is finite, we get
\begin{align*}
   & \E\norm{\lambda - \lambda^{(k)}}^2 \\
   &\quad \leq 2\eta_\lambda\E\insquare{ \inparen{\lambda^{(k-1)} - \lambda^{(k)}}^\top \nabla_\lambda f(t^{(k-1)},\lambda^{(k-1)}) +  \inparen{\lambda^{(k-1)} - \lambda}^\top \nabla_\lambda f(t^{(k-1)},\lambda^{(k-1)})} \\
    &\qquad + \E\norm{\lambda - \lambda^{(k-1)}}^2 - \frac{1}{2}\E\|\lambda^{(k)} - \lambda^{(k-1)}\|^2 + 2\eta_\lambda \delta \cdot \mathrm{diam}(\Lambda) + 4\eta_\lambda^2 (\sigma^2 + \delta^2) \,.
\end{align*}

Since $f(t_{k-1}, \cdot)$ is concave and $\Lambda$ is convex (take $\eta_\lambda \leq 1/2\ell$), 
\begin{align*}
    \E\norm{\lambda - \lambda^{(k)}}^2 &\leq
    \E\norm{\lambda - \lambda^{(k-1)}}^2 + 2\eta_\lambda (f(t^{(k-1)},\lambda^{(k)}) - f(t^{(k-1)}, \lambda)) \\ &\qquad+ 2\eta_\lambda \delta \cdot \mathrm{diam}(\Lambda) + 4\eta_\lambda^2 (\sigma^2 + \delta^2) \,.
\end{align*}
Substituting $\lambda = \lambda^*(t^{(s)})$ (where $s \leq k-1$),
\begin{align*} &\E\insquare{f(t^{(k-1)}, \lambda^*(t^{(s)})) - f(t^{(k-1)}, \lambda^{(k)})} \\
&\quad \leq \frac{1}{2\eta_\lambda}\inparen{\E\norm{\lambda^{(k-1)} - \lambda^*(t^{(s)})}^2 - \E\norm{\lambda^{(k)} - \lambda^*(t^{(s)})}^2 } + \delta \cdot \mathrm{diam}(\Lambda) + 2{\eta_\lambda(\sigma^2 + \delta^2)}\,.
\end{align*}

By the definition of $\Delta^{(k-1)}$,
\begin{align*}
    \Delta^{(k-1)} 
    &\leq \E\Biggl[
        f(t^{(k-1)},\lambda^*(t^{(k-1)})) - f(t^{(k-1)},\lambda^*(t^{(s)})) \\
    &\qquad\quad + \bigl(f(t^{(k)}, \lambda^{(k)}) - f(t^{(k-1)}, \lambda^{(k-1)})\bigr) 
        + \bigl(f(t^{(k-1)}, \lambda^{(k)})-f(t^{(k)}, \lambda^{(k)})\bigr)
    \Biggr] \\
    &\quad + \frac{1}{2\eta_\lambda}\left(
        \E\norm{\lambda^{(k-1)} - \lambda^*(t^{(s)})}^2 
        - \E\norm{\lambda^{(k)} - \lambda^*(t^{(s)})}^2 
    \right) \\
    &\quad + \delta \cdot \mathrm{diam}(\Lambda) + 2\eta_\lambda(\sigma^2 + \delta^2)\,.
\end{align*}

Following the steps in \cite[Lemma D.4]{pmlr-v119-lin20a}, using that $f(\cdot, \lambda)$ is $C$-Lipschitz (by Lemma \ref{lemma:lipschitz}) and Lemma \ref{lemma:bias_variance}, we have
\begin{align*}
    \E\insquare{f(t^{(k-1)}, \lambda^*(t^{(k-1)})) - f(t^{(s)}, \lambda^*(t^{(k-1)}))} &\leq 2\eta_t C \sqrt{C^2 + \sigma^2 + \delta^2}(t-1-s) \\
    \E\insquare{f(t^{(s)}, \lambda^*(t^{(s)})) - f(t^{(k-1)}, \lambda^*(t^{(s)}))} &\leq 2 \eta_t C \sqrt{C^2 + \sigma^2 + \delta^2}(t-1-s)\\
    \E\insquare{f(t^{(k-1)}, \lambda^{(k)}) - f(t^{(k)}, \lambda^{(k)})} &\leq 2\eta_t C \sqrt{C^2 + \sigma^2 + \delta^2}\,.
\end{align*}
Using \cite[Eqn. (25)]{pmlr-v119-lin20a}, 
\begin{align*}
    \Delta^{(k-1)} &\leq 2\eta_t C\sqrt{C^2 + \sigma^2 + \delta^2} \cdot (2t -2s - 1) \\ & \qquad + \frac{1}{2\eta_\lambda}\inparen{\E\norm{\lambda^{(k-1)} - \lambda^*(t^{(s)})}^2 - \E\norm{\lambda^{(k)} - \lambda^*(t^{(s)})}^2 } \\
    &\qquad + \E\insquare{f(t^{(k)}, \lambda^{(k)}) - f(t^{(k-1)}, \lambda^{(k-1)})} +\delta \cdot \mathrm{diam}(\Lambda) +  2\eta_\lambda(\sigma^2 +  \delta^2) \,.
\end{align*}
\end{proof}

\begin{lemma}\label{lem:d5}
    Let $\Delta^{(k)} = \E\insquare{\Phi(t^{(k)}) - f(t^{(k)}, \lambda^{(k)})}$. Let $B \leq J+1$ be such that $(J+1)/B$ is an integer. The following holds
\begin{align*}
    \frac{1}{J+1}\left(\sum_{k=0}^J \Delta^{(k)} \right) 
    &\leq 2 \eta_t C \sqrt{C^2 + \sigma^2 + \delta^2}(B+1) \\
    &\quad + \frac{\mathrm{diam}(\Lambda)^2}{2B \eta_\lambda} 
        + \delta \cdot \mathrm{diam}(\Lambda) 
        + 2 \eta_\lambda (\sigma^2 + \delta^2) 
        + \frac{\hat{\Delta}_0}{J+1}\,.
\end{align*}
\end{lemma}
\begin{proof}
    We divide $\inbrace{\Delta^{(k)}}_{k=0}^J$ into blocks where each block contains at most $B$ terms:
    \[ \inbrace{\Delta^{(k)}}_{k=0}^{B-1}, \inbrace{\Delta^{(k)}}_{k=B}^{2B-1}, \dots, \inbrace{\Delta^{(k)}}_{k=J-B + 1}^J\,.\]
Then,
\begin{align}
    \frac{1}{J+1}\inparen{\sum_{k=0}^J \Delta^{(k)}} \leq \frac{B}{J+1}\insquare{\sum_{i=0}^{(J+1)/B - 1}\inparen{\frac{1}{B}\sum_{k= i B}^{(i+1)B - 1} \Delta ^{(k)}}}\,. 
    \label{eq:d5sum}
\end{align}

Letting $s = 0$ and applying Lemma \ref{lem:d4}, 
\begin{align*}
    \sum_{k=0}^{B-1}\Delta^{(k)} 
    &\leq 2\eta_t C \sqrt{C^2 + \sigma^2 + \delta^2}\,B^2 
        + \frac{1}{2\eta_\lambda}\E\norm{\lambda^{(0)} - \lambda^*(t^{(0)})}^2  \\
    &\quad + \E\left[f(t^{(B)}, \lambda^{(B)}) - f(t^{(0)}, \lambda^{(0)})\right] 
        + \delta B \cdot \mathrm{diam}(\Lambda) 
        + 2 \eta_\lambda B(\sigma^2 +  \delta^2) \\
    &\leq 2 \eta_t C \sqrt{C^2 + \sigma^2 + \delta^2}\,B^2 
        + \E\left[f(t^{(B)}, \lambda^{(B)}) - f(t^{(0)}, \lambda^{(0)})\right] \\
    &\quad + \frac{\mathrm{diam}(\Lambda)^2}{2\eta_\lambda} 
        + \delta B \cdot \mathrm{diam}(\Lambda) 
        + 2 \eta_\lambda B(\sigma^2 +  \delta^2)\,.
\end{align*}

Letting $s = iB$ and applying Lemma \ref{lem:d4},
\begin{align*}
    \sum_{k=iB}^{(i+1)B - 1}\Delta^{(k)} 
    &\leq 2 \eta_t C \sqrt{C^2 + \sigma^2 + \delta^2}B^2 + \E\insquare{f(t^{(iB+B)}, \lambda^{(iB+B)}) - f(t^{(iB)}, \lambda^{(iB)})}\\
     &\qquad + \frac{\mathrm{diam}(\Lambda)^2}{2\eta_\lambda} +\delta B \cdot \mathrm{diam}(\Lambda) +  2 \eta_\lambda B(\sigma^2 +  \delta^2) \,.
\end{align*}
Substituting these into \eqref{eq:d5sum}:
\begin{align*}
    \frac{1}{J+1}\inparen{\sum_{k=0}^J \Delta^{(k)}} &\leq 2 \eta_t C \sqrt{C^2 + \sigma^2 + \delta^2}B + \frac{1}{J+1}\E\insquare{f(t_{J+1}, \lambda_{J+1}) - f(t_{0}, \lambda_{0})} \\
    &\qquad + \frac{\mathrm{diam}(\Lambda)^2}{2B \eta_\lambda} +\delta \cdot \mathrm{diam}(\Lambda) +  2 \eta_\lambda (\sigma^2 +  \delta^2) \,.
\end{align*}
By the Lipschitzness of $f(\cdot, \lambda)$,
\begin{align*}
    f(t_{J+1},\lambda_{J+1}) - f(t_0, \lambda_0) &\leq \eta_t C^2(J + 1) + \hat{\Delta}_0\,.
\end{align*}
This yields
\begin{align*}
    \frac{1}{J+1}\left(\sum_{k=0}^J \Delta_k\right) 
    &\leq 2 \eta_t C \sqrt{C^2 + \sigma^2 + \delta^2}\,(B+1) \\
    &\quad + \frac{\mathrm{diam}(\Lambda)^2}{2B \eta_\lambda} 
        + \delta \cdot \mathrm{diam}(\Lambda)  + 2 \eta_\lambda (\sigma^2 + \delta^2) 
        + \frac{\hat{\Delta}_0}{J+1}\,.
\end{align*}

\end{proof}

\subsubsection*{Proof of Theorem \ref{thm:main}}
Summing up the inequality from Lemma \ref{lem:d3}, over $k = 1, \dots, J+1$, yields
\begin{align*}
    \E\insquare{\Phi_{1/2\ell}(t^{(J+1)})} &\leq \Phi_{1/2\ell}(t^{(0)}) + 2 \eta_t \ell \inparen{\sum_{k=0}^J \Delta^{(k)}} - \frac{\eta_t}{4}\inparen{\sum_{k=0}^J \E \norm{\nabla \Phi_{1/2\ell}(t^{(k)})}^2 } \\
    &\qquad + \inparen{2\eta_t\delta \ell\cdot \mathrm{diam}(\mathcal{T}) + 3\eta_t^2 \ell (C^2 + \sigma^2 + \delta^2)}(J+1)\,.
\end{align*}

Applying Lemma \ref{lem:d5},
\begin{align*}
    \E\left[\Phi_{1/2\ell}(t^{(J+1)})\right] 
    &\leq \Phi_{1/2\ell}(t^{(0)}) \\
    &\quad + 2 \eta_t \ell (J+1)\Biggl( 
        2\eta_t C \sqrt{C^2 + \sigma^2 + \delta^2}(B+1) \\
    &\qquad\qquad + \frac{\mathrm{diam}(\Lambda)^2}{2B \eta_\lambda} 
        + \delta \cdot \mathrm{diam}(\Lambda) 
        + 2\eta_\lambda (\sigma^2 + \delta^2) 
    \Biggr) \\
    &\quad + 2\eta_t\ell \hat{\Delta}_0 
        - \frac{\eta_t}{4}\left(\sum_{k=0}^J \E \norm{\nabla \Phi_{1/2\ell}(t^{(k)})}^2 \right) \\
    &\quad + \left(2\eta_t \delta \ell \cdot \mathrm{diam}(\mathcal{T}) 
        + 3 \eta_t^2\ell (C^2 + \sigma^2 + \delta^2)\right)(J+1)\,.
\end{align*}

By the definition of $\hat{\Delta}_\Phi$, we obtain
\begin{align*}
    &\frac{1}{J+1}\inparen{\sum_{k=0}^J \E\norm{\nabla \Phi_{1/2\ell}(t^{(k)})}^2 } \\
    &\leq \frac{4 \hat{\Delta}_\Phi}{\eta_t(J+1)} + 8 \ell \inparen{ 2\eta_t C \sqrt{C^2 + \sigma^2 + \delta^2}(B+1) + \frac{\mathrm{diam}(\Lambda)^2}{2B \eta_\lambda} +\delta \mathrm{diam}(\Lambda) +  2 \eta_\lambda (\sigma^2 +  \delta^2) } \\
    & \qquad + \frac{8\ell \hat{\Delta}_0}{J+1} + 12 \eta_t\ell(C^2 + \sigma^2 + \delta^2) + 8\ell \delta\cdot \mathrm{diam}(\mathcal{T}) \,.
\end{align*}

\par Letting $B = \frac{\mathrm{diam}(\Lambda)}{2}\sqrt{\frac{1}{\eta_t\eta_\lambda C \sqrt{C^2 + \sigma^2 + \delta^2}}}$, we obtain
\begin{align*}
    \frac{1}{J+1}\left(\sum_{k=0}^J \E\norm{\nabla \Phi_{1/2\ell}(t^{(k)})}^2 \right) 
    &\leq \frac{4 \hat{\Delta}_\Phi}{\eta_t(J+1)} 
        + 24\ell \,\mathrm{diam}(\Lambda)\,
            \sqrt{\tfrac{\eta_t C \sqrt{C^2 + \sigma^2 + \delta^2}}{\eta_\lambda}} \\
    &\quad + 16 \eta_\lambda \ell (\sigma^2 + \delta^2) 
        + \frac{8\ell\hat{\Delta}_0}{J+1} \\
    &\quad + 12 \eta_t \ell (C^2 + \sigma^2 + \delta^2) 
        + 8\ell \delta\cdot \left(\mathrm{diam}(\mathcal{T}) + \mathrm{diam}(\Lambda)\right)\,.
\end{align*}
With the choice of step sizes
\begin{align*}
    \eta_\lambda &= \min\left\{ \frac{1}{2\ell}, \frac{\epsilon^2}{16\ell (\sigma^2 + \delta^2)} \right\},  \numberthis \label{eq:eta_l} \\
    \eta_t &= \min\left\{
        \begin{aligned}
            &\mathcal{O}\!\left(\frac{\epsilon^2}{\ell\,(C^2 + \sigma^2 + \delta^2)}\right),\\
            &\mathcal{O}\!\left(\frac{\epsilon^4}{\ell^3 \mathrm{diam}(\Lambda)^2}\,
                C \sqrt{C^2 + \sigma^2 + \delta^2}\right),\\
            &\mathcal{O}\!\left(\frac{\epsilon^6}{\ell^3 \mathrm{diam}(\Lambda)^2\,
                (\sigma^2 + \delta^2)\, C \sqrt{C^2 + \sigma^2 + \delta^2}}\right)
        \end{aligned}
    \right\} \numberthis \label{eq:eta_t},
\end{align*}
we have
\begin{align*}
    &\frac{1}{J+1}\inparen{\sum_{k=0}^J \E\norm{\nabla \Phi_{1/2\ell}(t_k)}^2 }\\ &\quad\leq \frac{4 \hat{\Delta}_\Phi}{\eta_t(J+1)} + \frac{8\ell\hat{\Delta}_0}{J+1} 
    + 8\ell \delta\cdot \left(\mathrm{diam}(\mathcal{T}) + \mathrm{diam}(\Lambda)\right)
    + \mathcal{O}(\epsilon^2)\,.
\end{align*}

Thus, we have an iteration complexity of 
\[ \mathcal{O}\inparen{\inparen{\frac{\ell(C^2 + \sigma^2 + \delta^2) \cdot \hat{\Delta}_\Phi}{\epsilon^4} + \frac{\ell \hat{\Delta}_0}{\epsilon^2} } \cdot \max \inbrace{1, \frac{\ell^2 \mathrm{diam}(\Lambda)^2}{\epsilon^2}, \frac{\ell^2 \mathrm{diam}(\Lambda)^2 (\sigma^2 + \delta^2)}{\epsilon^4}}  }\, \]
\noindent for recovering an $\mathcal{O}(\sqrt{\epsilon^2 + \delta \ell( \mathrm{diam}(\mathcal{T})+ \mathrm{diam}(\Lambda))})$-stationary point (cf. Definition \textnormal{\ref{def:Stationary_point}}). By simplifying, we obtain the desired iteration bound: 
\[ \mathcal{O} \inparen{\inparen{\frac{\ell^3 (C^2 + \sigma^2 + \delta^2)(\mathrm{diam}(\Lambda))^2 \cdot \hat{\Delta}_\Phi}{\epsilon^6} + \frac{\ell^3(\mathrm{diam}(\Lambda))^2\cdot \hat{\Delta}_0}{\epsilon^4} } \max \inbrace{1, \frac{\sigma^2 + \delta^2}{\epsilon^2}} }\,.\] \qed

%=========================

\section{Experimental Details}
\label{appendix:experimental_details}

\subsection{Simulations}
\label{app:simulations}
Constraint tasks are a central class of real-world RL problems with applications in robotics, autonomous vehicles, and industrial control. To simulate realistic conditions, we use the locomotion tasks from the Safety-Gymnasium suite \citep{safety_gym}. Further details, including visualizations of both constrained and unconstrained agent behaviors, are available on their website\footnote{\url{https://safety-gymnasium.readthedocs.io/en/latest/environments/safe_velocity.html}}.

\paragraph{Velocity Cost}  
The velocity cost is defined as:  
\begin{equation*}
    \text{cost} = \operatorname{bool}(\operatorname{vel}_\text{current} > \operatorname{vel}_\text{threshold}),
\end{equation*}
where the velocity thresholds $\operatorname{vel}_\text{threshold}$ for the tested environments are listed in Table \ref{table:env_params}. As reported by \citet{gym}, these thresholds are set to 50\% of the maximum velocity achieved by each agent after PPO training for $10^7$ steps. At each time step, the agent's instantaneous velocity is computed as:
\begin{equation}
\label{eq:vel_fn}
    \operatorname{vel}_\text{current} = \sqrt{\operatorname{vel}_{x, \text{current}}^2 + \operatorname{vel}_{y, \text{current}}^2},
\end{equation}
where $\operatorname{vel}_{x, \text{current}}$ and $\operatorname{vel}_{y, \text{current}}$ are the agent’s instantaneous velocities along the x- and y-axes, respectively, as provided by the simulator.

\paragraph{Environments}  
We use the latest version (\texttt{v1}) of Safety-Gymnasium without modifying the state, action, or reward space. Constraint violations do not alter the agent’s behavior or environment dynamics. As a result, vanilla PPO converges to high reward solutions with large cumulative costs—effectively ignoring the constraints. All actions are normalized to the range $[-1, 1]$. To isolate the effects of our method, we reset the agent to the same initial state after every termination—a common practice in robotics.

% \paragraph{Excluded Environments}
% {There are two additional environments from the MuJoCo interface available in the Safety-Gymnasium suite: Ant and Hopper. However, prior research has shown that PPO often fails to converge to meaningful rewards in these environments—even to suboptimal levels. Since our method builds on PPO as the underlying solver, we excluded these tasks to avoid drawing misleading conclusions.}

\paragraph{Episode Terminations}  
Episodes terminate either when the time limit is reached or when the agent fails (e.g., by falling). Notably, HalfCheetah, Swimmer, and all safe navigation environments have no failure condition; their episodes always end due to the time limit.

\begin{table}[htbp]
\caption{Environment-specific parameters used in our experiments.}
\label{table:env_params}
\begin{center}
    \begin{tabular}{l c c c c c}
        \toprule
        \textbf{Variable} & Safe Navigation & HalfCheetah & Hopper & Swimmer & Walker2d \\
        \midrule
        Threshold $c$ & 0.0 & 3.2096 & 0.7402 & 0.2282 & 2.3415 \\
        Failure condition & \xmark & \xmark & \cmark & \xmark & \cmark \\
        Time limit (steps) & 500 & 1000 & 1000 & 1000 & 1000 \\
        \midrule
        Initial CVaR variable $t_\text{init}$ & 0.0 & -1.3 & -0.1 & -0.0 & -0.975 \\
        \bottomrule
    \end{tabular}
\end{center}
\end{table}

\paragraph{Stochasticity for Risk Management}  
The environment is fully deterministic—identical actions from the same state always yield the same rewards and transitions—making it difficult to evaluate risk-sensitive behavior. To simulate uncertainty without altering the environment's internal dynamics, we inject zero-mean Gaussian noise (std.\ 0.05, i.e., 5\% of the action range) into \textit{all} agent actions at every step during both training and evaluation. This controlled perturbation introduces stochasticity in action execution, enabling us to assess how well the agent manages risk under uncertain conditions while maintaining consistent environment behavior.

\paragraph{Evaluation}
Agents are evaluated every 1000 time steps by averaging the undiscounted sum of rewards over 10 episodes. The PPO agent uses the mean action, ensuring consistency in evaluation. Evaluations are entirely separate from training—no data is stored, and no network updates are performed.

\subsection{Proximal Policy Optimization -- \textit{solver}}
We use Proximal Policy Optimization \citep{ppo} as a \textit{solver} to learn a policy (i.e., line 4 in Algorithm \ref{alg:main}). Our method serves as a \textit{wrapper}, described next, that modifies the agent's raw reward to incorporate risk measures and constraints.

PPO first collects rollouts of state-action-reward sequences using the current policy, storing them as trajectories. Once sufficient data is gathered, it applies minibatch learning, splitting the rollout data into smaller batches and iteratively updating the policy over multiple epochs. During this process, the dual and CVaR variables remain fixed.

\paragraph{Neural Networks}
The value function and policy are approximated by neural networks, each with two hidden layers of 64 neurons using the $\operatorname{tanh}$ activation function. The value network processes states \( s \) and outputs a scalar value. The policy network takes states \( s \) as input, extracts hidden features, and passes them to a Gaussian distribution with learnable mean and standard deviation parameters. The action $a$ is then sampled from this distribution.

\begin{table}[hpt]
\caption{PPO hyperparameters used in the experiments.}
\label{appendix:ppo_params}
\vskip 0.15in
\begin{center}
\begin{tabular}{l c}
    \toprule
    \textbf{Hyperparameter} & \textbf{Value} \\
    \midrule
    Optimizer & Adam \\
    Learning rate (all networks) & $3 \times 10^{-4}$ \\
    Linear learning rate decay & \cmark \\
    Adam $\epsilon$ & $10^{-6}$ \\
    Adam $\alpha$ & 0.99 \\
    \midrule
    \# rollout steps & 2048 \\
    \# minibatches per rollout & 32 \\
    \# epochs & 10 \\
    \midrule
    Discount factor $\gamma$ & 0.99 \\
    GAE $\lambda$ & 0.95 \\
    Entropy coefficient & 0.0 \\
    Value loss coefficient & 0.5 \\
    Maximum gradient norm & 0.5 \\
    Clip parameter & 0.2 \\
    \bottomrule
\end{tabular}
\end{center}
\vskip -0.1in
\end{table}

\paragraph{Hyperparameters}
We employ Generalized Advantage Estimation (GAE) \citep{gae} to estimate advantages in PPO. The hyperparameters used by the PPO agent is provided in Table \ref{appendix:ppo_params}.

\subsection{Reward-Based SGD with Risk Constraints -- \textit{wrapper}}
We follow the same rollout strategy to optimize the dual and CVaR variables. First, the policy is updated using collected rollouts while keeping $\lambda$ and $t$ fixed. Then, a new rollout is collected with the updated policy and used to update $\lambda$ and $t$, while keeping the policy parameters frozen.

\subsubsection{Implementation}
\label{app:implementation}
To balance reward maximization with constraint handling, we frame the problem as standard risk-neutral reward maximization subject to a constraint that regulates violations through the conditional value-at-risk (CVaR) of the constraint quantity.

Let $r: \mathcal{S} \times \mathcal{A} \rightarrow \R$ be the reward function and $\upsilon: \mathcal{S} \times \mathcal{A} \rightarrow \R$ a constraint-quantifying function, e.g., velocity function in \eqref{eq:vel_fn}. We want to solve:
\[
\sup_{\pi \in \mathcal{P(S)}} \E \left[ \sum_{\tau=0}^\infty \gamma^\tau r(s_\tau,\pi(s_\tau)) \right] \quad\textnormal{s.t.}\quad \cvarb_{\nu^\pi}(\upsilon(s,a)) \leq c.
\]
This constraint can be equivalently written as
\[
\cvarb_{\nu^\pi}(\upsilon(s,a)) \leq c \iff -\cvarb_{\nu^\pi}(-\upsilon(s,a)) \geq -c,
\]
which aligns with the supremal convolution form of the reflected CVaR. Rewriting the constraint:
\begin{align*}
    -\cvarb_{\nu^\pi}(-\upsilon(s,a)) \geq -c &\iff \sup_{t \in \R} \E\left[ \sum_{\tau=0}^\infty \gamma^\tau \left( t - \frac{1}{\beta}(t + \upsilon(s_\tau, \pi(s_\tau)))_+ \right) \right] \geq -\frac{c}{1-\gamma}.
\end{align*}
Thus, we solve the following constrained problem:
\begin{equation*}
    \begin{aligned}
        \sup_{\pi \in \mathcal{P(S)}, t \in \R} & \E \left[ \sum_{\tau=0}^\infty \gamma^\tau r(s_\tau,\pi(s_\tau)) \right] \\
        \textnormal{s.t.}\quad\ \  & \E \left[ \sum_{\tau=0}^\infty \gamma^\tau \left( t - \frac{1}{\beta} (t + \upsilon(s_\tau,\pi(s_\tau)))_+ \right) \right] \geq -\frac{c}{1-\gamma}.
    \end{aligned}
\end{equation*}
We implement this by modifying the reward at each time step using the Lagrangian:
\begin{equation}
\label{eq:modified_reward}
    r(s_\tau, \pi(s_\tau)) + \lambda_i \left(c + t - \frac{1}{\beta}(t + \upsilon(s_\tau, \pi(s_\tau)))_+ \right),
\end{equation}
where PPO is used as a black-box solver for the inner maximization over $\pi$, while $\lambda$ and $t$ are updated using single stochastic gradient descent and ascent steps as in Algorithm~\ref{alg:main}.

\begin{table}[tp]
\caption{Common hyperparameter values used across all environments.}
\label{table:common_params}
\begin{center}
    \begin{tabular}{l c}
        \toprule
        \textbf{Variable} & \textbf{Value} \\
        \midrule
        VaR level $\beta$ & 0.3 \\
        \# trajectories used to compute gradients of $\lambda$ and $t$ & 8 \\
        \midrule 
        Initial dual variable $\lambda_\text{init}$ & 0.0 \\
        Step size $\eta_\lambda$ & $5 \times 10^{-5}$ \\
        Learning rate decay on $\eta_\lambda$ & \xmark \\
        \midrule
        Step size $\eta_t$ & $5 \times 10^{-5}$ \\
        Learning rate decay on $\eta_t$ & \xmark \\
        \bottomrule
    \end{tabular}
\end{center}
\end{table}

\subsubsection{Hyperparameters}

All hyperparameters used in our algorithm are listed in Tables \ref{table:env_params} and \ref{table:common_params}.

\paragraph{Setting \(\beta\)}  
{
The parameter \(\beta\) controls which quantile mean is constrained. Higher values (close to 1) enforce more risk-neutral constraints, while lower values focus on rare events. We chose \(\beta = 0.3\) to strike a balance: strong enough risk control to avoid constraint violations, while keeping the problem solvable.
}

\paragraph{Number of Trajectories for Gradient Computation}  
We tested \(n = \{2, 4, 8, 16\}\) (in Algorithm \ref{alg:main}) and concluded that \(n = 8\) offers the best trade-off between runtime and gradient smoothness.

\paragraph{Step Sizes \(\eta_\lambda\) and \(\eta_t\)}  
Step sizes were extensively tuned on Hopper and Walker2d. We tested values from \(10^{-3}\) to \(10^{-7}\). The best-performing configuration—\(\eta_\lambda = \eta_t = 5 \times 10^{-5}\)—was selected based on the convergence of $\lambda$ and $t$ within 15M time steps.

\paragraph{Initial Value of \(t\)}  
{
Since \(t\) must take values in the negative real range (due to the supremal form), we expect it to converge to the negative value-at-risk. Thus, we initialized \(t\) such that, with step size \(\eta_t = 5 \times 10^{-5}\), its magnitude could reach the velocity threshold over training.
}

\paragraph{Initial Value of \(\lambda\)}  
Because \(\lambda\) scales the penalty term added to the reward in \eqref{eq:modified_reward}, we initialized it neutrally with \(\lambda = 0.0\) to avoid overly aggressive penalties at the start.

\subsection{Computational Resources}
All experiments were performed on a computing system powered by an AMD Ryzen processor with 64 cores and 512 GB of RAM. A single NVIDIA RTX A6000 GPU with 48 GB VRAM was used for neural network training.

% %==========================
% \newpage
% \include{NeurIPS/checklist}

\end{document}